\documentclass{article}

\PassOptionsToPackage{numbers, compress}{natbib}

\usepackage[main, final]{neurips_2026}

\usepackage{graphicx}
\usepackage[utf8]{inputenc} %
\usepackage[T1]{fontenc}    %
\usepackage{url}            %
\usepackage{booktabs}       %
\usepackage{amsfonts}       %
\usepackage{nicefrac}       %
\usepackage{microtype}      %
\usepackage[dvipsnames]{xcolor}         %
\usepackage{amsmath,amssymb,amsthm,mathtools}
\usepackage{enumitem}

\theoremstyle{definition}
\newtheorem{definition}{Definition}
\newtheorem{remark}{Remark}
\newtheorem{assumption}{Assumption}

\usepackage{color}  
\usepackage{xcolor}         %
\usepackage{thmtools}   %
\usepackage{wrapfig}
\usepackage{float}
\usepackage{placeins}
\definecolor{hcolor}{cmyk}{0, 0.85, 0.85, 0.50}

\usepackage{hyperref}
\hypersetup{
    colorlinks=true,   
    citecolor=blue,    
    linkcolor=blue,   
    urlcolor=blue   
}

\usepackage{caption}
\usepackage{subcaption}
\usepackage{cleveref}
\definecolor{tocdarkblue}{RGB}{35,70,160}
\newcommand{\apptoclink}[2]{%
  \begingroup
  \hypersetup{allbordercolors=white}%
  \hyperref[#1]{\textcolor{tocdarkblue}{#2}}%
  \endgroup
}
\title{On the Burden of Achieving Fairness in \\ Conformal Prediction}

\author{%
  Ziang Gao\thanks{Equal contribution. Alphabetically ordered by last name.} \ \thanks{Corresponding author. Email: \texttt{ziang.gao@mail.mcgill.ca}.} \\
  McGill University\\
  Montreal, Canada\\
  \small\texttt{ziang.gao@mail.mcgill.ca}
  \And  
  Pengqi Liu\footnotemark[1] \\
  McGill University\\
  Montreal, Canada\\
  \small\texttt{pengqi.liu@mail.mcgill.ca}
  \And 
    Archer Yi Yang \\
  McGill University\\
  Montreal, Canada\\
  \small\texttt{archer.yang@mcgill.ca}
  \AND 
    Mouloud Belbahri \\
  TD Insurance\\
  Montreal, Canada\\
  \small\texttt{mouloud.belbahri@td.com}
  \And 
    Jesse C. Cresswell \\
  Layer 6 AI\\
  Toronto, Canada\\
  \small\texttt{jesse@layer6.ai}
  \And 
    Masoud Asgharian \\
  McGill University\\
  Montreal, Canada\\
  \small\texttt{masoud.asgharian2@mcgill.ca}
}

\begin{document}

\maketitle
\enlargethispage{\baselineskip}

\begin{abstract}
\looseness=-1
Conformal prediction is often calibrated with a single pooled threshold, but this can hide cross-group heterogeneity in score distributions and distort group-wise coverage. We study this phenomenon through the population score distributions underlying split conformal calibration. First, we derive a conservation law and lower bound showing that pooled calibration incurs irreducible group-wise coverage distortion at a scale set by cross-group quantile heterogeneity. Second, we demonstrate that the two leading fairness definitions for conformal prediction, Equalized Coverage and Equalized Set Size, are fundamentally in tension. Third, we quantify the cost of moving between policies which treat groups separately or pool them. Experiments on synthetic and real data confirm the same bidirectional trade-off after finite-sample calibration. Our results show that, for the policy families studied here, calibration choice does not remove cross-group heterogeneity; it determines whether the resulting distortion appears in the coverage or size dimension, providing a principled lens for analyzing fairness-oriented calibration choices in practice.
\end{abstract}

\section{Introduction}
\label{sec1:Intro}
Conformal prediction (CP) \citep{vovk2005algorithmic, shafer2008tutorial} provides a model-agnostic solution  for quantifying uncertainty over machine learning predictions. We focus throughout on classification with a finite label space \(\mathcal Y\). Let $Y$ be the outcome variable and $X\in\mathbb{R}^p$ be a vector of features. Given a calibration dataset $\{(X_i,Y_i)\}_{i=1}^n$, CP generates \emph{prediction sets} \(\hat{C}(X_{n+1})\) for test label $Y_{n+1}$, with a \emph{coverage} guarantee that the true label is in the set with user-specified probability $1-\alpha$ for $\alpha\in(0,1)$, i.e., \(\mathbb{P}(Y_{n+1} \in \hat{C}(X_{n+1})) \geq 1 - \alpha\). Smaller sets indicate less uncertainty about the prediction for $Y_{n+1}$, so average set size is commonly used as a metric given a fixed coverage level $1-\alpha$. CP is extremely versatile in applications since the coverage guarantee is valid in finite samples and is distribution-free, assuming only that test data is exchangeable with the calibration dataset.

CP is often implemented with a single threshold calibrated on the pooled calibration set. This delivers marginal coverage over the entire distribution, but still allows some groups within the data to be under-covered if others are over-covered to compensate. Exact distribution-free conditional coverage is impossible in general \citep{foygel2021limits}. We therefore study group-conditional CP as an intermediate target. %
Group-wise calibration resolves the potential disparity in coverage relative to pooled calibration \cite{vovk2003mondrian}, but the same heterogeneity reappears as disparity in expected set size. Thus, a calibration policy does not remove heterogeneity; it determines whether the disparity appears in coverage or in set size.%

In this work, we present a theoretical and empirical study of the fundamental tension between coverage disparity and set size disparity when applying CP over data containing heterogeneous groups. 
Our theoretical results give practitioners a principled way to understand what is gained and what is sacrificed when one fairness-oriented calibration objective is chosen over another. We study this question through the population score distributions underlying split conformal prediction. %
This viewpoint separates the structural effect of group heterogeneity from finite-sample calibration noise.

Our contributions are as follows:
\vspace{-2pt}
\begin{itemize}[leftmargin=*, nosep]
    \item We characterize pooled conformal calibration through a conservation law and lower bound for group-wise coverage, showing that pooled calibration can hide nontrivial group-level distortion.
    \item We establish a bidirectional impossibility result showing that \emph{exact group-wise coverage and equalized expected set size cannot, in general, be achieved simultaneously.} This identifies a structural limitation of fairness-oriented conformal calibration policies.
    \item We quantify the costs of switching from pooled calibration to group-wise calibration, and from the coverage-calibrated setting to equalized-size calibration.
\end{itemize}

The rest of the paper is organized as follows. Section~\ref{sec:related-work} reviews related work. Sections~\ref{sec3:UR} and ~\ref{sec4:trade-offs} develop our theoretical results. Section~\ref{sec4:experiments} translates the theory into empirical results on synthetic and real datasets. Section~\ref{conclusion} concludes with limitations and future directions. Technical discussions, supplemental results, proofs, and additional experimental details are deferred to the appendices.

\vspace{-5pt}
\section{Related Work}
\label{sec:related-work}
\vspace{-5pt}

A central line of work studies fairness notions for CP. \textbf{Equalized Coverage}~\citep{romano2020malice} identifies unequal group-wise empirical coverage under pooled calibration as a fairness concern, and later work extends this perspective to adaptively selected groups and more general frameworks~\citep{zhou2024conformal, vadlamani2025a}. Other works extend additional fairness notions to conformal prediction, including demographic parity~\citep{liu2022conformalized}, equal opportunity~\citep{wang2023equal}, and counterfactual fairness criteria~\citep{guldogan2025counterfactual}. A related literature studies fairness in downstream decision-making, where conformal sets are provided as a decision aid: conformal sets can improve human decisions~\citep{cresswell2024cphuman}, but human-subject experiments show that enforcing Equalized Coverage can worsen downstream fairness whereas \textbf{Equalized Set Size} can improve it~\citep{cresswell2025conformal}; \citet{liu2026beyond} extend this perspective with an LLM-in-the-loop evaluator. \citet{tasar2025} studies a trade-off between coverage parity and deferral parity in a binary human-in-the-loop setting. Our work is complementary to this literature: we formalize a structural incompatibility between Equalized Coverage and Equalized Set Size, and provide practitioners with a clear lens for weighing the costs of these disparities. The closest conceptual precedent to our results is the classical algorithmic-fairness literature on scores and classifiers: \citet{hardt2016equality} introduce equalized odds and equal opportunity; \citet{chouldechova2017} and \citet{kleinberg2017inherent} establish impossibility results for competing fairness criteria under unequal base rates; \citet{reich2021possibility} identify settings where partial reconciliation is possible. Our work identifies analogous impossibility results in CP, where the structural driver comes from cross-group heterogeneity in conformal quantiles. More broadly, our work is also related to limits of conditional coverage~\citep{foygel2021limits, gibbs2025conformal} and efficiency-oriented work on volume optimality for structured prediction sets~\citep{gao2025volume} and learned size-coverage trade-offs~\citep{bach2025convex}, but differs in isolating how cross-group heterogeneity constrains the simultaneous attainment of two group-level calibration objectives: exact group-wise coverage and equalized expected set size.        

Our focus on Equalized Coverage and Equalized Expected Set Size is deliberate. Both are fundamental procedural criteria proposed and studied in previous CP work. Importantly, we do not regard either criterion as universally fair: which criterion is substantively appropriate depends on how prediction sets enter the downstream decision process. Rather than advocating one criterion over the other, we discover a precise bidirectional characterization of the tension between Equalized Coverage and Equalized Expected Set Size and quantify the disparity induced when either criterion is prioritized.

\vspace{-6pt}
\section{Pooled Calibration and Groupwise Coverage Distortion}
\label{sec3:UR}
\vspace{-5pt}

\looseness=-1
This section is organized around two questions. What does a single pooled threshold guarantee? %
How large is the resulting group-wise coverage distortion when group-specific quantiles differ? We first recall the finite-sample CP constructions to fix notation, then
introduce the population-level objects that help answer these questions.
\vspace{-4pt}
\subsection{Conformal Prediction}
\vspace{-4pt}
We work in the standard setting of split conformal prediction~\citep{shafer2008tutorial}. Let \(X\in\mathbb{R}^p\) be a covariate vector, \(Y\) be a label, and \(S:\mathcal X\times\mathcal Y\to\mathbb R\) be any measurable real-valued nonconformity score, with larger values indicating greater nonconformity. Given a calibration sample \(\mathcal{D}_{\text{cal}}=\{(X_i,Y_i)\}_{i=1}^n\), let $S_i=S(X_i,Y_i)$ for $i=1,\ldots,n$ denote the calibration scores. Pooled split CP computes an empirical quantile at level $1-\alpha$ over the entire calibration sample as 
\begin{equation}\label{eq:quantile}\hat{q}:=\mathrm{Quantile}(1-\alpha;\{S_{i}\}_{i=1}^{n}\cup \{\infty\})\,
\end{equation}
where the (pooled) quantile is equivalently the
\(\lceil (n+1)(1-\alpha)\rceil\)-th order statistic of the augmented multiset. For any threshold \(t\), we define the associated prediction set as the collection of labels whose scores fall below the threshold
\begin{equation}
\label{eq:2}
    \hat{C}_t(X_{n+1}):= \{y: S(X_{n+1},y)\le t\}.
\end{equation}
The CP set for a new test point is thus given by \(\hat{C}_{\hat{q}}(X_{n+1})\), which, under exchangeability of the calibration and test examples, has the marginal coverage guarantee $\mathbb{P}(Y_{n+1}\in\hat{C}_{\hat{q}}(X_{n+1}))\geq 1-\alpha$. 

On the other hand, to construct a prediction set with a group-wise coverage guarantee \citep{vovk2003mondrian}, we define \(G : \mathcal X \to \mathcal G = \{1, \dots, K\}\) to be a prespecified discrete partition of the covariate space, with \(\mathbb P\{G(X)=g\}>0\) for every \(g\in\mathcal G\), and group label \(G(X)=g\), and choose threshold $t$ in \Cref{eq:2} to be the $1-\alpha$ quantile (group-wise) of  calibration scores within group $g$, i.e. $\hat{q}_{g}:=\mathrm{Quantile}(1-\alpha;\{S_{i}\}_{i:G(X_i)=g}\cup \{\infty\})$. The prediction set \(\hat C_{\hat q_g}(X_{n+1})\), under exchangeability within each prespecified group, satisfies $\mathbb P\{Y_{n+1}\in \hat C_{\hat q_g}(X_{n+1})
\mid G(X_{n+1})=g\}\ge 1-\alpha$ for a test point in group \(g\). Appendix~\ref{app:technical-discussion} Table~\ref{tab:notation},
summarizes the main population quantities used below.

\vspace{-4pt}
\subsection{Conservation Law for the Pooled Threshold}
\vspace{-3pt}
We first investigate the group-wise miscoverage of a marginal prediction set. The standard split CP set $\hat{C}_{\hat{q}}(X)$ with the pooled empirical quantile $\hat{q}$ only has a marginal coverage guarantee rather than a group-wise one: some groups may be under-covered while others are over-covered.   To study the group-wise miscoverage of $\hat{C}_{\hat{q}}(X)$, we define the \emph{signed group-wise coverage distortion} for a test point in group $g$ by
\begin{equation}
\varepsilon_g(\hat q)
:=
\mathbb P\{Y\in \hat C_{\hat q}(X)
\mid G(X)=g\}
-(1-\alpha).
\end{equation}
A positive $\varepsilon_g(\hat q)$ corresponds to over-coverage in group \(g\), while a negative
value corresponds to under-coverage. Under standard quantile-consistency conditions, the empirical pooled quantile 
\(\hat q\) concentrates around the pooled population quantile $q$ of the nonconformity score  $S$:
\begin{equation}\label{eq:threshold_definition}
q:=\inf\{t\in\mathbb{R}:F_S(t)\geq 1-\alpha\} 
\end{equation}
with $F_S(t) = \mathbb{P}(S\leq t)$. For any fixed threshold \(t\), we can see that the group-wise coverage is the conditional CDF
of the nonconformity score $S$: 
\begin{equation}
\label{eq:CDF_Fsg}
\mathbb P\{Y\in \hat C_t(X)\mid G(X)=g\}
=
\mathbb P\{S\le t\mid G(X)=g\}
=: F_{S|g}(t).
\end{equation}
In the large calibration sample case, $\varepsilon_g(\hat q)$ can be approximated by the \emph{population signed group-wise coverage distortion} $\varepsilon_g(q)$, defined as:
\begin{equation}
\varepsilon_g(q) := F_{S \mid g}(q) - (1-\alpha).
\end{equation}  
We begin by characterizing what the pooled threshold $q$ guarantees at the aggregate level for $\varepsilon_g(q)$. By definition of \(q\) (\Cref{eq:threshold_definition}), Theorem~\ref{thm:conservation} below gives the aggregate identity for group-wise coverage distortion; when \(F_{S}\) is continuous at \(q\), this identity becomes a zero-sum conservation law. (All proofs are presented in Appendix~\ref{app:proofs}.)

\begin{restatable}{theorem}{conservation}
\label{thm:conservation}
\text{(Conservation law for pooled calibration)} 
Let \(p_g:=\mathbb P\{G(X)=g\}\) be the group mass. Group-wise coverage distortion $\varepsilon_g(q)$ under the pooled threshold satisfies the aggregate identity 
\begin{equation}
\label{eq:conservation}
\sum_{g\in\mathcal{G}} p_g\,\varepsilon_g(q)=F_{S}(q)-(1-\alpha)=:\delta(q).
\end{equation}
If \(F_{S}\) is continuous at \(q\), then \(\delta(q)=0\), and hence \(\sum_{g\in\mathcal G} p_g\,\varepsilon_g(q)=0\).

\end{restatable}
\begin{remark}
If \(F_{S}\) has a jump at \(q\), randomized tie handling via the probability integral transform restores the zero-sum identity. Let \(U\sim\mathrm{Uniform}(0,1)\) be independent of \(S\) and \(G\), and define \(Z:=F_{S}(S^-)+U\cdot(F_{S}(S)-F_{S}(S^-))\). Then \(Z\sim\mathrm{Uniform}(0,1)\), which means that the event \(\{Z\le 1-\alpha\}\) implements randomized tie handling at \(q\). Next, we define \(\tilde\varepsilon_g:=\mathbb{P}(Z\le 1-\alpha \mid G=g)-(1-\alpha)\) to obtain \(\sum_{g\in\mathcal G}p_g\tilde\varepsilon_g=0.\)
\end{remark}

\vspace{-6pt}

\noindent\textbf{Take-away.} A pooled threshold \(q\) can satisfy marginal coverage while miscovering individual groups, which can be a fairness concern. Under continuity, or under randomized tie handling in the atomic case, Theorem~\ref{thm:conservation} shows that weighted over-coverage in some groups is exactly balanced by weighted under-coverage in others.

Theorem~\ref{thm:conservation} is an additive form. A product-type form of the conservation law is given in Appendix~\ref{app:two-group-specialization}.

\vspace{-4pt}
\subsection{The Pooled Threshold Uncertainty Relation}
\vspace{-3pt}
We now move from balance to magnitude. For each group, let \(q_g:=\inf\{t\in\mathbb{R}:F_{S\mid g}(t)\ge 1-\alpha\}\) denote the population group quantile. We quantify the irreducible root mean square (RMS) group-wise miscoverage induced by using a pooled threshold in the presence of group-quantile heterogeneity.
\begin{assumption}
\label{assumption_2}
There exists a common \(\eta >0\) such that, for each \(g\in\mathcal{G}\), the conditional CDF \(F_{S|g}\) is absolutely continuous on the interval
\(
I_g:=[\min(q,q_g)-\eta,\max(q,q_g)+\eta]
\),
with density \(f_{S|g}\) satisfying \(f_{S|g}(t)>0\) for almost every \(t\in I_g\). Moreover, whenever \(q\neq q_g\), \(\operatorname*{ess\,inf}_{t\in J_g} f_{S|g}(t)>0\), where \(J_g:=[\min(q,q_g),\max(q,q_g)]\).
In particular,
\begin{equation}
m_g(q):=\operatorname*{ess\,inf}_{t\in J_g%
}
f_{S|g}(t),
\end{equation}
with the convention \(m_g(q)=0\) when \(q=q_g\).
\end{assumption}
\vspace{-5pt}
Let $q_G$ and $\varepsilon_G(q)$ denote the random variable versions of $q_g$ and $\varepsilon_g(q)$, respectively, and
\vspace{-2pt}
\begin{equation}
\label{eq:FTC_segment}
F_{S|g}(q)-F_{S|g}(q_g)=\int_{q_g}^{q} f_{S|g}(t)\,dt.
\end{equation}
\begin{definition}
\label{def:meff}
The \emph{effective stiffness} is defined as
\vspace{-2pt}
\begin{equation}
\label{eq:meff}
m_{\mathrm{eff}}(q):=
\sqrt{\frac{\mathbb{E}[m_G(q)^2\,(q-q_G)^2]}{\mathbb{E}[(q-q_G)^2]}} \quad\text{whenever }\mathbb{E}[(q-q_G)^2]>0.
\end{equation}
\end{definition}
\vspace{-10pt}
Let \vspace{-2pt}
\begin{equation}
\sigma_\Delta^2
:= \operatorname{Var}(q_G)
= \mathbb{E}\!\left[(q_G - \mathbb{E}[q_G])^2\right]
= \sum_{g \in \mathcal G} p_g (q_g - \sum_{g' \in \mathcal G} p_{g'} q_{g'})^2.
\end{equation}
The quantity \(\sigma_\Delta\) measures the cross-group dispersion of the target quantiles \(\{q_g\}\). When \(\sigma_\Delta=0\), pooled and group-conditional calibration coincide at the population level. As a property of the population and choice of nonconformity score,  $\sigma_\Delta$, which we refer to as {\it intrinsic heterogeneity}, is a central object in our study of the differences between pooled and group-wise conformal calibration.

\begin{restatable}{theorem}{heisenberg}
\label{thm:heisenberg_L2}
\text{(Pooled-threshold uncertainty relation)}
Suppose Assumption~\ref{assumption_2} %
holds. Then

\vspace{-10pt}
\begin{equation}
\label{eq:heisenberg_L2}
\mathrm{Var}(\varepsilon_G(q)) \ge m_{\mathrm{eff}}(q)^2 \mathrm{Var}(q_G).
\end{equation}
\end{restatable}
\vspace{-3pt}
\noindent
\textbf{Take-away.} %
Under pooled calibration, the cross-group variation in the target quantiles induces a non-zero cross-group disparity in group-wise coverage distortions at the scale \(m_{\mathrm{eff}}(q)\sigma_\Delta\).\vspace{-2pt}

Thus, the fact that a pooled threshold causes group-wise coverage distortion is not an artifact of finite-sample calibration noise, but is already present at the population level. 

\vspace{-6pt}
\section{Calibration Policies and Coverage--Size Trade-offs}
\label{sec4:trade-offs}
\vspace{-5pt}
\subsection{Equalized Coverage versus Equalized Expected Set Size}
\label{sec:impossibility_trade_off}
\vspace{-5pt}
The results above quantify the irreducible coverage error incurred by a single pooled threshold $q$. We now ask what happens when one changes the calibration policy to mitigate pooled threshold distortion. We formalize two directions of the resulting trade-off between two group-level fairness criteria: \emph{Equalized Coverage}~\cite{romano2020malice}, meaning $\mathbb P(S \le t_g \mid G=g)=1-\alpha$ for all $g$, and \emph{Equalized Expected Set Size}~\cite{cresswell2024cphuman}, meaning the expected size of the CP set under a group-specific threshold $t_g$ is equal to a constant $\lambda$, i.e. $\mathbb E[|\hat C_{t_g}(X)| \mid G=g]=\lambda$ for all $g$.

Recall that, for each group $g\in\mathcal G$, the threshold
\(
q_g := F_{S|g}^{-1}(1-\alpha)
\)
achieves exact group-wise coverage level $1-\alpha$ whenever $F_{S|g}$ is continuous (\Cref{eq:CDF_Fsg}). 
We first present one sufficient condition, going from equalized coverage to expected-size disparity. To do so, fix a reference group $r\in\mathcal G$ and define
\(
    \mathcal H_r := \{g\in\mathcal G : q_g \ge q_r\},
\)
the set of groups whose coverage-calibrated thresholds are at least as large as that of the chosen reference group. %
Let 
\(
\ell_g(t) := \mathbb E[|\hat C_t(X)| \mid G=g]
\)
be the group-wise expected size of the CP set at threshold $t$. For a reference group \(r\), define the restricted mean squared size gap 
\begin{equation}D_r^2
:=
\mathbb E\!\left[
(\ell_G(q_G)-\ell_r(q_r))^2\mathbf 1\{G\in H_r\setminus\{r\}\}
\right].
\end{equation}
\begin{restatable}{lemma}{MonoSetSize}
\text{(Monotonicity of the expected set size)}
\label{lem:size-monotonicity}
For any \(t_1\le t_2\),
\(
\hat C_{t_1}(x)\subseteq \hat C_{t_2}(x)
\) for every \(x\),
\begin{equation}
\ell_g(t_1)\le \ell_g(t_2)
\quad\text{for every }g\in\mathcal G.
\end{equation}
\end{restatable}

\noindent
Indeed, \(S(x,y)\le t_1\le t_2\) implies that every label included at
\(t_1\) is also included at \(t_2\); taking conditional expectations gives the
second statement. While Lemma~\ref{lem:size-monotonicity} establishes within-group monotonicity of expected set size, it does not guarantee a disparity in expected set size across groups under exact group-wise coverage. To see this, for any \(q_g\ge q_r\), the coverage-calibrated size difference decomposes as
\begin{equation}
\ell_g(q_g)-\ell_r(q_r)
=
\underbrace{\ell_g(q_g)-\ell_g(q_r)}_{\text{within-group expansion}}
+
\underbrace{\ell_g(q_r)-\ell_r(q_r)}_{\text{cross-group size difference}}.
\end{equation}
By Lemma~\ref{lem:size-monotonicity}, the first term is nonnegative, whereas the second term can have either sign and may offset the first term. To rule out such compensation, we impose the following sufficient condition.
\begin{assumption}
    \label{insert_assump_1}
There exists a reference group $r \in \mathcal{G}$ with \(\mathcal{H}_r \setminus \{r\} \neq \varnothing\), and positive constants $\{c_g \}_{g \in \mathcal{H}_r \setminus \{r\}}$, such that 
\(
        \ell_{g}(q_r) - \ell_{r}(q_r) \ge c_g
\)
for every $g \in \mathcal{H}_r \setminus \{r\}$.
\end{assumption}

\vspace{-2pt}
Assumption~\ref{insert_assump_1} provides a sufficient no-compensation condition rather than a necessary characterization, and it does not assume that the final coverage-calibrated expected set sizes are already unequal. A natural regime captured by this assumption occurs when greater true-label calibration difficulty is aligned with broader candidate-label ambiguity: a harder-to-predict group may require a larger threshold to attain the target coverage, while less separated candidate-label scores cause more competing labels to enter its prediction set at the same reference threshold. Under Assumption~\ref{insert_assump_1}, the following theorem establishes that exact group-wise coverage necessarily induces cross-group disparity in expected set size.

\begin{restatable}{theorem}{SizeDisparity}
\label{thm:3}
\text{(Equalized coverage induces cross-group expected set size disparity)}
Assume the score CDF $F_{S|g}$ is continuous $\forall g \in \mathcal{G}$. Suppose Assumption \ref{insert_assump_1} holds for some reference group $r \in \mathcal{G}$. Then, 
\begin{enumerate}[topsep=1pt, itemsep=1pt, parsep=0pt,leftmargin=*]
\item 
The group-wise thresholds $\{q_g \}_{g \in \mathcal{G}}$, which achieve an exact group-wise coverage level $1-\alpha$, necessarily induce a nonzero cross-group disparity in expected set size. More precisely, for every $g \in \mathcal{H}_r \setminus \{r\}$, \(\ell_g(q_g)-\ell_r(q_r)\ge c_g>0\).
Consequently, 
\begin{equation}
    \max_{g,g' \in \mathcal{G}} |\ell_g(q_g) - \ell_{g'}(q_{g'})| \ge \max_{g \in \mathcal{H}_r \setminus \{r\}} c_g > 0.
\end{equation}
Therefore, exact group-wise coverage cannot simultaneously satisfy equalized expected set size across groups.
\item
The restricted mean squared cross-group size disparity relative to the reference group $r$ satisfies \begin{equation}
\begin{aligned}
D_r^2
&= \sum_{g \in \mathcal{H}_r \setminus \{r\}} p_g(\ell_g(q_g)-\ell_r(q_r))^2 
\ge \sum_{g \in \mathcal{H}_r \setminus \{r\}} p_g c_g^2 > 0.
\end{aligned}
\end{equation}
\end{enumerate}
\end{restatable}
\vspace{-8pt}
\noindent
\textbf{Take-away.} Exact group-wise coverage generally entails unequal expected set sizes across groups.

We now turn to the reverse direction and consider an equalized expected set
size policy at target level $\lambda$. We define
\(
\lambda_g := \ell_g(q_g), \ g\in\mathcal G,
\)
the coverage-calibrated expected set size of group \(g\), and we denote \(\lambda_G\) as the random-group version of \(\lambda_g\). We ask how an equalized-size policy alters group-wise coverage. 
Let \(\tau_g\) denote a group-specific threshold satisfying
\(
\ell_g(\tau_g)=\lambda.
\)
That is, \(\tau_g\) is the threshold required for group \(g\) to attain the equalized-size target. To quantify the coverage effect of moving from \(q_g\) to \(\tau_g\), we impose local regularity only on the segment between \(q_g\) and \(\tau_g\).
\begin{assumption}%
\label{assumption_6}
\text{(Local regularity for equalized-size perturbations)}
For each group \(g\), on the segment between \(q_g\) and \(\tau_g\) both \(F_{S|g}\) and \(\ell_g\) are absolutely continuous, with derivatives satisfying
\(
f_{S|g}(t) \ge m_g > 0, \ |\ell_g'(t)| \le V_g
\)
for almost every \(t\) on the segment, where \(0<V_g<\infty\).
\end{assumption}

\vspace{-5pt}

\Cref{thm:size_to_cov_disparity} below establishes the reverse direction of the trade-off compared to \Cref{thm:3}. Under an equalized-size policy $\{\tau_g\}_{g\in\mathcal G}$, one may eliminate cross-group disparity in expected set size, but only at the expense of introducing cross-group disparity in coverage whenever the common target level $\lambda$ lies strictly between two distinct coverage-calibrated set sizes. We utilize the local coverage--size conversion factor
\(
\kappa_g := m_g/V_g
\) for each group \(g\).
\begin{restatable}{theorem}{CovDisparity}
\label{thm:size_to_cov_disparity}
\text{(Equalized expected set size induces cross-group coverage disparity)}
Suppose Assumption \ref{assumption_6} holds,
and suppose for some common target level \(\lambda\) the thresholds \(\{\tau_g\}_{g\in\mathcal G}\)
satisfy
\(
\ell_g(\tau_g)=\lambda, \ g\in\mathcal G.
\)
Then the following results hold.

\noindent
\textbf{(i)} For every group \(g\) with \(\lambda_g < \lambda\),
\(
F_{S|g}(\tau_g) - (1-\alpha) \;\ge\; \kappa_g\,(\lambda-\lambda_g).
\)

\noindent
\textbf{(ii)} For every group \(g'\) with \(\lambda_{g'} > \lambda\),
\(
(1-\alpha) - F_{S|g'}(\tau_{g'}) \;\ge\; \kappa_{g'}\,(\lambda_{g'}-\lambda).
\)

Consequently, for any pair \(g,g'\in\mathcal G\) such that
\(
\lambda_g < \lambda < \lambda_{g'},
\)
we have the pairwise lower bound
\begin{equation}
F_{S|g}(\tau_g)-F_{S|g'}(\tau_{g'})
\;\ge\;
\kappa_g(\lambda-\lambda_g)+\kappa_{g'}(\lambda_{g'}-\lambda).
\end{equation}
In particular, if there exist groups \(g,g'\) with
\(
\lambda_g < \lambda_{g'}
\ \text{and}\
\lambda\in(\lambda_g,\lambda_{g'}),
\)
then
\begin{equation}
\max_{a,b\in\mathcal G}|F_{S|a}(\tau_a)-F_{S|b}(\tau_b)|
\;\ge\;
\kappa_g(\lambda-\lambda_g)+\kappa_{g'}(\lambda_{g'}-\lambda)
\;>\;0.
\end{equation}
\end{restatable}

\vspace{-10pt}

\noindent
\textbf{Take-away.} Whenever the equalized-size level \(\lambda\) lies strictly between two distinct coverage-calibrated set sizes, the equalized set size policy induces nonzero coverage disparity. 

\noindent
Together, \Cref{thm:3} and \Cref{thm:size_to_cov_disparity} show that \textbf{the two leading fairness notions in CP, namely Equalized Coverage and Equalized Set Size, cannot in general be achieved simultaneously}. To the best of our knowledge, this is the first formal result in CP establishing this structural incompatibility. This places our result alongside the classical impossibility theorems of \citet{kleinberg2017inherent} and \citet{chouldechova2017}, but in the distinct setting of CP, where the trade-off is driven by cross-group heterogeneity in conformal quantiles and its effect on coverage and set size. We next quantify the scale of policy-conversion distortions.

\vspace{-4pt}
\subsection{Quantitative Bounds for Policy-Conversion Distortions}
\vspace{-3pt}
\label{subsec:quantitative_bounds}

We now turn from directional trade-off results to quantitative policy-conversion bounds. The goal is to measure the scale of the distortions induced when one moves between pooled calibration, exact group-wise coverage, and equalized-size calibration. We first quantify the cost of moving from pooled calibration to exact group-wise coverage. To do so, we require a local responsiveness condition on the group-wise size curves along the segment between \(q\) and \(q_g\).
\begin{assumption}
\label{assumption_5}
\text{(Local set size responsiveness for group-wise calibration)}
For each group \(g\), the mapping \(t\mapsto \ell_g(t)\) is absolutely continuous on the segment between \(q\) and \(q_g\), and is non-decreasing on the segment.
Moreover, its derivative $\ell_g'(t)$ is defined almost everywhere and satisfies
\(
\ell_g'(t)\ge v_g>0\ \text{for almost every }t\text{ on that segment}.
\)
\end{assumption}
\vspace{-6pt}

Next, we define the \emph{effective set-size responsiveness} as
\begin{equation}
\label{eq:ceff}
v_{\mathrm{eff}}(q):=
\sqrt{\frac{\mathbb{E}[v_G^2\,(q-q_G)^2]}{\mathbb{E}[(q-q_G)^2]}}
\quad\text{whenever }\mathbb{E}[(q-q_G)^2]>0.
\end{equation}
The next corollary converts the intrinsic heterogeneity $\sigma_\Delta$ into a lower bound on set size distortion.

\begin{restatable}{corollary}{SizeDistortion}
\label{cor:length_disparity}
\text{(Exact group-wise coverage induces set size distortion)}
Suppose Assumption~\ref{assumption_5} holds. Let the pooled threshold \(q\) be replaced by the group-wise thresholds \(\{q_g\}\). Then the RMS set size lower bound is
\begin{equation}
\label{eq:length_disparity}
\sqrt{\mathbb{E}[(\ell_G(q_G)-\ell_G(q))^2]}
\ %
\ \ge\ v_{\mathrm{eff}}(q)\,\sigma_{\Delta}.
\end{equation}
\end{restatable}
\vspace{-6pt}
\noindent
\textbf{Take-away.} Exact group-wise coverage generally incurs a measurable expected set size distortion.

\noindent
Corollary \ref{cor:length_disparity} turns the directional statement into a scale statement. It shows that moving from pooled calibration to exact group-wise coverage produces a set size change with magnitude controlled by the same underlying cross-group quantile distortion.

We next quantify the reverse conversion cost. Theorem \ref{thm:size_to_cov_disparity} establishes the reverse direction of the trade-off by showing that an equalized-size policy induces coverage disparity. The following corollary converts the same argument into an aggregate RMS lower bound by measuring how far the equalized-size thresholds $\{\tau_g\}$ shift coverage from the exact group-wise benchmark $\{q_g\}$. From Theorem~\ref{thm:size_to_cov_disparity}, we obtain
\(
\mathbb E\!\left[(F_{S|G}(\tau_G)-F_{S|G}(q_G))^2\right]
\ge
\mathbb E\!\left[\kappa_G^2(\lambda_G-\lambda)^2\right],
\
\kappa_G := m_G/V_G.
\) Before we formally state this result, we define the \emph{effective coverage responsiveness}
\begin{equation}
\label{eq:kappa_eff}
\kappa_{\mathrm{eff}}(\lambda):=
\sqrt{\frac{\mathbb{E}[(m_G/V_G)^2(\lambda-\lambda_G)^2]}{\mathbb{E}[(\lambda-\lambda_G)^2]}}
\quad\text{whenever} \quad \mathbb{E}[(\lambda-\lambda_G)^2]>0 ,
\end{equation}
and the set-size heterogeneity \(\sigma_{\lambda} := \mathrm{sd}(\lambda_{G}) = \sqrt{\mathrm{Var}(\lambda_G)}\).
\begin{restatable}{corollary}{CovDistortion}
\label{cor:coverage_disparity}
\looseness=-1\text{(Equalized expected set size induces coverage distortion)}
Suppose Assumption~\ref{assumption_6} holds. Let the group-wise thresholds \(\{\tau_g\}_{g\in G}\) satisfy \(\ell_g(\tau_g)=\lambda\) for all \(g\in G\), so that they enforce equalized expected set size at a common level \(\lambda\).
Then the RMS coverage distortion relative to the exact group-wise coverage satisfies
\begin{equation}
\label{eq:coverage_disparity}
\sqrt{\mathbb{E}[(F_{S|G}(\tau_G)-F_{S|G}(q_G))^2]} \ \ge%
\ \kappa_{\mathrm{eff}}(\lambda) \sigma_\lambda .
\end{equation}
\end{restatable}
\vspace{-7pt}
\textbf{Take-away.} Equalizing expected set size generally incurs a measurable coverage cost.

Corollary \ref{cor:coverage_disparity} provides the corresponding quantitative statement for the reverse conversion compared to \Cref{cor:length_disparity}. The induced coverage distortion under equalized size calibration is governed by the dispersion of the coverage-calibrated set sizes and the local coverage--size sensitivity. Finally, the coverage distortion induced by changing from \(q\) to \(\tau_g\) can also be characterized with local regularity conditions. Let \(\ell_g(\tau_g)=\lambda\) for every group \(g\), and define the direct adjustment interval
\begin{equation}
I_g^\circ:=[\min\{q,\tau_g\},\max\{q,\tau_g\}],
\end{equation}
where \(\circ\) denotes the direct move from \(q\) to \(\tau_g\).

\begin{restatable}{corollary}{DirectPooledESSCoverage}
\label{cor:direct_pooled_ess}
\text{(Pooled to Equalized Set Size coverage change)}
Suppose that, on \(I_g^\circ\), both \(F_{S|g}\) and \(\ell_g\) are absolutely continuous and, for almost every \(t\in I_g^\circ\),
\(
f_{S|g}(t)\ge m_g^\circ>0, 
|\ell_g'(t)|\le V_g^\circ<\infty,
\)
with \(V_g^\circ>0\). Define
\(\kappa_g^\circ:=m_g^\circ/V_g^\circ.\)
Then,
\begin{equation}
\left|F_{S|g}(\tau_g)-F_{S|g}(q)\right|
\ge
\kappa_g^\circ\left|\lambda-\ell_g(q)\right| \, \text{ for every group \(g\)}
\end{equation}
\noindent
and
\begin{equation}
\sqrt{\mathbb E\!\left[\left(F_{S|G}(\tau_G)-F_{S|G}(q)\right)^2\right]}
\ge
\sqrt{\mathbb E\!\left[(\kappa_G^\circ)^2\left(\lambda-\ell_G(q)\right)^2\right]}.
\end{equation}
\end{restatable}
\vspace{-6pt}
\noindent
\textbf{Take-away.} Switching from pooled calibration to Equalized Set Size calibration induces a group-wise coverage change whose magnitude is lower-bounded by the required expected set size adjustment.

Taken together, Corollaries \ref{cor:length_disparity}--\ref{cor:direct_pooled_ess} quantify the cost of moving from pooled calibration to the two fairness policies studied in Section \ref{sec:impossibility_trade_off}, as well as between them. Corollary \ref{cor:length_disparity} measures how intrinsic heterogeneity reappears as set size distortion when exact group-wise coverage is enforced, while Corollary \ref{cor:coverage_disparity} measures how it reappears as coverage distortion when expected set size is equalized across groups. Corollary \ref{cor:direct_pooled_ess} additionally quantifies the direct group-wise coverage change from pooled calibration to Equalized Set Size calibration.

\vspace{-4pt}
\section{Experiments}
\vspace{-4pt}
\label{sec4:experiments}

\looseness=-1 The results in Sections~\ref{sec3:UR} and~\ref{sec4:trade-offs} were derived using population analysis. Our experiments in this section are not designed as a benchmark comparison against fairness-aware conformal methods; rather, they are designed to test whether the distortion transfer pattern remains visible under finite-sample calibration, matching how CP is used in practice. We focus on three finite-sample consequences of Sections~\ref{sec3:UR} and~\ref{sec4:trade-offs}: the pooled threshold lower-bound behavior (\Cref{thm:heisenberg_L2}) with the effective lower-bound scale \(m_{\mathrm{eff}}(q)\sigma_\Delta\), the RMS set-size distortion after switching from \(q\) to \(q_g\) (\Cref{cor:length_disparity}), and the RMS coverage distortion under equalized expected set size (\Cref{cor:coverage_disparity}).
Our code is available at \url{https://github.com/GreenPenguin001/group-cp-tradeoffs}.
\vspace{-4pt}
\subsection{Synthetic Simulations}
\vspace{-3pt}
\label{sec:synthetic-main}
\begin{figure*}[t]
  \centering
  \includegraphics[width=0.98\textwidth,trim={8 4 8 6},clip]{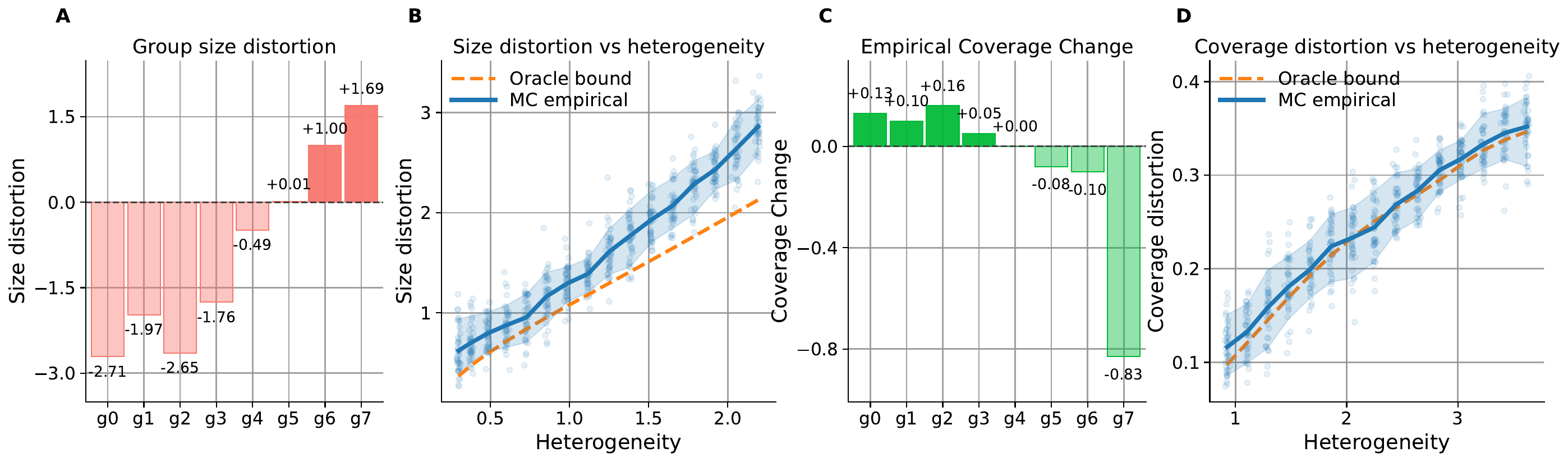}
  \vspace{-6pt}
  \caption{Bidirectional policy conversion in the synthetic study. Panels~A--B illustrate the coverage-to-size direction in Theorem~\ref{thm:3} and Corollary~\ref{cor:length_disparity}. Panel~A shows the signed change in expected set size for each of eight equally weighted groups, and Panel~B shows that the aggregate size distortion grows with cross-group score heterogeneity and tracks the oracle lower-bound scale. Panels~C--D correspond to the size-to-coverage direction in Theorem~\ref{thm:size_to_cov_disparity} and Corollary~\ref{cor:coverage_disparity}. Panel~C reports the signed finite-sample change \(\hat F_{S\mid g}(\hat\tau_g)-\hat F_{S\mid g}(\hat q_g)\) after replacing empirical group-wise thresholds by equalized-size thresholds; hence the displayed values are changes relative to the empirical group-wise benchmark. Panel~D shows that the aggregate coverage distortion grows with heterogeneity. %
  }
  \label{fig:policy-conversion}
\vspace{-8pt}
\end{figure*}

We generate calibration and test nonconformity scores directly from known group-conditional score distributions, and then apply the same empirical 
quantile rule used throughout the paper. This lets us compare finite-sample distortions with the corresponding population lower-bound scales from Sections~\ref{sec3:UR} and ~\ref{sec4:trade-offs}. Figure~\ref{fig:policy-conversion} studies the bidirectional policy conversion with eight equally weighted groups \(g_0, \ldots, g_7\), and with Gaussian mixture scores for the coverage-to-size direction, and Student-\(t\)/Gaussian mixture scores for size-to-coverage.
\begin{wrapfigure}[12]{r}{0.48\textwidth}
    \vspace{-10pt}
    \centering
    \includegraphics[width=0.44\textwidth]{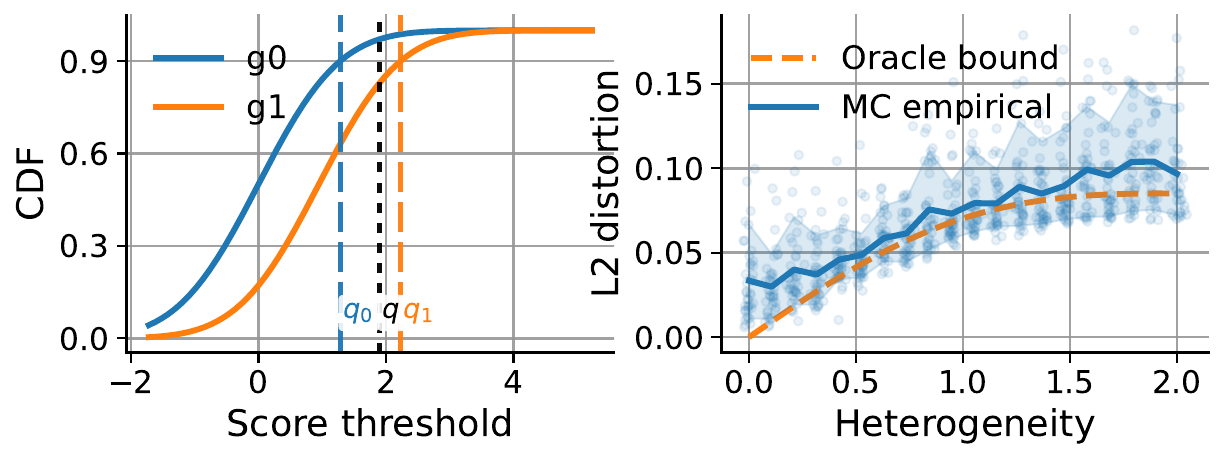}
    \caption{Two-group Gaussian pooled-threshold picture with \( q\) lying between \( q_0\) and \( q_1\); across the heterogeneity sweep, the empirical RMS miscoverage remains above the effective lower-bound scale, consistent with Theorem~\ref{thm:heisenberg_L2}.}
    \label{fig:two-group-app}
\end{wrapfigure}
For the policy conversion panels, we use monotone proxy size curves 
\(\ell_g(t)=a_g+b_gt\) to map thresholds to expected set sizes.
In Figure~\ref{fig:policy-conversion}, Panels A--B show the coverage-to-size direction in \Cref{cor:length_disparity}, 
with oracle scale \(v_{\mathrm{eff}}(q)\sigma_\Delta\); Panels C--D show the size-to-coverage direction in \Cref{cor:coverage_disparity}, with oracle scale \(\kappa_{\mathrm{eff}}(\lambda)\sigma_\lambda\), where coverage change is measured using the empirical quantile \(\hat{q}_g\). We find that mean distortions in the finite sample case are still bounded by the population analysis results.

\Cref{fig:two-group-app} gives a two-group Gaussian example. To achieve group-wise conditional coverage, one group requires a smaller quantile \(q_0\), whereas the other requires a larger quantile \(q_1\). The two groups have different target thresholds, and the pooled threshold lies between them, so under pooled calibration one group is over-covered and the other is under-covered (\Cref{thm:conservation}). As the separation between %
\(q_0\) and \(q_1\) increases (increasing intrinsic heterogeneity $\sigma_\Delta$), the finite-sample mean distortion remains above the oracle effective lower bound (\Cref{eq:heisenberg_L2}). Together, the synthetic studies support the bidirectional trade-off. All empirical curves use calibration thresholds constructed at target level \(\alpha=0.1\), and the reported distortions are evaluated on independent test samples and averaged over \(40\) Monte Carlo seeds with equal group weights. Detailed score families, heterogeneity ranges, additional diagnostics, and empirical--oracle ratio results under imbalance are deferred to Appendix~\ref{app:additional-experiments}. %

\begin{figure}[h]
\centering
\includegraphics[width=\linewidth,trim={8 4 8 35},clip]{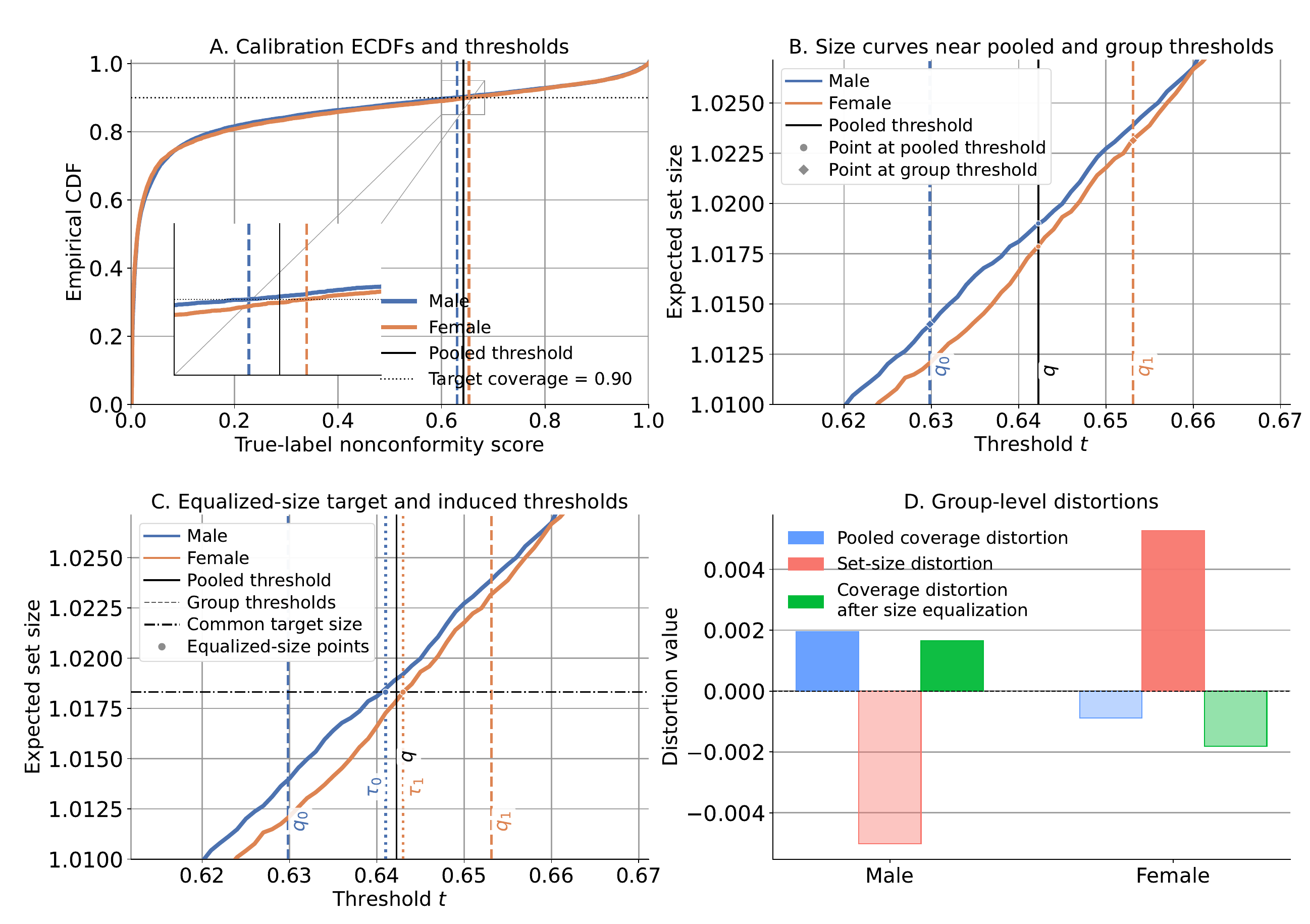}
  \vspace{-4pt}
\caption{\texttt{Bias in Bios} mechanism view at \(\alpha = 0.1\) for the simple score. Panel~A illustrates the pooled-threshold mechanism in Theorem~\ref{thm:conservation}; Panels~B--C illustrate Theorems~\ref{thm:3}--\ref{thm:size_to_cov_disparity} and Corollaries~\ref{cor:length_disparity}--\ref{cor:coverage_disparity}; Panel~D summarizes the three distortions (\Cref{thm:heisenberg_L2}, Corollaries~\ref{cor:length_disparity}--\ref{cor:coverage_disparity}) for male and female groups.}
\label{fig:biobias-main}
  \vspace{-1pt}
\end{figure}
\vspace{-5pt}
\subsection{\texttt{Bias in Bios} Experiments}
\vspace{-4pt}
\label{sec:biobias-main}
We use \texttt{Bias in Bios}~\citep{dearteaga2019biasinbios} as a real-data illustration of the policy-conversion mechanism from Sections~\ref{sec3:UR} and~\ref{sec4:trade-offs}.
We restrict to the ten most frequent professions, two demographic groups (Male and Female), and use a \texttt{DistilBERT} classifier~\citep{sanh2019distilbert}. Unless otherwise noted, we use the simple nonconformity score
\(
s(x,y) = 1-\hat p_y(x).
\) In Figure~\ref{fig:biobias-main}, Panel~A shows empirical score CDFs with the pooled and distinct group-specific thresholds.
Panel~B shows that equalizing coverage by moving from $q$ to $q_g$ enlarges the cross-group size disparity.
Panel~C shows the reverse move: equalizing set size requires shifting the thresholds away from $q_g$, reintroducing coverage distortion. 
Panel~D summarizes the three distortions.
At \(\alpha=0.1\), the corresponding RMS quantities, associated with \Cref{eq:heisenberg_L2,eq:length_disparity,eq:coverage_disparity} are \(0.0015\), \(0.0051\), and \(0.0017\), respectively. Thus, \texttt{Bias in Bios} supports the pooled-threshold distortion mechanism in Theorem~\ref{thm:heisenberg_L2} and the two policy-conversion effects quantified in Corollaries~\ref{cor:length_disparity} and~\ref{cor:coverage_disparity}. Detailed per-group quantities are deferred to Table~\ref{tab:biobias-main}. Additional robustness experiments, including alternative score comparisons and finite-calibration diagnostics, are reported in Appendix~\ref{app:biobias}.

\vspace{-5pt}
\subsection{\texttt{MultiNLI} Experiments}
\vspace{-5pt}
\label{sec:multinli-main}
\begin{figure}[h]
    \centering
    \includegraphics[height=4.8cm,trim={4 3 6 6},clip]{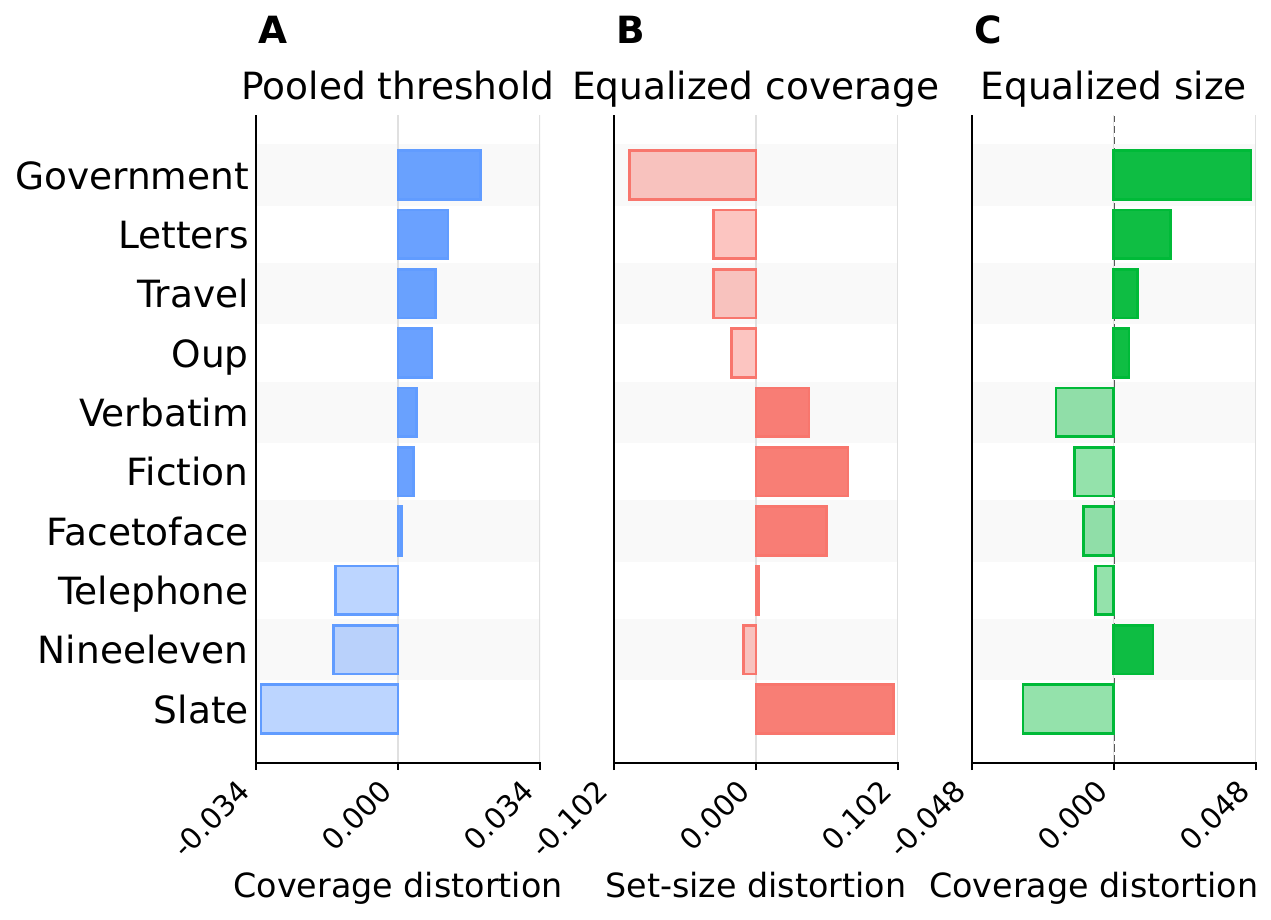}
    \hspace{0.02\linewidth}
    \includegraphics[height=4.8cm,trim={4 2 6 6},clip]{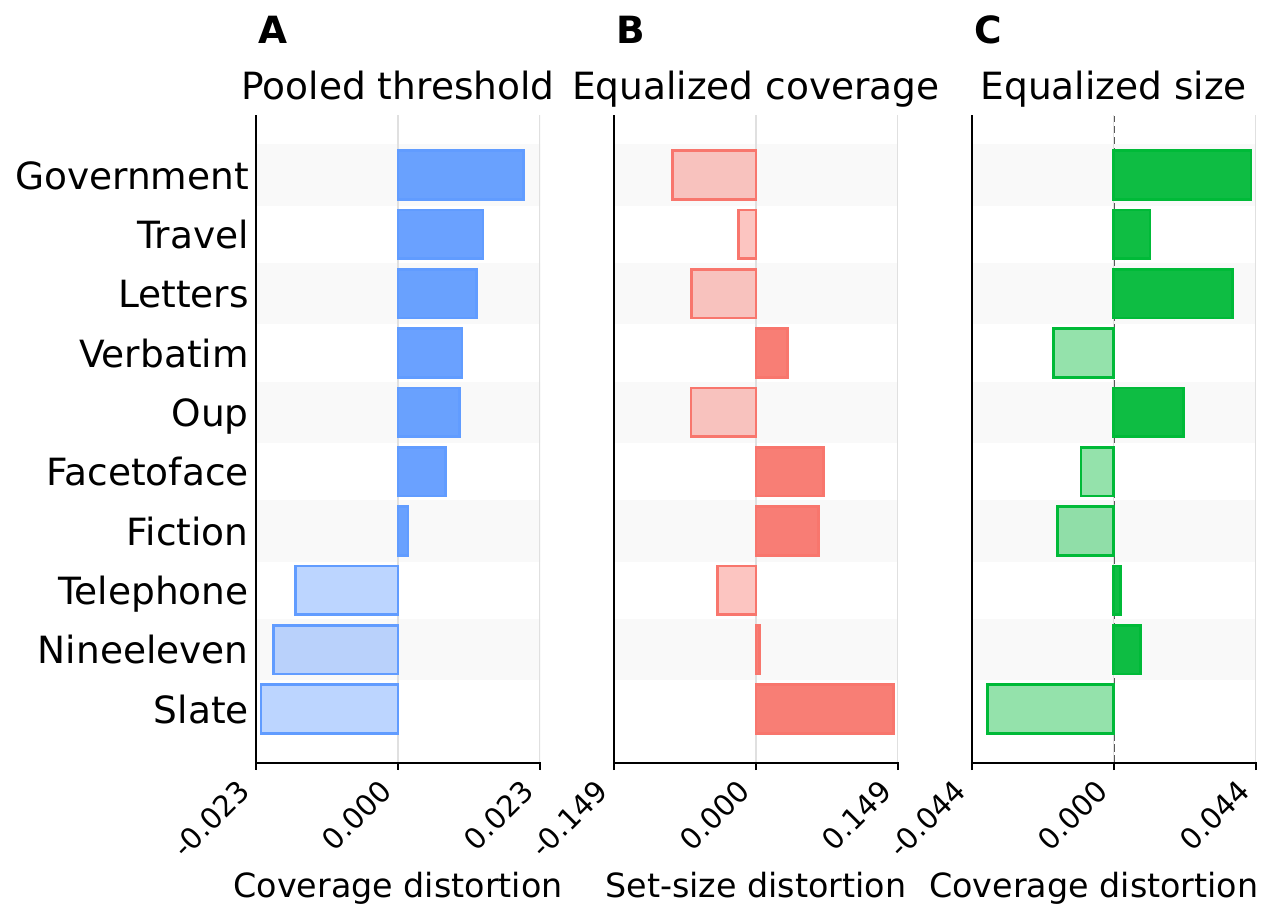}
    \caption{\texttt{MultiNLI} at \(\alpha = 0.1\) with simple (left) and RAPS (right) scores. 
    For each score, Panel~A shows signed coverage distortion \(\hat F_{S|g}(\hat{q}) - (1-\alpha)\) under pooled threshold (Theorem~\ref{thm:conservation}): positive bars indicate over-coverage and negative bars indicate under-coverage. Panel~B shows the signed change in expected set size \(\hat\ell_g(\hat q_g)-\hat\ell_g(\hat q)\) after switching to group-wise thresholds that equalize coverage (Corollary~\ref{cor:length_disparity}). Panel~C shows the signed coverage distortion \(\hat F_{S|g}(\hat\tau_g)-\hat F_{S|g}(\hat q_g)\) after enforcing a common expected set size (Corollary~\ref{cor:coverage_disparity}).}
    \label{fig:multi-primary}
    \vspace{-8pt}
\end{figure}
\looseness-1 Using the same post-hoc protocol as in Section~\ref{sec:biobias-main}, we treat the ten \texttt{MultiNLI}~\citep{williams2018broad} genres as groups. Figure~\ref{fig:multi-primary} shows the same qualitative transfer pattern (Theorems~\ref{thm:3} and \ref{thm:size_to_cov_disparity}) for both the simple score, \(
s(x,y) = 1-\hat p_y(x),
\) and the RAPS nonconformity score~\citep{angelopoulos2021uncertainty}. For the simple score at \(\alpha=0.1\), the pooled RMS genre-wise coverage distortion is \(0.0150\). The corresponding RMS set-size distortion after moving from \(q\) to \(q_g\) is \(0.0532\). In addition, the RMS coverage distortion under equalized expected set size is \(0.0209\). Thus, \texttt{MultiNLI} exhibits the signed genre-wise pattern and is consistent with the bidirectional trade-off described in Sections~\ref{sec3:UR} and ~\ref{sec4:trade-offs}. Per-genre summaries, additional robustness experiments, alternative-score results, and finite-calibration diagnostics are reported in Appendix~\ref{app:multi}.

\vspace{-4pt}
\subsection{\textbf{\texttt{FACET}} Experiments}
\vspace{-2pt}
\label{sec:facet-main}

We next show the same mechanisms on \texttt{FACET}~\citep{gustafson2023facet} using the RAPS score on the age group split (Younger, Middle, Older, Unknown) with a zero-shot \texttt{CLIP ViT-L/14} classifier \citep{radford2021learning}. 
\Cref{fig:facet-primary} shows the same transfer pattern on a more group imbalanced computer-vision dataset: under pooled calibration, the Younger group is over-covered while the others are under-covered. Switching from \(q\) to \(q_g\) removes the pooled-regime distortion but induces set-size distortion. Equalizing expected set size then reintroduces coverage distortion. Specifically, Panel A shows pooled-threshold coverage distortion, Panel B the set-size shift after equalized coverage, and Panel C the coverage shift after equalized expected size.  At \(\alpha=0.1\), the empirical pooled RMS coverage distortion is \(0.0083\).
\begin{wrapfigure}{r}{0.45\textwidth}
    \centering
    \includegraphics[width=0.43\textwidth]{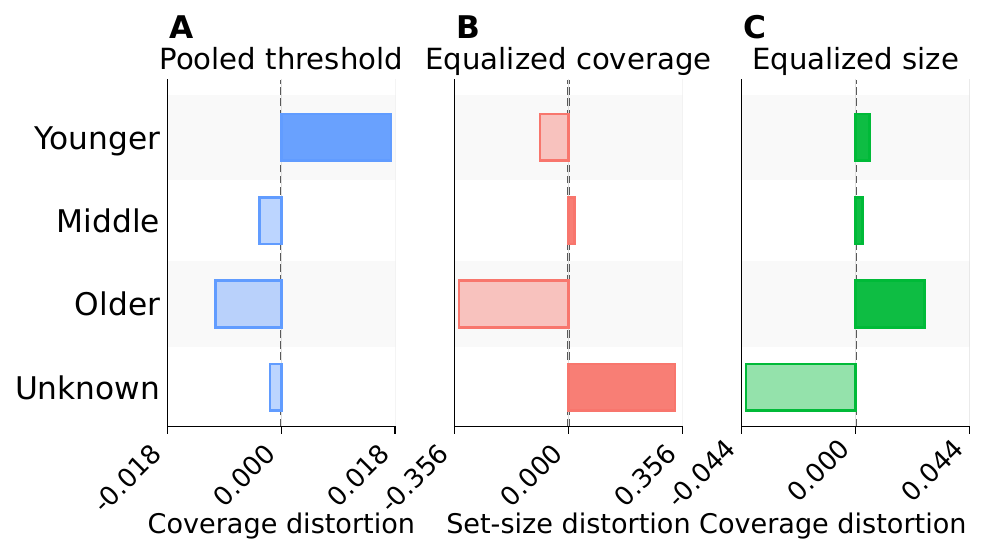}
    \caption{\texttt{FACET} at \(\alpha = 0.1\) with the RAPS score; Panel A illustrates Theorem~\ref{thm:conservation}, and Panels B--C illustrate Corollaries~\ref{cor:length_disparity}--\ref{cor:coverage_disparity}.}
    \label{fig:facet-primary}
    \vspace{-6pt}
\end{wrapfigure}
The RMS set-size distortion after changing from \(q\) to \(q_g\) is \(0.1717\). The RMS coverage distortion under equalized expected set size is \(0.0199\). Thus, \texttt{FACET} supports the behavior in Theorem~\ref{thm:conservation} together with the two policy-conversion effects quantified in Corollaries~\ref{cor:length_disparity} and~\ref{cor:coverage_disparity}. Per-group summaries, additional robustness experiments, and calibration-resampling stability results are reported in Appendix~\ref{app:facet}.

Taken together, the experiments in this section examine three finite-sample consequences of Sections~\ref{sec3:UR} and~\ref{sec4:trade-offs}: the pooled-threshold effective lower-bound behavior, the RMS set-size distortion after switching from \(q\) to \(q_g\), and the RMS coverage distortion under equalized expected set size. We use the effective constants only as oracle/proxy diagnostic scales, not as finite-sample estimators for deployment. Appendix~\ref{app:plugin-diagnostics} describes the empirical plug-in diagnostics used to compute the heterogeneity scales and lower-bound proxies reported in our experiments. Additional detectability experiments on how large calibration splits should be for the structural floor to become empirically resolvable are deferred to Appendices~\ref{app:biobias-detectability}, \ref{app:multi-detectability}, and~\ref{app:facet-detectability}.
  \vspace{-6pt}
\section{Conclusion}
\label{conclusion}
  \vspace{-4pt}
We studied group-conditional conformal prediction through the population score distributions underlying split conformal calibration. Our main result is structural: when group-wise conformal quantiles differ, the two group-level objectives studied here, exact group-wise coverage and equalized expected set size, cannot in general be achieved simultaneously. Under a pooled threshold, this heterogeneity appears as a group-wise coverage disparity; enforcing exact group-wise coverage shifts the heterogeneity to expected set size disparity across groups, while equalizing expected set size reintroduces coverage disparity. These results show that the two leading fairness notions in CP, Equalized Coverage and Equalized Set Size, exhibit a structural trade-off, thereby identifying the CP counterpart of classical impossibility results in algorithmic fairness. 

Our goal is not to solve full conditional coverage, nor to prescribe a normative fairness criterion. Rather, we identify a population-level constraint induced by heterogeneous group quantiles under conformal calibration. The quantitative bounds in Section~\ref{subsec:quantitative_bounds} rely on local regularity assumptions, including continuity of score distributions and nondegenerate local sensitivity of the set size curves. In addition, the analysis is developed at the level of prespecified discrete groups and population score distributions, while finite-sample behavior is studied empirically. %
These limitations are deliberate: they separate the structural effect of heterogeneity from broader questions about adaptive subgroup validity, end-to-end training, or fully conditional guarantees. 

In this sense, a calibration policy should be understood not as removing group heterogeneity, but as determining whether it manifests as cross-group disparity in coverage or set size, and how large the resulting policy-conversion distortions will be. A natural next step is to extend the analysis beyond prespecified discrete groups and population score distributions, and to develop finite-sample theory for the policy-conversion distortions. Another extension is to connect these structural trade-offs to adaptive subgroup validity and to learned calibration policies in practice.

\newpage
\begin{ack}
\begin{center}
    \begin{minipage}[c]{0.23\linewidth}
        \centering
        \includegraphics[width=2cm]{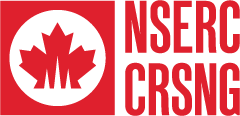}
    \end{minipage}%
    \hfill
    \begin{minipage}[c]{0.23\linewidth}
        \centering
        \includegraphics[
            height=1.2cm,
            trim=140bp 100bp 125bp 130bp,
            clip
        ]{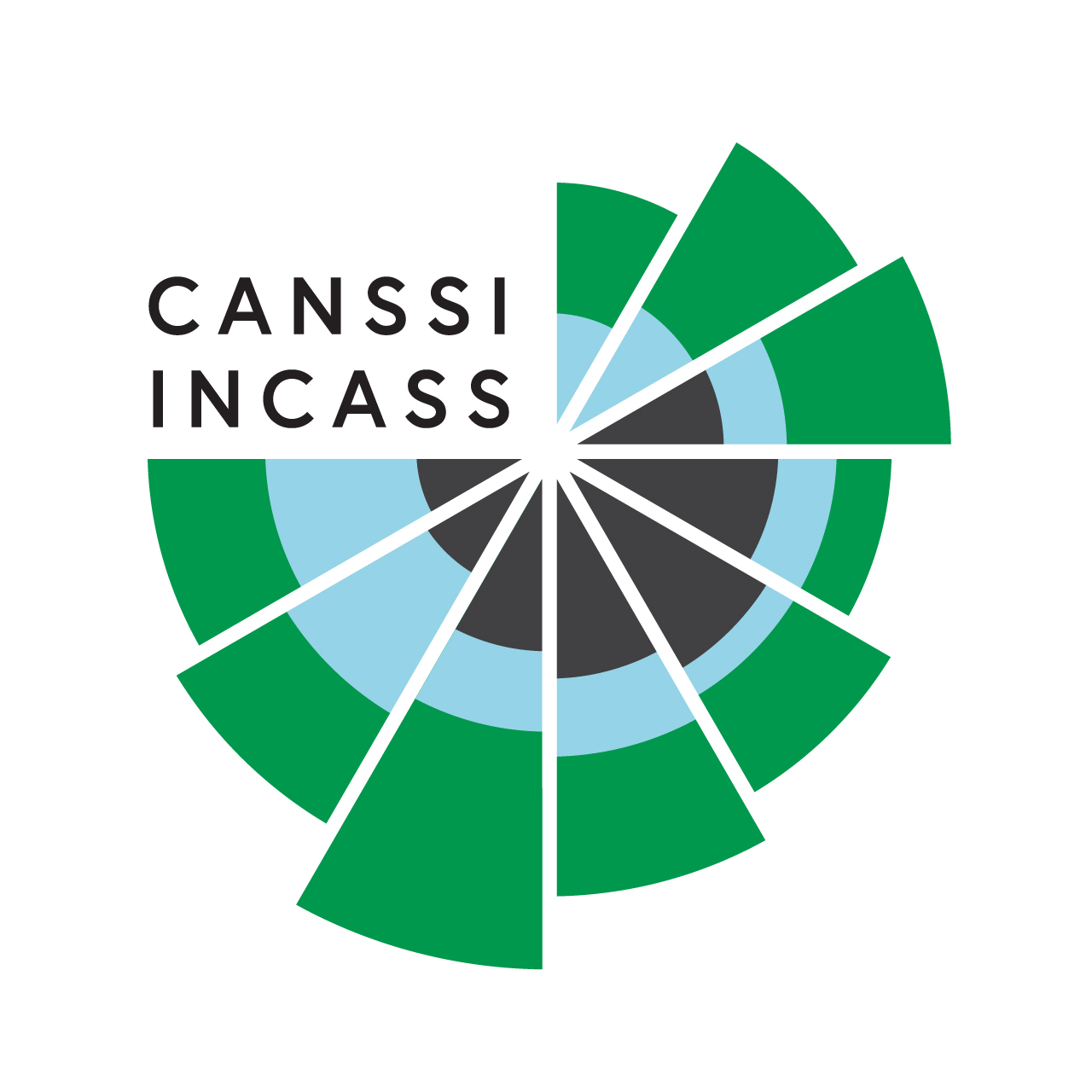}
    \end{minipage}%
    \hfill
    \begin{minipage}[c]{0.23\linewidth}
        \centering
        \includegraphics[
            height=1.1cm,
            trim=150bp 85bp 0bp 270bp,
            clip
        ]{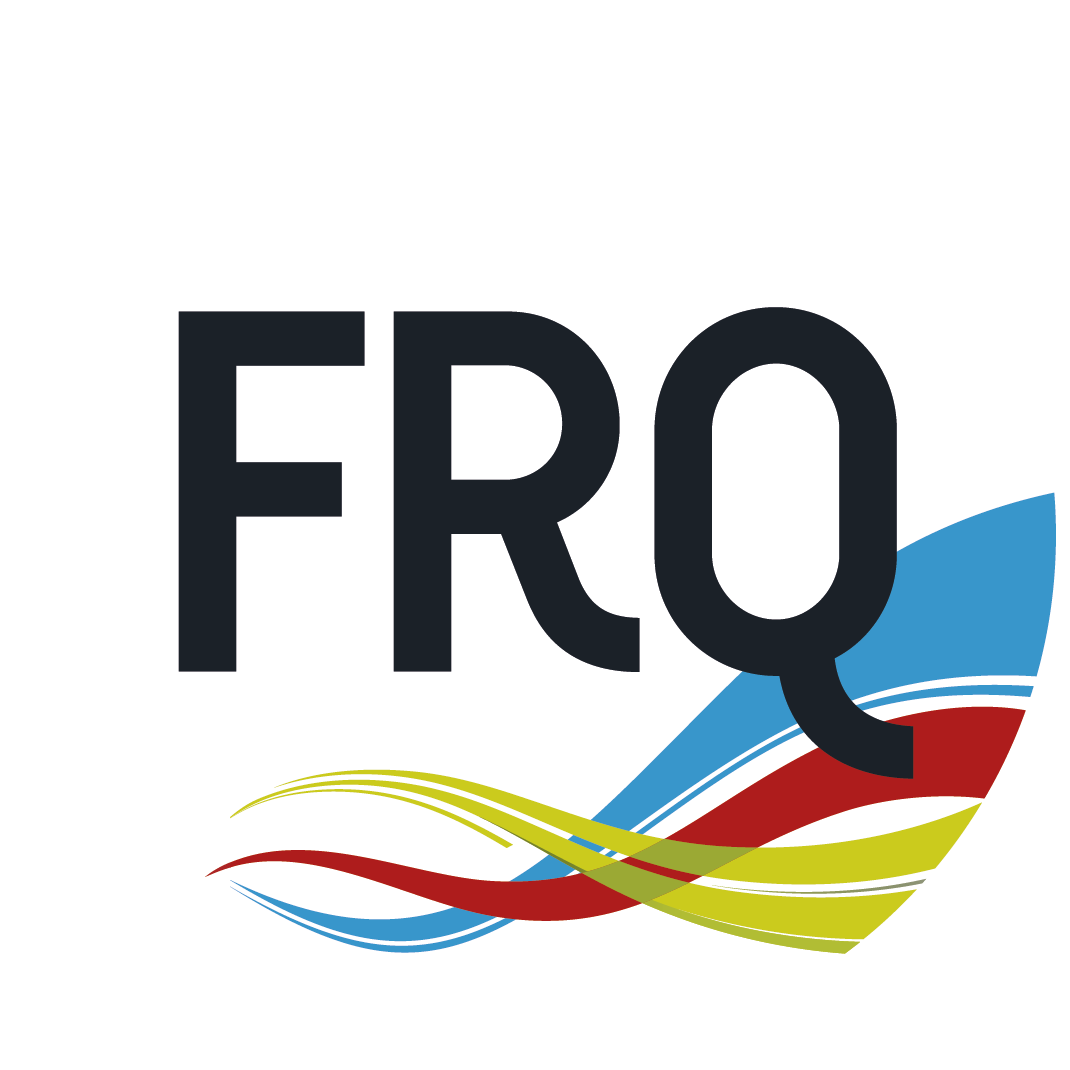}
    \end{minipage}%
    \hfill
    \begin{minipage}[c]{0.23\linewidth}
        \centering
        \includegraphics[width=2.6cm]{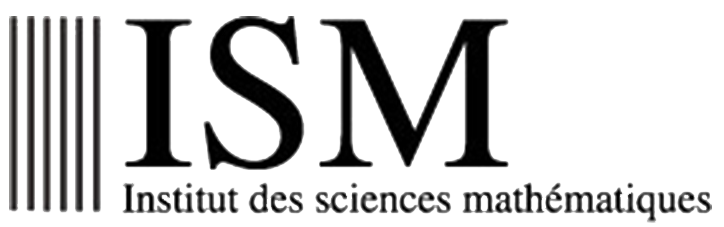}
    \end{minipage}
\end{center}
\par\vspace{0.2cm}
\noindent
We acknowledge the support of the Natural Sciences and Engineering Research Council of Canada (NSERC) through Discovery Grants held by Masoud Asgharian [RGPIN-2024-05640] and Archer Yi Yang [RGPIN-2024-06780]. This work was also supported in part by research funding held by Archer Yi Yang from the Canadian Statistical Sciences Institute (CANSSI) through its Collaborative Research Teams program and from the Fonds de recherche du Qu\'ebec (FRQ), Nature et technologies sector (Team Research Project Grant No. 327788). Ziang Gao acknowledges scholarship support from the Institut des sciences math\'ematiques (ISM).

\noindent\textbf{Competing interests.}

The authors declare no additional competing interests
beyond their disclosed affiliations.
\end{ack}
\bibliographystyle{plainnat}
\bibliography{bib}
\clearpage

\appendix
\section*{Appendix Contents}
\begin{itemize}
    \item \apptoclink{app:technical-discussion}{Appendix A: Technical Discussion}
    \item \apptoclink{app:proofs}{Appendix B: Proofs of Theoretical Results}
    \item \apptoclink{app:additional-experiments}{Appendix C: Additional Experimental Details}
    \item \apptoclink{app:biobias}{Appendix D: \texttt{Bias in Bios} Experiments}
    \item \apptoclink{app:multi}{Appendix E: \texttt{MultiNLI} Experiments}
    \item \apptoclink{app:facet}{Appendix F: \texttt{FACET} Experiments}
    \item \apptoclink{app:additional-robustness}{Appendix G: Additional Robustness and Scale Analyses}
    \item \apptoclink{app:compute_resource}{Appendix H: Computational Resources}
\end{itemize}

\FloatBarrier
\section{Technical Discussion}
\label{app:technical-discussion}

\begin{table}[t]
    \centering
    \small
    \caption{Notation summary}
    \label{tab:notation}
    \begin{tabular}{p{0.18\linewidth} p{0.63\linewidth}}
    \toprule
    Symbol & Meaning\\
    \midrule
    \(\mathcal G=\{1,\dots,K\}\) & Prespecified discrete group set.\\
    
    \(G(X), \quad p_g\) & Group label of \(X\) and group proportion \(p_g=\mathbb P(G(X)=g)\).\\
    
    \(S\) & Nonconformity score.\\
    
    \(F_{S \mid g}(t)\) & Conditional score CDF for group \(g:\mathbb P (S \le t \mid G=g)\).\\

    \(q_g\) & Group-specific \(1-\alpha\) quantile of \(F_{S \mid g}\).\\

    \(F_{S}(t)\) & Mixture score CDF: \(\sum_{g\in\mathcal G} p_g F_{S\mid g}(t)\). \\

    \(q\) & Pooled \(1-\alpha\) quantile of \(F_{S}\).\\
    
    \(\varepsilon_g(q)\) & Group-wise miscoverage under pooled \(q\): \(F_{S\mid g}(q)-(1-\alpha)\).\\

    \(q_G, \quad \varepsilon_G(q)\) & Random-group versions of \(q_g\) and \(\varepsilon_g(q)\) when \(G\sim p\).\\

    \(\sigma_\Delta\) & Cross-group quantile heterogeneity: \(\mathrm{sd}(q_G)\).\\

    \(\ell_g(t)\) & Expected set size for group \(g\) at threshold \(t\).\\

    \(\lambda_g\) & Coverage-calibrated expected set size for group \(g\): \(\ell_g(q_g)\).\\

    \(\tau_g\) & Group-wise threshold satisfying \(\ell_g(\tau_g)=\lambda\) under an equalized size.\\

    \(m_g(q),\quad m_{\mathrm{eff}}(q)\) & Local score density stiffness and its effective aggregate version.\\

    \(v_g,\quad v_{\mathrm{eff}}(q)\) & Local responsiveness of \(\ell_g\) and its aggregate version.\\

    \(\kappa_g, \quad \kappa_{\mathrm{eff}}(\lambda)\) & Local coverage--size conversion factor and its aggregate version.\\

    \(\sigma_\lambda\) & Dispersion of coverage-calibrated set size: \(\sigma_\lambda=\mathrm{sd}(\lambda_G)\).\\
    \bottomrule
    \end{tabular}
\end{table}

\subsection{Two-group Specialization of the Conservation Law and Uncertainty Relations}\label{app:two-group-specialization}
In finite-sample split conformal prediction, the empirical quantile \(\hat q\) satisfies the identity \eqref{eq:conservation} approximately and converges to the population identity asymptotically. We first formalize a two-group version of the conservation law. Let \(K=2\), with groups \(g\in\{0,1\}\) and \(p:= \mathbb{P}(G=1)\in(0,1)\). Without loss of generality, we assume the two quantiles satisfy \(q_0 < q_1\). The distance between them is denoted by \(\Delta := q_1-q_0\). For notational convenience, in the two-group case we write \(F_0 := F_{S\mid 0}\) and \(F_1 := F_{S\mid 1}\). Whenever densities exist, we likewise write \(f_0 := f_{S\mid 0}\) and \(f_1 := f_{S\mid 1}\).
\begin{restatable}{lemma}{TwoGroupQuantileInterval}
\label{lem:two_group_quantile_interval}
Assume $q_0<q_1$, $p\in(0,1)$ and that \(F_0\) and \(F_1\) are continuous. Then the pooled quantile $q=F_{S}^{-1}(1-\alpha)$ satisfies $q\in[q_0,q_1]$.
\end{restatable}
\begin{assumption}[Local density lower bounds for two-group case]
\label{assumption_1}
For each group \(g\in\{0,1\}\), the conditional distribution of scores admits a density \(f_{S|g}\) on the interval \([q_0,q_1]\) and
\(
\operatorname*{ess\,inf}_{t\in[q_0,q_1]} f_0(t)\ge m_0>0,\text{ and }
\operatorname*{ess\,inf}_{t\in[q_0,q_1]} f_1(t)\ge m_1>0.
\)
\end{assumption}
We define the over- and under-coverage magnitudes
\(
\omega_o:=\varepsilon_0(q)=F_0(q)-(1-\alpha),
\
\omega_u:=-\varepsilon_1(q)=(1-\alpha)-F_1(q).
\)
By Theorem~\ref{thm:conservation}, the two-group conservation identity gives
\begin{equation}
\label{eq:uo_coupling}
(1-p)\,\omega_o = p\,\omega_u\qquad\Longleftrightarrow\qquad \omega_o=\frac{p}{1-p}\omega_u.
\end{equation}
Finally, we write
\begin{equation}
\rho:=\frac{p}{1-p},
\qquad
B_{01}:=
\frac{\Delta}{\frac{1}{m_1}+\rho\frac{1}{m_0}}.
\end{equation}
Next, we present the two-group uncertainty relation in the following theorem.

\begin{restatable}{theorem}{TwoGroupProduct}
\label{thm:two_group_product}
\text{(Two-group product-type uncertainty relation)}
Suppose Assumption~\ref{assumption_1} holds. Let \(q_0 < q_1\).
Then %
\begin{align}
\label{eq:two_group_u_bound}
\omega_u &\ge B_{01},\\
\label{eq:two_group_o_bound}
\omega_o &\ge \rho B_{01},\\
\label{eq:two_group_product}
\omega_u\,\omega_o &\ge \rho B_{01}^{2}.
\end{align}
\end{restatable}
We note that the constant \(\Delta\) is an intrinsic heterogeneity disparity at level \(1-\alpha\). The ratio \(p/(1-p)\) quantifies how much the pooled threshold is pulled toward group \(1\). The terms \(m_0\) and \(m_1\) encode local score sensitivities, implying that smaller local densities amplify the threshold displacement required to compensate for a given level of group-wise miscoverage. We next extend this product-type statement to the multi-group setting.

\subsection{Generalization of a Product-type Conservation Law and Uncertainty Relation}
\label{app:multi-group-specialization}
The two-group statement above admits the following aggregate multi-group extension.
\begin{assumption}[Local density lower bounds for multi-group case]
\label{assumption:multigroup_heisenberg}
Let \(q_{\min}:=\min_{g\in\mathcal G} q_g\) and
\(q_{\max}:=\max_{g\in\mathcal G} q_g\), with \(q_{\min}<q_{\max}\),
and let \(q\in(q_{\min},q_{\max})\) be the threshold under consideration. 
For each group
\(g\in\mathcal G\), \(F_{S|g}\) is absolutely continuous on
\([q_{\min},q_{\max}]\), satisfies \(F_{S|g}(q_g)=1-\alpha\), and has
density \(f_{S|g}\) satisfying
\(
f_{S|g}(t)\ge m_g>0
\)
for almost every \(t\in[q_{\min},q_{\max}]\).
\end{assumption}

We define the over- and under-coverage magnitudes for \(K\) groups as
\begin{equation}
 \Omega_o(q):=\sum_{\varepsilon_g(q)>0}p_g\varepsilon_g(q),
\qquad
 \Omega_u(q):=\sum_{\varepsilon_g(q)<0}p_g[-\varepsilon_g(q)],
\end{equation}
and
\begin{equation}
w_g:=p_gm_g,\qquad
\bar q_m:=\frac{\sum_{g\in\mathcal G}w_gq_g}{\sum_{g\in\mathcal G}w_g},
\qquad
B_K:=\frac12\sum_{g\in\mathcal G}w_g|q_g-\bar q_m|.
\end{equation}

\begin{restatable}{theorem}{MultiGroupProduct}
\label{thm:multigroupproduct}
\text{(Generalization of a product-type uncertainty relation)}
Suppose Assumption~\ref{assumption:multigroup_heisenberg} holds. 
Then
\begin{enumerate}

\item
$
\max\{\Omega_o(q), \Omega_u(q)\}\ge B_K
\ \text{and} \  
\min\{\Omega_o(q),\Omega_u(q)\}
\ge \bigl(B_K-|\delta(q)|\bigr)_+, $
where \((x)_+=\max\{x,0\}\). %

\item
$
\Omega_o(q) \  \Omega_u(q)
\ge
B_K\bigl(B_K-|\delta(q)|\bigr)_+ \ \cdot
$

\item
Under exact conservation, \(\delta(q)=0\), we have
$\Omega_o(q)\ \Omega_u(q)\ge B_K^2. $

\end{enumerate}
\end{restatable}
When \(q\) is the pooled population quantile and the mixture CDF \(F_S\) is continuous at \(q\), Theorem~\ref{thm:conservation} gives \(\delta(q)=0\), so the exact-conservation form in part 3 is the relevant pooled-calibration case.

Theorem~\ref{thm:multigroupproduct} refines Theorem~\ref{thm:conservation} from a signed additive conservation law to a magnitude lower bound. The product-type form echoes the fact that pooled calibration does not remove group heterogeneity but redistributes it into over- and under-coverage aggregates.

\subsection{Coverage Disparity and Quantile Variance}
In this section, we give an elaboration of Theorem~\ref{thm:heisenberg_L2} when \(\mathbb{E}[\varepsilon_G(q)]=0\). The resulting inequality exhibits a Cram\'er--Rao-type structure.
\begin{definition}
\label{def:avg_density}
For each group, define the segment average density
\begin{equation}
\label{eq:fbar}
\bar f_{S|g}(q):=
\begin{cases}
\dfrac{F_{S|g}(q)-F_{S|g}(q_g)}{q-q_g} = \dfrac{\varepsilon_g(q)}{q-q_g}, & q\ne q_g,\\
1, & q=q_g\ .
\end{cases}
\end{equation}
Assume \(\bar f_{S|g}(q)>0\) for all groups with \(q\ne q_g\); when \(q=q_g\), the corresponding term in \eqref{eq:apparatus_cost} is zero.
We define the apparatus cost term
\begin{equation}
\label{eq:apparatus_cost}
\mathrm{L}(q):=\sqrt{\mathbb{E}\Big[\Big(\frac{q-q_G}{\bar f_{S|G}(q)}\Big)^2\Big]}.
\end{equation}
\end{definition}

\begin{restatable}{theorem}{HeisenbergProduct}
\label{thm:product_heisenberg}
Suppose Assumption~\ref{assumption_2} holds. With Definition~\ref{def:avg_density},
\begin{equation}
\label{eq:product_heisenberg}
\sqrt{\mathbb{E}\big[\varepsilon_G(q)^2\big]}
\cdot
\mathrm{L}(q)
\ \ge\ 
\mathbb{E}\big[(q-q_G)^2\big]
\ \ge\ \sigma_{\Delta}^2.
\end{equation}
\end{restatable}
The term \(\mathrm{L}(q)\) measures the inevitable cost of imposing a pooled \(q\). The term is large for groups whose quantile displacement is large relative to the induced coverage distortion. In particular, the term quantifies how much the device must move to compensate for heterogeneity. The inequality \eqref{eq:product_heisenberg} states that one cannot simultaneously make RMS miscoverage small and keep the apparatus cost small when \(\sigma_{\Delta}\) is non-negligible.

The Cauchy--Schwarz step in Theorem~\ref{thm:product_heisenberg} also has a H\"older version. Let
\(r,s\in[1,\infty]\) be conjugate exponents, \(1/r+1/s=1\). Under Definition~\ref{def:avg_density},
\begin{equation}
\|\varepsilon_G(q)\|_{L^r(p)}
\left\|
\frac{q-q_G}{\bar f_{S|G}(q)}
\right\|_{L^s(p)}
\ge
\mathbb E[(q-q_G)^2]
\ge
\sigma_\Delta^2 .
\end{equation}
Indeed, writing \(Q=q-q_G\), Definition~\ref{def:avg_density} gives
\(
|\varepsilon_G(q)|=\bar f_{S|G}(q)|Q|,
\)
and hence
\begin{equation}
\mathbb E[Q^2]
=
\mathbb E\left[
|\varepsilon_G(q)|
\left|
\frac{Q}{\bar f_{S|G}(q)}
\right|
\right].
\end{equation}
Applying H\"older's inequality gives the first inequality. The second follows from
\begin{equation}
\mathbb E[(q-q_G)^2]\ge \mathrm{Var}(q_G)=\sigma_\Delta^2.
\end{equation}
Taking \(r=s=2\) recovers Theorem~\ref{thm:product_heisenberg}.

\subsection{Connection to Classical Fairness Impossibility Results}
A closely related precursor appears in the algorithmic fairness literature. Specifically, \citet{chouldechova2017} shows that in binary settings, when two groups have different base rates, \(\pi_g = \mathbb{P}(Y=1\mid  G=g)\), predictive parity, i.e., equal \(\mathrm{PPV}_g = \mathbb{P}(Y=1 \mid \hat{Y} = 1, G=g)\) across groups, cannot generally hold simultaneously with equalized error profiles matching the false positive rates \(\mathrm{FPR}_g = \mathbb{P}(\hat{Y} = 1 \mid Y=0, G=g)\) and the true positive rates \(\mathrm{TPR}_g = \mathbb{P}(\hat{Y} = 1 \mid Y=1, G=g)\). The incompatibility follows from a Bayes coupling identity relating
predictive values, base rates, and the likelihood ratio
\(\mathrm{TPR}_g/\mathrm{FPR}_g\). Expressed in log-odds form, one may rewrite the incompatibility in the following multiplicative form
\begin{equation}
\exp\big(|\Delta\operatorname{logit}(\mathrm{PPV})|\big)
\exp\big(\big|\Delta\log(\mathrm{TPR}/\mathrm{FPR})\big|\big)
\ge
\exp\big(|\Delta\operatorname{logit}(\pi)|\big),
\end{equation}
where \(\operatorname{logit}(x)=\log(x/(1-x))\) and \(\Delta f := f_1 - f_2\).
The two exponential terms quantify deviations from predictive parity and from equalized error rate structure, while the right-hand side depends only on the base rate gap. Thus, in that setting, unequal base rates act as a structural lower bound that prevents fairness deviations from simultaneously vanishing.

\subsection{Practical Plug-in Diagnostics}
\label{app:plugin-diagnostics}
Our experiments use empirical plug-in diagnostics for the heterogeneity scales and effective lower-bound proxies. For these diagnostics, \(\hat q\) and \(\hat q_g\) are computed from the calibration split, whereas \(\hat p_g\), \(\hat F_{S\mid g}\), \(\hat\ell_g\), and the reported RMS distortions are computed on the corresponding test sample, namely the independent test sample in the synthetic simulations and the test split in the real-data experiments. The quantile
heterogeneity can be estimated by
\[
\hat\sigma_\Delta^2
=
\sum_g \hat p_g
\left(\hat q_g-\sum_h \hat p_h\hat q_h\right)^2 .
\]
For any threshold \(t\), let
\[
\hat F_{S\mid g}(t)=\frac{1}{n_g}\sum_{i:G_i=g}\mathbf 1\{S_i\le t\},
\qquad
\hat\ell_g(t)=\frac{1}{n_g}\sum_{i:G_i=g}|\hat C_t(X_i)|
\]
be the empirical group score CDF and empirical set-size curve on this test sample, where \(n_g\) is the number of test examples in group \(g\). Rather than estimating
essential infima, segment-average proxies along the same policy-conversion
paths are:
\[
\hat m_g^{\rm seg}
=
\frac{|\hat F_{S\mid g}(\hat q)-\hat F_{S\mid g}(\hat q_g)|}{|\hat q-\hat q_g|},
\qquad
\hat v_g^{\rm seg}
=
\frac{|\hat\ell_g(\hat q_g)-\hat\ell_g(\hat q)|}{|\hat q_g-\hat q|}.
\]
Next, define
\[
\hat\lambda_g:=\hat\ell_g(\hat q_g),
\qquad
\hat\lambda:=\sum_g\hat p_g\hat\lambda_g,
\qquad
\hat\sigma_\lambda^2
=
\sum_g\hat p_g(\hat\lambda_g-\hat\lambda)^2 .
\]
If \(\hat\tau_g\) is the empirical threshold used to attain the common size target
\(\hat\lambda\), then
\[
\hat\kappa_g^{\rm seg}
=
\frac{|\hat F_{S\mid g}(\hat\tau_g)-\hat F_{S\mid g}(\hat q_g)|}
{|\hat\lambda-\hat\lambda_g|}.
\]
When a denominator is zero, the corresponding segment proxy is set to zero.

Finally, for group-wise coefficients \(a_g\) and gaps \(d_g\), define
\[
\mathcal E(a,d)
=
\left(
\frac{\sum_g\hat p_g a_g^2 d_g^2}
{\sum_g\hat p_g d_g^2}
\right)^{1/2},
\]
which can be used to compute
\[
\widehat m_{\rm eff}^{\rm seg}
=
\mathcal E(\hat m^{\rm seg},\hat q-\hat q_g),
\quad
\widehat v_{\rm eff}^{\rm seg}
=
\mathcal E(\hat v^{\rm seg},\hat q-\hat q_g),
\quad
\widehat\kappa_{\rm eff}^{\rm seg}
=
\mathcal E(\hat\kappa^{\rm seg},\hat\lambda_g-\hat\lambda),
\]
yielding the diagnostic scales
\[
\widehat m_{\rm eff}^{\rm seg}\hat\sigma_\Delta,\qquad
\widehat v_{\rm eff}^{\rm seg}\hat\sigma_\Delta,\qquad
\widehat\kappa_{\rm eff}^{\rm seg}\hat\sigma_\lambda .
\]
These quantities are empirical diagnostic proxy scales for assessing whether group heterogeneity is large enough for coverage or
set-size distortions to be visible.
\FloatBarrier
\section{Proofs of Theoretical Results}
\label{app:proofs}
\subsection*{Proof of Theorem~\ref{thm:conservation}}

\conservation*
\begin{proof}
By definition of \(F_{S}\), \[
\sum_{g\in\mathcal G} p_g\,\varepsilon_g(q)
=\sum_{g\in\mathcal G} p_gF_{S|g}(q)-(1-\alpha)
=F_{S}(q)-(1-\alpha)
=\delta(q)\ge 0,
\]
which yields the generalized conservation law. If \(F_{S}\) is continuous at \(q\), then by definition of the quantile \(F_{S}(t)<1-\alpha\) for all \(t<q\); hence, we have \(F_{S}(q^-)\le 1-\alpha\). The continuity at \(q\) gives
\[
F_{S}(q)=F_{S}(q^-)\le 1-\alpha,
\] 
while definition of \(q\) gives \(F_{S}(q)\ge 1-\alpha\). Thus, we have \(F_{S}(q)=1-\alpha\), i.e., \(\delta(q)=0\).

\end{proof}

\subsection*{Proof of Theorem~\ref{thm:heisenberg_L2}}

\heisenberg*
\begin{proof}

If \(\mathbb{E}[(q-q_G)^2]=0\), then \(\mathrm{Var}(q_G)=0\), so the claim is trivial.
We therefore assume \(\mathbb{E}[(q-q_G)^2]>0\), so that \(m_{\mathrm{eff}}(q)\) is well defined.

Fix a group \(g\in\mathcal G\). Since \(q_g\) is the group-specific \(1-\alpha\) quantile and
\(F_{S|g}\) is continuous on the relevant segment, \(F_{S|g}(q_g)=1-\alpha\). Hence, by
Assumption~\ref{assumption_2},
\[
\varepsilon_g(q)
=
F_{S|g}(q)-F_{S|g}(q_g)
=
\int_{q_g}^{q} f_{S|g}(t)\,dt .
\]
By the definition of \(m_g(q)\),
\[
|\varepsilon_g(q)|
=
\left|\int_{q_g}^{q} f_{S|g}(t)\,dt\right|
\ge
m_g(q)|q-q_g|.
\]
Squaring and averaging over \(G\sim p\) gives
\[
\mathbb{E}[\varepsilon_G(q)^2]
\ge
\mathbb{E}\!\left[m_G(q)^2(q-q_G)^2\right].
\]
By the definition of \(m_{\mathrm{eff}}(q)\),
\[
\mathbb{E}\!\left[m_G(q)^2(q-q_G)^2\right]
=
m_{\mathrm{eff}}(q)^2 \mathbb{E}[(q-q_G)^2].
\]
Moreover,
\[
\mathbb{E}[(q-q_G)^2]
=
\mathrm{Var}(q_G)+(q-\mathbb{E}[q_G])^2
\ge
\mathrm{Var}(q_G).
\]
Finally, Assumption~\ref{assumption_2} implies that the mixture CDF \(F_S\) is continuous at \(q\), so Theorem~\ref{thm:conservation} gives
\[
\mathbb{E}[\varepsilon_G(q)] =
\sum_{g\in\mathcal G}p_g\varepsilon_g(q)=0.
\]
Therefore,
\[
\mathrm{Var}(\varepsilon_G(q)) =\mathbb{E}[\varepsilon_G(q)^2]
\ge
m_{\mathrm{eff}}(q)^2\mathrm{Var}(q_G).
\]

\end{proof}

\subsection*{Proof of Lemma~\ref{lem:size-monotonicity}}
\MonoSetSize*
\begin{proof}
Fix \(t_1\le t_2\) and \(x\in\mathcal X\).
Let \(y\in\hat C_{t_1}(x)\). By definition,
\[
S(x,y)\le t_1.
\]
Since \(t_1\le t_2\), we have \(S(x,y)\le t_2\), and hence
\(y\in\hat C_{t_2}(x)\). Therefore,
\[
\hat C_{t_1}(x)\subseteq\hat C_{t_2}(x),
\]
which implies
\[
|\hat C_{t_1}(x)|\le|\hat C_{t_2}(x)|.
\]
Applying this inequality to \(X\) and taking conditional expectations given
\(G=g\) yields
\[
\ell_g(t_1)=\mathbb E\!\left[|\hat C_{t_1}(X)|\mid G=g\right]\le
\mathbb E\!\left[|\hat C_{t_2}(X)|\mid G=g\right]=
\ell_g(t_2).
\]
\end{proof}
\subsection*{Proof of \Cref{thm:3}}
\SizeDisparity*
\begin{proof}
    The group-wise thresholds $\{q_g \}_{g \in \mathcal{G}}$ achieve equalized coverage at level $1 - \alpha$ across groups. Now fix any $g \in \mathcal{H}_r \setminus \{r\}$. By definition of $\mathcal{H}_r$, we have $q_g \ge q_r$. By Lemma~\ref{lem:size-monotonicity}, 
    \begin{equation*}
        \ell_{g}(q_g) \ge \ell_g(q_r).
    \end{equation*}
    Subtracting $\ell_r(q_r)$ from both sides and using Assumption \ref{insert_assump_1} gives
    \begin{equation*}
        \ell_{g}(q_g) - \ell_r(q_r) \ge \ell_g(q_r) - \ell_r(q_r) \ge c_g > 0.
    \end{equation*}
    This proves the first displayed claim.

    Taking the maximum over $g \in \mathcal{H}_r \setminus \{r\}$ yields
    \begin{equation*}
        \max_{g,g' \in \mathcal{G}} |\ell_g(q_g) - \ell_{g'}(q_{g'})| \ge \max_{g \in \mathcal{H}_r \setminus \{r\}} |\ell_g(q_g) - \ell_r(q_r)| \ge \max_{g \in \mathcal{H}_r \setminus \{r\}} c_g > 0.
    \end{equation*}
    Thus, equalized expected set size across all groups is impossible under the equalized group-wise coverage policy.

    For the mean square claim, using the previously established bound for each $g \in \mathcal{H}_r \setminus \{r\}$, 
    \begin{equation*}
        (\ell_g(q_g) - \ell_r(q_r))^2 \ge c_g^2.
    \end{equation*}
    Multiplying by $p_g$ and summing over $g \in \mathcal{H}_r \setminus \{r\}$ gives
    \begin{equation*}
        D_r^2 = \sum_{g \in \mathcal{H}_r \setminus \{r\}} p_g(\ell_g(q_g) - \ell_r(q_r))^2 \ge \sum_{g \in \mathcal{H}_r \setminus \{r\}} p_g c_g^2 > 0.
    \end{equation*}
    This proves the claim.
\end{proof}

\subsection*{Proof of Theorem~\ref{thm:size_to_cov_disparity}}
\CovDisparity*
\begin{proof}
Fix a group \(g\in\mathcal G\).
If \(\lambda_g < \lambda\), then, since \(\ell_g\) is monotone on the segment between \(q_g\) and
\(\tau_g\) and \(\ell_g(q_g)=\lambda_g\), the identity \(\ell_g(\tau_g)=\lambda\) implies
\[
\tau_g > q_g.
\]
By absolute continuity and Assumption \ref{assumption_6},
\[
\lambda-\lambda_g
=
\ell_g(\tau_g)-\ell_g(q_g)
=
\int_{q_g}^{\tau_g}\ell_g'(t)\,dt.
\]
Hence
\[
\lambda-\lambda_g
\le
\int_{q_g}^{\tau_g} |\ell_g'(t)|\,dt
\le
V_g(\tau_g-q_g),
\]
so that
\[
\tau_g-q_g \ge \frac{\lambda-\lambda_g}{V_g}.
\]
Again by Assumption \ref{assumption_6},
\[
F_{S|g}(\tau_g)-F_{S|g}(q_g)
=
\int_{q_g}^{\tau_g} f_{S|g}(t)\,dt
\ge
m_g(\tau_g-q_g)
\ge
\frac{m_g}{V_g}(\lambda-\lambda_g)
=
\kappa_g(\lambda-\lambda_g).
\]
Since \(F_{S|g}(q_g)=1-\alpha\), this proves
\[
F_{S|g}(\tau_g)-(1-\alpha)\ge \kappa_g(\lambda-\lambda_g).
\]

Now suppose \(\lambda_{g'} > \lambda\). Since
\(\ell_{g'}(q_{g'})=\lambda_{g'}\) and \(\ell_{g'}(\tau_{g'})=\lambda\), we have
\[
\ell_{g'}(q_{g'})>\ell_{g'}(\tau_{g'}).
\]
Because \(\ell_{g'}\) is non-decreasing on the segment between \(q_{g'}\) and \(\tau_{g'}\), this implies
\[
\tau_{g'}<q_{g'}.
\]
Indeed, if \(\tau_{g'}\ge q_{g'}\), monotonicity would imply
\[
\ell_{g'}(\tau_{g'})\ge \ell_{g'}(q_{g'}),
\]
contradicting \(\ell_{g'}(\tau_{g'})<\ell_{g'}(q_{g'})\).
Then
\[
\lambda_{g'}-\lambda
=
\ell_{g'}(q_{g'})-\ell_{g'}(\tau_{g'})
=
\int_{\tau_{g'}}^{q_{g'}}\ell_{g'}'(t)\,dt
\le
\int_{\tau_{g'}}^{q_{g'}} |\ell_{g'}'(t)|\,dt
\le
V_{g'}(q_{g'}-\tau_{g'}),
\]
hence
\[
q_{g'}-\tau_{g'} \ge \frac{\lambda_{g'}-\lambda}{V_{g'}}.
\]
Using Assumption \ref{assumption_6} again,
\[
F_{S|g'}(q_{g'})-F_{S|g'}(\tau_{g'})
=
\int_{\tau_{g'}}^{q_{g'}} f_{S|g'}(t)\,dt
\ge
m_{g'}(q_{g'}-\tau_{g'})
\ge
\frac{m_{g'}}{V_{g'}}(\lambda_{g'}-\lambda)
=
\kappa_{g'}(\lambda_{g'}-\lambda).
\]
Since \(F_{S|g'}(q_{g'})=1-\alpha\), we obtain
\[
(1-\alpha)-F_{S|g'}(\tau_{g'})\ge \kappa_{g'}(\lambda_{g'}-\lambda).
\]

Finally, if \(\lambda_g < \lambda < \lambda_{g'}\), then combining the two inequalities gives
\[
F_{S|g}(\tau_g)-F_{S|g'}(\tau_{g'})
=
\bigl(F_{S|g}(\tau_g)-(1-\alpha)\bigr)
+
\bigl((1-\alpha)-F_{S|g'}(\tau_{g'})\bigr)
\ge
\kappa_g(\lambda-\lambda_g)+\kappa_{g'}(\lambda_{g'}-\lambda).
\]
Because \(\lambda_g<\lambda<\lambda_{g'}\), both terms on the right-hand side are strictly positive, and hence the right-hand side is \(>0\). Therefore
\[
\max_{a,b\in\mathcal G}|F_{S|a}(\tau_a)-F_{S|b}(\tau_b)|
\ge
F_{S|g}(\tau_g)-F_{S|g'}(\tau_{g'})
>0,
\]
which proves that coverage cannot be equalized across groups under this equalized-size policy.
\end{proof}

\subsection*{Proof of Corollary~\ref{cor:length_disparity}}
\SizeDistortion*
\begin{proof}
Let \(g\) be fixed. By Assumption~\ref{assumption_5} and the fundamental theorem of calculus, we have
\[
\ell_g(q_g)-\ell_g(q)=\int_{q}^{q_g} \ell_g'(t)\,dt.
\]
Under Assumption~\ref{assumption_5}, \(|\ell_g'(t)|\ge v_g\) along the segment between \(q\) and \(q_g\), hence
\(|\ell_g(q_g)-\ell_g(q)|\ge v_g|q_g-q|\).
Squaring and averaging gives
\(
\mathbb{E}[(\ell_G(q_G)-\ell_G(q))^2]\ge \mathbb{E}[v_G^2 (q_G-q)^2].
\)
Taking square roots and using the definition of \(v_{\mathrm{eff}}(q)\) gives
\(
\sqrt{E[(\ell_G(q_G)-\ell_G(q))^2]}
\ge
v_{\mathrm{eff}}(q)\sqrt{E[(q_G-q)^2]}.
\)
Since \(E[(q_G-q)^2]\ge \mathrm{Var}(q_G)=\sigma_\Delta^2\), the desired bound follows.
\end{proof}

\subsection*{Proof of Corollary~\ref{cor:coverage_disparity}}
\CovDistortion*
\begin{proof}
We fix group \(g\). Since %
\(\ell_g(q_g)=\lambda_g\) by definition, we have 
\[
\lambda - \lambda_g = \ell_g(\tau_g) -\ell_g(q_g) = \int_{q_g}^{\tau_g}\ell_g'(t)dt\ .
\]
Therefore, \(|\lambda-\lambda_g| \le V_g|\tau_g-q_g|\), which implies
\[
|\tau_g-q_g| \ge \frac{|\lambda-\lambda_g|}{V_g}\ .
\]
Moreover,
\[
|F_{S|g}(\tau_g)-F_{S|g}(q_g)| = \Bigg|\int_{q_g}^{\tau_g} f_{S|g}(t)dt  \Bigg| \ge m_g |\tau_g-q_g| \ge m_g \frac{|\lambda-\lambda_g|}{V_g}.
\]
Squaring, averaging, taking square roots, and using the definition of
\(\kappa_{\mathrm{eff}}(\lambda)\) gives
\begin{equation}
\sqrt{E[(F_{S|G}(\tau_G)-F_{S|G}(q_G))^2]}
\ge
\kappa_{\mathrm{eff}}(\lambda)\sqrt{E[(\lambda-\lambda_G)^2]}.
\end{equation}
Since \(E[(\lambda-\lambda_G)^2]\ge \mathrm{Var}(\lambda_G)=\sigma_\lambda^2\), the desired bound follows.
\end{proof}

\subsection*{Proof of Corollary~\ref{cor:direct_pooled_ess}}
\DirectPooledESSCoverage*
\begin{proof}
For each group \(g\),
\begin{equation}
\left|\lambda-\ell_g(q)\right|=\left|\int_q^{\tau_g}\ell_g'(t)\,dt\right|
\le V_g^\circ |\tau_g-q|.
\end{equation}
Also,
\begin{equation}
\left|F_{S|g}(\tau_g)-F_{S|g}(q)\right|=\left|\int_q^{\tau_g} f_{S|g}(t)\,dt\right|
\ge
m_g^\circ |\tau_g-q|.
\end{equation}
Combining the two inequalities gives
\begin{equation}
\left|F_{S|g}(\tau_g)-F_{S|g}(q)\right|
\ge
\kappa_g^\circ\left|\lambda-\ell_g(q)\right|.
\end{equation}
Squaring and averaging over \(G\) gives the RMS bound.
\end{proof}

\subsection*{Proof of Lemma~\ref{lem:two_group_quantile_interval}}
\TwoGroupQuantileInterval*
\begin{proof}
Assume $q < q_0$ and prove by contradiction: Given $q < q_0$, we have $q < q_1$ since $q_0 < q_1$. According to the definition of $q_0$ and $q_1$, we must have $F_0(q) < 1 - \alpha$ and $F_1(q) < 1 - \alpha$, so
    \begin{equation*}
        F_{S}(q)=(1-p)F_0(q)+pF_1(q) < (1-p)(1 - \alpha) + p(1 - \alpha) = 1 - \alpha.
    \end{equation*}
    However, the above derivation of $F_{S}(q) < 1 - \alpha$ contradicts the definition of $q$, which requires $F_{S}(q) \ge 1 - \alpha$. Therefore, we must have $q \ge q_0$. Regarding $q_1$, we have
    \begin{equation*}
        F_{S}(q_1)=(1-p)F_0(q_1)+pF_1(q_1) \ge (1-p)(1 - \alpha) + p(1 - \alpha) = 1 - \alpha.
    \end{equation*}
    Thus, $q_1 \in \{t \in \mathbb{R}: F_{S}(t) \ge 1 - \alpha \}$. Because $q = \inf \{t \in \mathbb{R}: F_{S}(t) \ge 1 - \alpha \}$, we have $q \le q_1$. Together with $q \ge q_0$, we have $q \in [q_0, \, q_1]$.
\end{proof}

\subsection*{Proof of Theorem~\ref{thm:two_group_product}}
\TwoGroupProduct*
\begin{proof}
Let \(d_0:=q-q_0\ge 0\) and \(d_1:=q_1-q\ge 0\), so \(d_0+d_1=\Delta\).
By the fundamental theorem of calculus and Assumption~\ref{assumption_1},
\[
 \omega_o=F_0(q)-F_0(q_0)=\int_{q_0}^{q} f_0(t)\,dt \ge m_0 d_0,
\qquad
 \omega_u=F_1(q_1)-F_1(q)=\int_{q}^{q_1} f_1(t)\,dt \ge m_1 d_1.
\]
The relationship \eqref{eq:uo_coupling} gives \((1-p)m_0 d_0 \le (1-p)\omega_o = p \omega_u\), while \(\omega_u\ge m_1 d_1\).
A convenient way to combine the constraints is to express \(d_0\) in terms of \(\omega_u\):
\((1-p)\omega_o=p \omega_u\) and \(\omega_o\ge m_0 d_0\) imply \((1-p)m_0 d_0\le p \omega_u\), so
\(d_0\le \frac{p}{(1-p)m_0}\omega_u\).
Similarly, \(\omega_u\ge m_1 d_1\) implies \(d_1\le \frac{1}{m_1}\omega_u\).
Since \(\Delta=d_0+d_1\), we obtain
\[
\Delta \le \Big(\frac{1}{m_1}+\frac{p}{1-p}\frac{1}{m_0}\Big)\omega_u,
\]
which yields \eqref{eq:two_group_u_bound}.
Then \eqref{eq:two_group_o_bound} follows from \eqref{eq:uo_coupling}.
Multiplying \eqref{eq:two_group_u_bound} and \eqref{eq:two_group_o_bound} gives \eqref{eq:two_group_product}.
\end{proof}

\subsection*{Proof of Theorem~\ref{thm:multigroupproduct}}
\MultiGroupProduct*
\begin{proof}
For any group with \(q \ge q_g\), we have \(\varepsilon_g(q)\ge m_g(q-q_g)\). Similarly, for any group with \(q \le q_g\), we have \(\varepsilon_g(q) \le m_g (q-q_g)\). Therefore, we have
\begin{equation}
    \Omega_{o}(q) \ge \sum_g w_g(q-q_g)_+=:A_+(q) \qquad \Omega_{u}(q) \ge \sum_g w_g(q_g-q)_+=:A_-(q).
\end{equation}
At the crossing point \(\bar{q}_m\),
\begin{equation}
A_+(\bar{q}_m)=A_-(\bar{q}_m)=\frac{1}{2}\sum_gw_g|q_g-\bar{q}_m| = B_K.
\end{equation}
Since \(A_+(q)-A_-(q)=\sum_{g\in\mathcal{G}} w_g(q-q_g)\), the two functions cross at \(\bar{q}_m\). Moreover, \(A_+(q)\) is non-decreasing in \(q\) and \(A_-(q)\) is non-increasing in \(q\). Hence,  the minimum of \(\max\{A_+(q),A_-(q)\}\) is attained at the crossing point, which equals \(B_K\).
Thus, we obtain \(\max\{\Omega_{o}(q),\Omega_{u}(q)\}\ge B_K\). Furthermore, Theorem~\ref{thm:conservation} gives \(\Omega_{o}(q)-\Omega_{u}(q) = \delta(q)\). Since
\(\max\{\Omega_o(q),\Omega_u(q)\}-
\min\{\Omega_o(q),\Omega_u(q)\}=
|\delta(q)|,
\)
we also have
\(
\min\{\Omega_o(q),\Omega_u(q)\}
\ge
\bigl(B_K-|\delta(q)|\bigr)_+ .
\)
Therefore,
\begin{equation}
\Omega_o(q)\Omega_u(q)
\ge
B_K\bigl(B_K-|\delta(q)|\bigr)_+ .
\end{equation}
Under exact conservation, \(\delta(q)=0\), and hence
\begin{equation}
    \Omega_{o}(q) \ \Omega_{u}(q) \ge B_K^2.
\end{equation}
\end{proof}

\subsection*{Proof of Theorem~\ref{thm:product_heisenberg}}
\HeisenbergProduct*
\begin{proof}
Let \(Q:=q-q_G\).
For groups with \(q\ne q_g\), \eqref{eq:fbar} implies \(\varepsilon_G(q)=\bar f_{S|G}(q)\cdot Q\) almost surely.
Then
\(
\mathbb{E}[Q^2]=\mathbb{E}[\varepsilon_G(q)\cdot \frac{Q}{\bar f_{S|G}(q)}].
\)
Applying Cauchy--Schwarz yields
\(
\mathbb{E}[Q^2]^2
\le \mathbb{E}[\varepsilon_G(q)^2] \cdot \mathbb{E}[(\frac{Q}{\bar f_{S|G}(q)})^2].
\)
Taking square roots gives the first inequality in \eqref{eq:product_heisenberg}.
The second inequality is identical to the argument in Theorem~\ref{thm:heisenberg_L2}.
\end{proof}
\FloatBarrier
\section{Additional Synthetic Experimental Details}
\label{app:additional-experiments}

\subsection{Synthetic Simulations}
\label{app:synthetic}

In this experiment, we directly synthesize nonconformity scores, rather than computing them from data and a model. For each group, we estimate the split-conformal threshold by the empirical \(1-\alpha\) quantile of the calibration scores. The pooled threshold \(\hat q\) is obtained by concatenating all calibration scores across subgroups and applying the same quantile rule. Empirical coverage quantities are then evaluated on an independent test sample, so the observed distortion reflects finite-sample threshold estimation.

Figure~\ref{fig:multigroup-app} provides additional pooled-threshold validation.  Across the two-group sweep and the four multi-group families, the empirical RMS miscoverage increases with heterogeneity and remains above the oracle scale \(m_{\mathrm{eff}}(q)\sigma_\Delta\). Table~\ref{tab:summary} reports the empirical-oracle comparison at the largest heterogeneity value in each setting. Table~\ref{tab:setup_pooled} provides details about the distributions of the simulations.

\begin{figure}[t]
    \centering
    \includegraphics[width=\linewidth,trim={10 3 10 6},clip]{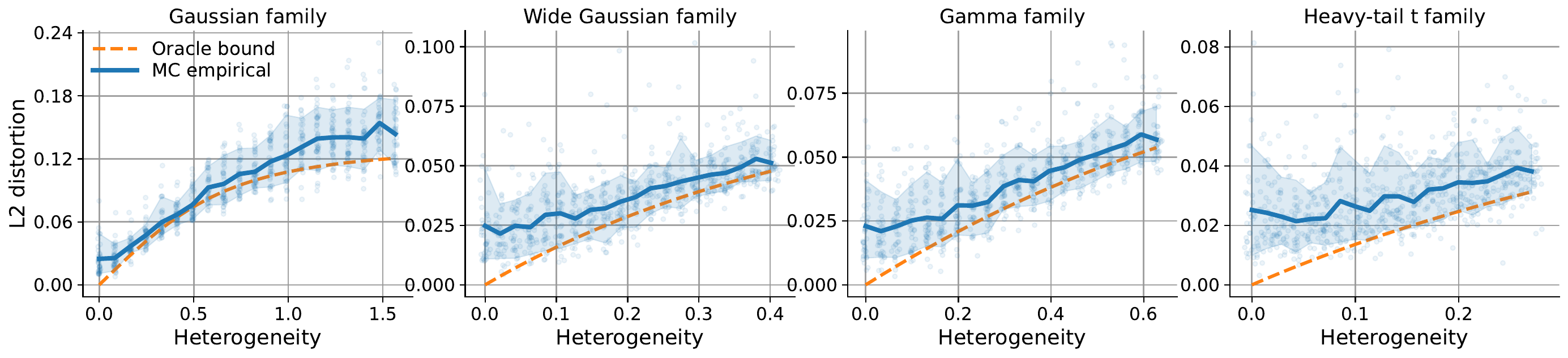}
    \caption{Four multi-group families. Across all four score families, the empirical RMS miscoverage increases with heterogeneity and remains above the lower-bound scale from Theorem~\ref{thm:heisenberg_L2}.}
    \label{fig:multigroup-app}
\end{figure}

For the set-size based experiments, we use the monotone linear proxy \(\ell_g(t)=a_g+b_gt\), with \(a_g=2+0.15(g-1)\) and \(b_g=0.8+0.1(g-1)\) after labeling the group index as \(g=1,\dots,K\). In the size-to-coverage experiment, the common target size is the weighted average of the coverage-calibrated group sizes, namely \(\hat\lambda=\sum_g p_g \ell_g(\hat q_g)\), and the equalized-size threshold \(\hat\tau_g\) solves \(\ell_g(\hat\tau_g)=\hat\lambda\).

\begin{table}
    \centering
    \small
    \caption{Summary of empirical-oracle comparison at the largest heterogeneity value in each setting (means over \(40\) seeds). Ratios are computed from unrounded means.}
    \label{tab:summary}
    \setlength{\tabcolsep}{4.2pt}
    \begin{tabular}{lccc}
    \toprule
    Settings & Empirical & Oracle & Ratio \\
    \midrule
    Two-group Gaussian & 0.097 & 0.085 & 1.141 \\
    Four-group Gaussian & 0.144 & 0.121 & 1.190 \\
    Four-group wide Gaussian & 0.051 & 0.048 & 1.063 \\
    Four-group Gamma & 0.057 & 0.054 & 1.056 \\
    Four-group heavy-tail $t$ & 0.038 & 0.031 & 1.226 \\
    Coverage$\to$size  & 2.860 & 2.127 & 1.345 \\
    Size$\to$coverage  & 0.352 & 0.347 & 1.014 \\
    \bottomrule
    \end{tabular}
\end{table}

\begin{table*}[t]
\centering
\small
\caption{Simulation setups for the pooled-threshold experiments in \Cref{fig:two-group-app} and \Cref{fig:multigroup-app}. All groups have equal weight \(p_g=1/K\), all experiments use \(\alpha=0.1\), and the empirical miscoverage is compared against the theoretical scale \(m_{\mathrm{eff}}(q)\sigma_\Delta\).}
\label{tab:setup_pooled}
\setlength{\tabcolsep}{4pt}
\begin{tabular}{@{}p{0.30\textwidth}p{0.40\textwidth}c@{}}
\hline
\textbf{Setting} & \textbf{Score family} & \textbf{Cal/Test} \\
\hline
Two-group Gaussian shift
& \(S \mid G=g \sim \mathcal{N}(\mu_g,1)\), with \((\mu_0,\mu_1)=(0,\delta)\) and \(\delta \in [0,2]\)
& 50/500 \\

Four-group Gaussian shift
& \(S \mid G=g \sim \mathcal{N}(\mu_g,1)\), with \(\mu_g=s(-1.5,-0.5,0.5,1.5)\) and \(s \in [0,1.4]\)
& 50/500 \\

Four-group wide Gaussian
& \(S \mid G=g \sim \mathcal{N}(0,\sigma_g^2)\), with \(\sigma_g=1+s(0,0.2,0.4,0.6)\) and \(s \in [0,1.4]\)
& 50/500 \\

Four-group Gamma
& \(S \mid G=g \sim \mathrm{Gamma}(4,\theta_g)\), with \(\theta_g=0.55+s(0,0.06,0.12,0.18)\) and \(s \in [0,1.4]\)
& 50/500 \\

Four-group heavy-tail \(t\)
& \(S \mid G=g \sim t_6(0,\sigma_g)\), with \(\sigma_g=0.9+s(0,0.12,0.24,0.36)\) and \(s \in [0,1.4]\)
& 50/500 \\
\hline
\end{tabular}
\end{table*}

\begin{table*}[t]
\centering
\small
\caption{Simulation setups for the group-adjusted trade-off experiments in Figure~\ref{fig:policy-conversion}. All groups have equal weight \(p_g=1/K\) and all experiments use \(\alpha=0.1\).}
\label{tab:setup_tradeoff}
\setlength{\tabcolsep}{4pt}
\begin{tabular}{@{}p{0.17\textwidth}p{0.30\textwidth}cp{0.24\textwidth}@{}}
\hline
\textbf{Setting} & \textbf{Score family} & \textbf{Cal/Test} & \textbf{Compared quantity} \\
\hline
Coverage vs. size
& Two-component asymmetric Gaussian mixtures with group-specific offsets, scales, and weights; \(s \in [0,1.8]\) scales group centers and increases quantile separation
& 25/100
& size change vs. \(v_{\mathrm{eff}}(q)\sigma_\Delta\) \\

Size vs. coverage
& Two-component mixtures of one Student-\(t\) component and one Gaussian component, with group-specific offsets, scales, and weights; \(s \in [0,2.0]\) scales group centers
& 25/100
& coverage change vs. \(\kappa_{\mathrm{eff}}(\lambda)\sigma_\lambda\) \\
\hline
\end{tabular}
\end{table*}

\subsection{Imbalance Bottleneck and Ratio Diagnostics}
\label{app:imbalance-diagnosis}

As an additional diagnostic beyond the balanced synthetic experiments, we perform pooled-threshold simulations in four imbalanced settings: Gaussian, wide Gaussian, gamma, and Student-\(t\) distributions. Each setting has four distributions of the same type but with imbalanced group masses \(p=(0.60,0.25,0.10,0.05)\). We use a total calibration budget of \(n_{\mathrm{cal}}=400\), \(4000\) test points per group, target level \(\alpha=0.1\), and \(40\) Monte Carlo seeds. Figure~\ref{fig:imbalance-diagnostics} shows that the RMS pooled distortion remains above the oracle scale of all four families, whereas the minority-group absolute gap grows more rapidly with heterogeneity. To verify that the bottleneck is indeed concentrated on the rarest group, Table~\ref{tab:minority-largest-by-sweep} reports how often the minority group is the largest-gap group for each heterogeneity level. The entries report counts out of 40 seeds for which the minority group attains the largest absolute gap. The frequency rises rapidly with heterogeneity and reaches \(40/40\) at the largest heterogeneity of \(2.0\). This pattern indicates a clear minority bottleneck under pooled calibration.

\begin{figure}[t]
    \centering
    \includegraphics[width=0.80\linewidth,trim={8 3 8 6},clip]{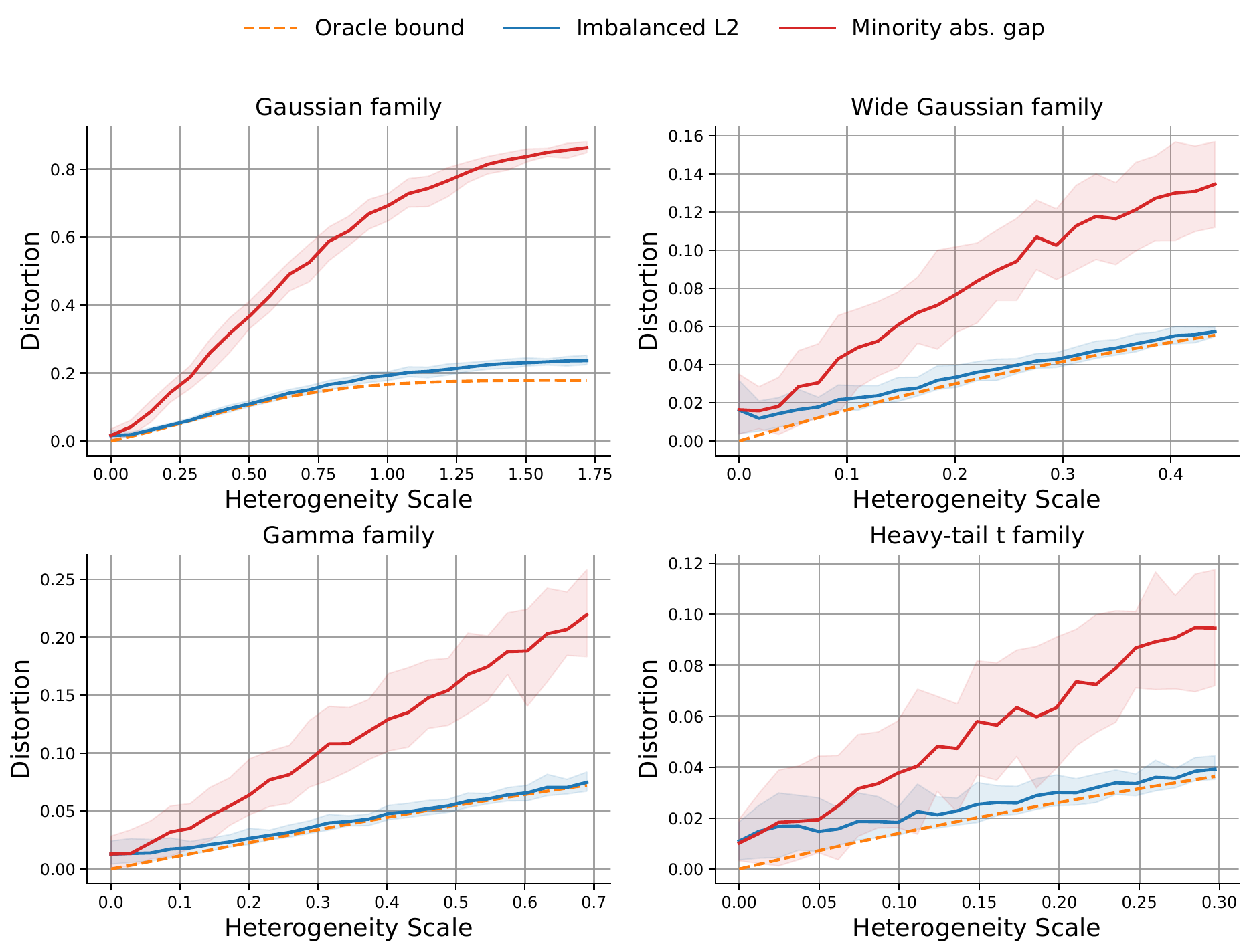}
    \caption{Imbalanced four-group pooled-threshold diagnostics. Across all four families, the weighted RMS distortion remains above the lower-bound scale from Theorem~\ref{thm:heisenberg_L2}. At the same time, the minority-group absolute gap grows more rapidly with heterogeneity, showing that pooled distortion can concentrate on rare groups.}
    \label{fig:imbalance-diagnostics}
\end{figure}

\begin{table}
    \centering
    \small
    \caption{Minority frequency across heterogeneity levels.}
    \label{tab:minority-largest-by-sweep}
    \begin{tabular}{lccccc}
    \toprule
    & \multicolumn{5}{c}{Heterogeneity level} \\
    \cmidrule(lr){2-6}
    Family & 0.0 & 0.5 & 1.0 & 1.5 & 2.0 \\
    \midrule
    Gaussian & 13/40 & 40/40 & 40/40 & 40/40 & 40/40 \\
    Wide Gaussian & 13/40 & 36/40 & 39/40 & 40/40 & 40/40 \\
    Gamma & 13/40 & 37/40 & 40/40 & 40/40 & 40/40 \\
    Heavy-tail $t$ & 9/40 & 30/40 & 39/40 & 37/40 & 40/40 \\
    \bottomrule
    \end{tabular}
\end{table}

Tables~\ref{tab:coverage-to-size-ratios} and~\ref{tab:size-to-coverage-ratios} record empirical and oracle ratios for the two policy-conversion diagnostics at selected heterogeneity levels. The coverage-to-size ratio remains above one throughout, with the finite-sample slack shrinking as \(n_{\mathrm{cal}}\) grows. The size-to-coverage ratio is tighter and concentrates near one much more quickly.

\begin{table}
    \centering
    \small
    \caption{Coverage-to-size empirical/oracle ratio at heterogeneity levels \(\sigma_\Delta\).}
    \label{tab:coverage-to-size-ratios}
    \begin{tabular*}{0.8\linewidth}{@{\extracolsep{\fill}}lccccc@{}}
    \toprule
    & \multicolumn{5}{c}{Selected heterogeneity level \(\sigma_\Delta\)} \\
    \cmidrule(lr){2-6}
    \(n_{\mathrm{cal}}\) & 0.378 & 0.855 & 1.250 & 1.652 & 2.192 \\
    \midrule
    12 & 1.837 & 1.417 & 1.327 & 1.334 & 1.392 \\
    25 & 1.443 & 1.174 & 1.243 & 1.244 & 1.296 \\
    50 & 1.284 & 1.174 & 1.178 & 1.232 & 1.287 \\
    100 & 1.226 & 1.087 & 1.172 & 1.227 & 1.305 \\
    200 & 1.086 & 1.085 & 1.161 & 1.210 & 1.298 \\
    400 & 1.034 & 1.068 & 1.139 & 1.218 & 1.293 \\
    \bottomrule
    \end{tabular*}
\end{table}
\begin{table}
    \centering
    \small
    \caption{Size-to-coverage empirical/oracle ratio at heterogeneity levels \(\sigma_\Delta\).}
    \label{tab:size-to-coverage-ratios}
    \begin{tabular*}{0.8\linewidth}{@{\extracolsep{\fill}}lccccc@{}}
    \toprule
    & \multicolumn{5}{c}{Selected heterogeneity level \(\sigma_\Delta\)} \\
    \cmidrule(lr){2-6}
    \(n_{\mathrm{cal}}\) & 0.437 & 0.741 & 1.201 & 1.862 & 2.369 \\
    \midrule
    12  & 1.313 & 1.181 & 1.028 & 1.087 & 1.038 \\
    25  & 1.095 & 1.025 & 1.009 & 0.990 & 1.012 \\
    50  & 0.981 & 0.994 & 0.997 & 1.001 & 0.994 \\
    100 & 1.002 & 0.997 & 0.987 & 1.003 & 1.002 \\
    200 & 1.019 & 1.009 & 1.003 & 0.999 & 1.004 \\
    400 & 1.017 & 1.002 & 0.996 & 1.004 & 1.004 \\
    \bottomrule
    \end{tabular*}
\end{table}

\section{\texttt{Bias in Bios} Experiments}
\label{app:biobias}
\subsection{Experimental Settings}
We perform conformal analysis on outputs from a \texttt{DistilBERT}\footnote{We use the Hugging Face checkpoint \texttt{distilbert-base-uncased}; model card: \url{https://huggingface.co/distilbert/distilbert-base-uncased}. The model card lists the license as Apache-2.0.} classifier trained on the ten most frequent professions from \texttt{Bias in Bios}\footnote{We use the Hugging Face dataset \texttt{LabHC/bias\_in\_bios}; dataset card: \url{https://huggingface.co/datasets/LabHC/bias_in_bios}. The dataset card lists the license as MIT. For the original dataset source, see ~\citep{dearteaga2019biasinbios}.}.
Training uses cleaned, preprocessed text with a maximum sequence length of 160, a learning rate of \(2\times 10^{-5}\), weight decay of \(0.01\), batch sizes of \(64/32\) for training/validation, warm-up ratio \(0.05\), and two training epochs. The validation accuracy and macro-F1 score are \(0.8886\) and \(0.8569\), respectively. The dataset is partitioned into training, validation, calibration, and test splits of sizes \(185758\), \(20640\), \(31764\), and \(79397\), with stable group balance across splits. Details are provided in Table~\ref{tab:split-summary}. Throughout, we use the simple score in the main presentation, while SAPS and RAPS are included to show that the same conclusions are not specific to a single score construction. Detailed per-group values for the simple score are shown in Table~\ref{tab:biobias-main}.
\begin{table}[htbp]
\centering
\small
\caption{Summary for \texttt{Bias in Bios} at $\alpha=0.10$ for the simple score.}
\label{tab:biobias-main}
\begin{tabular}{lcc}
\toprule
Quantity & Group (Male) & Group (Female) \\
\midrule
$p_g$ & 0.5275 & 0.4725 \\
$q_g$ & 0.6298 & 0.6531 \\
${\varepsilon}_g(q)$ & 0.0020 & -0.0009 \\
$\lambda_g=\ell_g(q_g)$ & 1.0140 & 1.0231 \\
$\ell_g(q_g)-\ell_g(q)$ & -0.0050 & 0.0053 \\
$\tau_g$ & 0.6410 & 0.6430 \\
$\hat{F}_{S\mid g}(\tau_g)-\hat{F}_{S\mid g}(q_g)$ & 0.0016 & -0.0018 \\
\bottomrule
\end{tabular}
\end{table}

\begin{table}[t]
\centering
\small
\caption{Split sizes and group composition for \texttt{Bias in Bios}.}
\label{tab:split-summary}
\begin{tabular*}{0.90\linewidth}{@{\extracolsep{\fill}}lccc@{}}
\toprule
Split & Total \(n\) & Group 0 (Male) & Group 1 (Female) \\
\midrule
Model train & 185,758 & 97,985 (52.75\%) & 87,773 (47.25\%) \\
Model validation   & 20,640  & 10,886 (52.74\%) & 9,754 (47.26\%) \\
Calibration & 31,764  & 16,755 (52.75\%) & 15,009 (47.25\%) \\
Test        & 79,397  & 41,881 (52.75\%) & 37,516 (47.25\%) \\
\bottomrule
\end{tabular*}
\end{table}

\subsection{Robustness across Target Coverage}
We vary the target miscoverage level over \(\alpha\in\{0.05, 0.07, 0.085, 0.10\}\) and keep the rest of the pipeline unchanged across three score families. The simple score gives the baseline pattern, while the same trade-off is preserved under SAPS and RAPS. Tables~\ref{tab:alpha-robustness}, \ref{tab:alpha-robustness-saps} and \ref{tab:alpha-robustness-raps} report the \(\alpha\)-robustness summaries for simple, SAPS and RAPS.

\begin{table}[t]
    \centering
    \small
    \caption{\(\alpha\)-robustness for the simple score.}
    \label{tab:alpha-robustness}
    \begin{tabular*}{0.80\linewidth}{@{\extracolsep{\fill}}ccccc@{}}
    \toprule
    \(\alpha\) & \(\sigma_\Delta\) & \shortstack[c]{RMS pooled\\coverage} & \shortstack[c]{RMS size deviation\\from \(q_g\)} & \shortstack[c]{RMS coverage\\after equalized size} \\
    \midrule
    0.050 & \(4.60 \times 10^{-5}\) & 0.0021 & 0.0002 & 0.0006 \\
    0.070 & 0.0019 & 0.0005 & 0.0017 & 0.0014 \\
    0.085 & 0.0054 & 0.0005 & 0.0035 & 0.0016 \\
    0.100 & 0.0116 & 0.0015 & 0.0051 & 0.0017 \\
    \bottomrule
    \end{tabular*}
\end{table}

\begin{table}[t]
    \centering
    \small
    \caption{\(\alpha\)-robustness for SAPS.}
    \label{tab:alpha-robustness-saps}
    \begin{tabular*}{0.80\linewidth}{@{\extracolsep{\fill}}ccccc@{}}
    \toprule
    \(\alpha\) & \(\sigma_\Delta\) & \shortstack[c]{RMS pooled\\coverage} & \shortstack[c]{RMS size deviation\\from \(q_g\)} & \shortstack[c]{RMS coverage\\after equalized size} \\
    \midrule
    0.050 & 0.0039 & 0.0011 & 0.0047 & 0.0005 \\
    0.070 & 0.0022 & 0.0009 & 0.0019 & 0.0008 \\
    0.085 & 0.0015 & 0.0007 & 0.0009 & 0.0013 \\
    0.100 & 0.0162 & 0.0011 & 0.0049 & 0.0018 \\
    \bottomrule
    \end{tabular*}
\end{table}

\begin{table}[t]
    \centering
    \small
    \caption{\(\alpha\)-robustness for RAPS.}
    \label{tab:alpha-robustness-raps}
    \begin{tabular*}{0.80\linewidth}{@{\extracolsep{\fill}}ccccc@{}}
    \toprule
    \(\alpha\) & \(\sigma_\Delta\) & \shortstack[c]{RMS pooled\\coverage} & \shortstack[c]{RMS size deviation\\from \(q_g\)} & \shortstack[c]{RMS coverage\\after equalized size} \\
    \midrule
    0.050 & 0.0055 & 0.0018 & 0.0121 & 0.0002 \\
    0.070 & 0.0005 & 0.0013 & 0.0008 & 0.0013 \\
    0.085 & 0.0007 & 0.0003 & 0.0006 & 0.0006 \\
    0.100 & 0.0276 & 0.0010 & 0.0083 & 0.0019 \\
    \bottomrule
    \end{tabular*}
\end{table}

\subsection{Controlled Group Temperature Sweep}
In this section, we construct a controlled heterogeneity view at the fixed target level \(\alpha=0.1\) by temperature scaling the class probability vector of group~1 through a specific temperature \(T\in\{1.00,1.10,1.25,1.50,1.75,2.00\}\) and renormalizing \(p^{1/T}\). Adjusting the temperature parameter amplifies the same mechanism shown at the baseline. Parallel to the previous section, we present the simple, SAPS, and RAPS scores under temperature sweep, showing a score-robust conversion pattern in Tables~\ref{tab:temperature-sweep}, \ref{tab:temperature-sweep-saps}, and \ref{tab:temperature-sweep-raps}.

\begin{table}[t]
    \centering
    \small
    \caption{\texttt{Bias in Bios} temperature sweep for the simple score at \(\alpha=0.10\).}
    \label{tab:temperature-sweep}
    \begin{tabular*}{0.80\linewidth}{@{\extracolsep{\fill}}ccccc@{}}
    \toprule
    \(T\) & \(\sigma_\Delta\) & \shortstack[c]{RMS pooled\\coverage} & \shortstack[c]{RMS size\\from \(q_g\)} & \shortstack[c]{RMS coverage\\after equalized size} \\
    \midrule
    1.00 & 0.0116 & 0.0015 & 0.0051 & 0.0017 \\
    1.10 & 0.0136 & 0.0018 & 0.0061 & 0.0015 \\
    1.25 & 0.0169 & 0.0027 & 0.0083 & 0.0015 \\
    1.50 & 0.0241 & 0.0044 & 0.0126 & 0.0015 \\
    1.75 & 0.0346 & 0.0067 & 0.0205 & 0.0022 \\
    2.00 & 0.0432 & 0.0088 & 0.0275 & 0.0028 \\
    \bottomrule
    \end{tabular*}
\end{table}

\begin{table}[t]
    \centering
    \small
    \caption{\texttt{Bias in Bios} temperature sweep for SAPS at \(\alpha=0.10\).}
    \label{tab:temperature-sweep-saps}
    \begin{tabular*}{0.80\linewidth}{@{\extracolsep{\fill}}ccccc@{}}
    \toprule
    \(T\) & \(\sigma_\Delta\) & \shortstack[c]{RMS pooled\\coverage} & \shortstack[c]{RMS size\\from \(q_g\)} & \shortstack[c]{RMS coverage\\after equalized size} \\
    \midrule
    1.00 & 0.0162 & 0.0011 & 0.0049 & 0.0018 \\
    1.10 & 0.0058 & 0.0009 & 0.0018 & 0.0020 \\
    1.25 & 0.0086 & 0.0015 & 0.0024 & 0.0020 \\
    1.50 & 0.0315 & 0.0028 & 0.0077 & 0.0067 \\
    1.75 & 0.0396 & 0.0056 & 0.0128 & 0.0063 \\
    2.00 & 0.0489 & 0.0107 & 0.0226 & 0.0107 \\
    \bottomrule
    \end{tabular*}
\end{table}

\begin{table}[t]
    \centering
    \small
    \caption{\texttt{Bias in Bios} temperature sweep for RAPS at \(\alpha=0.10\).}
    \label{tab:temperature-sweep-raps}
    \begin{tabular*}{0.80\linewidth}{@{\extracolsep{\fill}}ccccc@{}}
    \toprule
    \(T\) & \(\sigma_\Delta\) & \shortstack[c]{RMS pooled\\coverage} & \shortstack[c]{RMS size\\from \(q_g\)} & \shortstack[c]{RMS coverage\\after equalized size} \\
    \midrule
    1.00 & 0.0276 & 0.0010 & 0.0083 & 0.0019 \\
    1.10 & 0.0121 & 0.0004 & 0.0034 & 0.0020 \\
    1.25 & 0.0080 & 0.0013 & 0.0024 & 0.0020 \\
    1.50 & 0.0355 & 0.0033 & 0.0098 & 0.0022 \\
    1.75 & 0.0584 & 0.0045 & 0.0150 & 0.0086 \\
    2.00 & 0.0683 & 0.0093 & 0.0248 & 0.0076 \\
    \bottomrule
    \end{tabular*}
\end{table}

\subsection{Robustness to Alternative Scores}
The same conversion mechanism is also visible under the SAPS and RAPS scores. Table~\ref{tab:alt-scores} shows that all three scores exhibit the same pattern: pooled threshold coverage discrepancy, nonzero set size distortion after applying the group-wise thresholds \(q_g\), and nonzero coverage distortion after imposing an equalized set size.
\begin{table}[t]
\centering
\small
\caption{\texttt{Bias in Bios} comparison across three scores at \(\alpha=0.10\).}
\label{tab:alt-scores}
\begin{tabular*}{0.80\linewidth}{@{\extracolsep{\fill}}lccccc@{}}
\toprule
Score & \(q\) & \(\sigma_\Delta\) & \shortstack[c]{RMS pooled\\coverage} & \shortstack[c]{RMS size\\from \(q_g\)} & \shortstack[c]{RMS coverage\\after equalized size} \\
\midrule
Simple & 0.6423 & 0.0116 & 0.0015 & 0.0051 & 0.0017 \\
SAPS   & 0.5777 & 0.0162 & 0.0011 & 0.0049 & 0.0018 \\
RAPS   & 0.6297 & 0.0276 & 0.0010 & 0.0083 & 0.0019 \\
\bottomrule
\end{tabular*}
\end{table}

\begin{figure}
    \centering
    \includegraphics[width=1.0\linewidth,trim={8 4 8 7},clip]{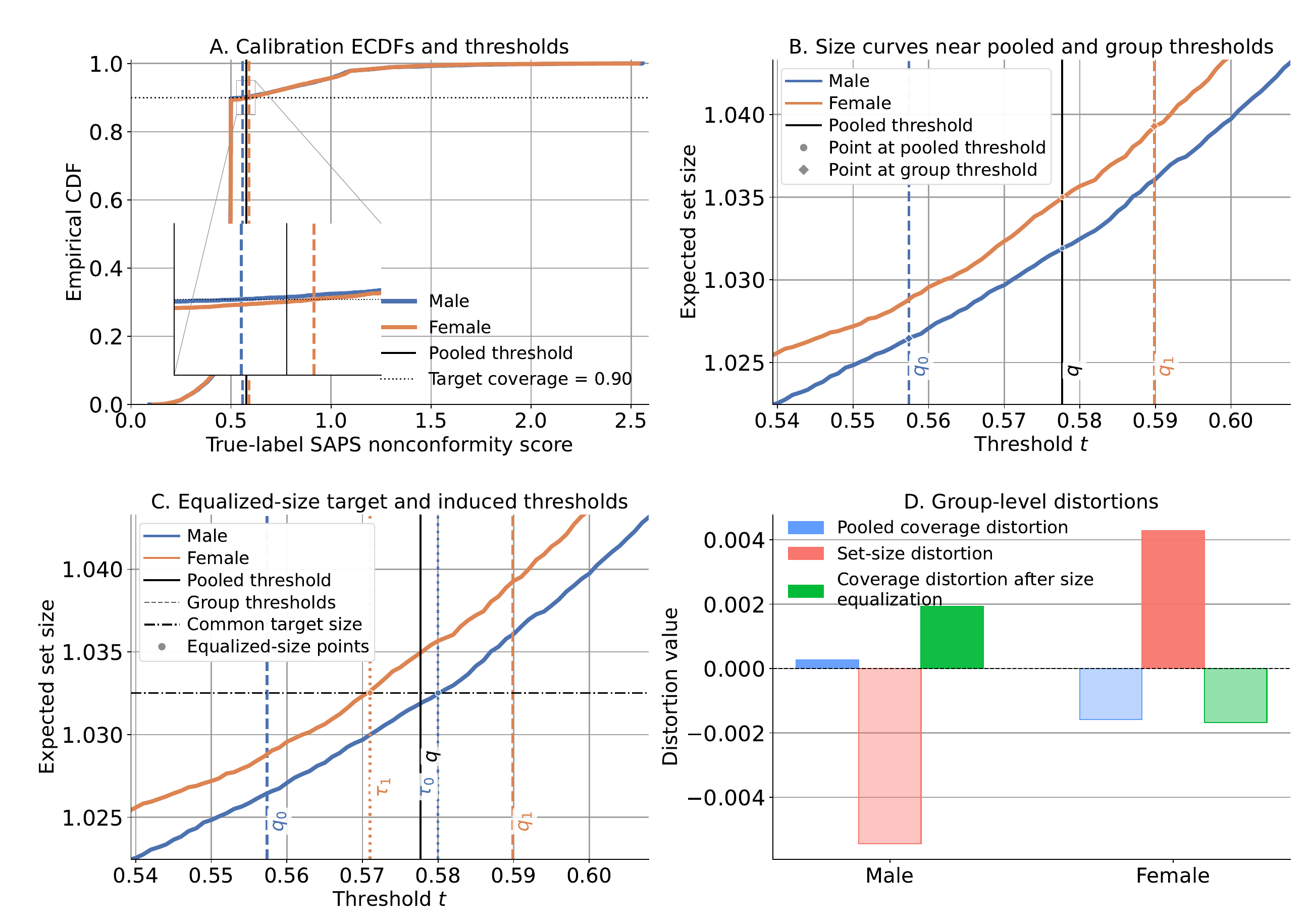}
    \caption{\texttt{Bias in Bios} mechanism view at \(\alpha = 0.10\) for SAPS score. Panel A illustrates the pooled-threshold mechanism in Theorem~\ref{thm:conservation}; Panels B--C illustrate Theorems~\ref{thm:3}--~\ref{thm:size_to_cov_disparity} and  Corollaries~\ref{cor:length_disparity}--\ref{cor:coverage_disparity}; Panel D summarizes the three distortions for male and female groups.}
    \label{fig:biobias-saps}
    \end{figure}
    
    \begin{figure}
    \centering
    \includegraphics[width=1.0\linewidth,trim={8 4 8 7},clip]{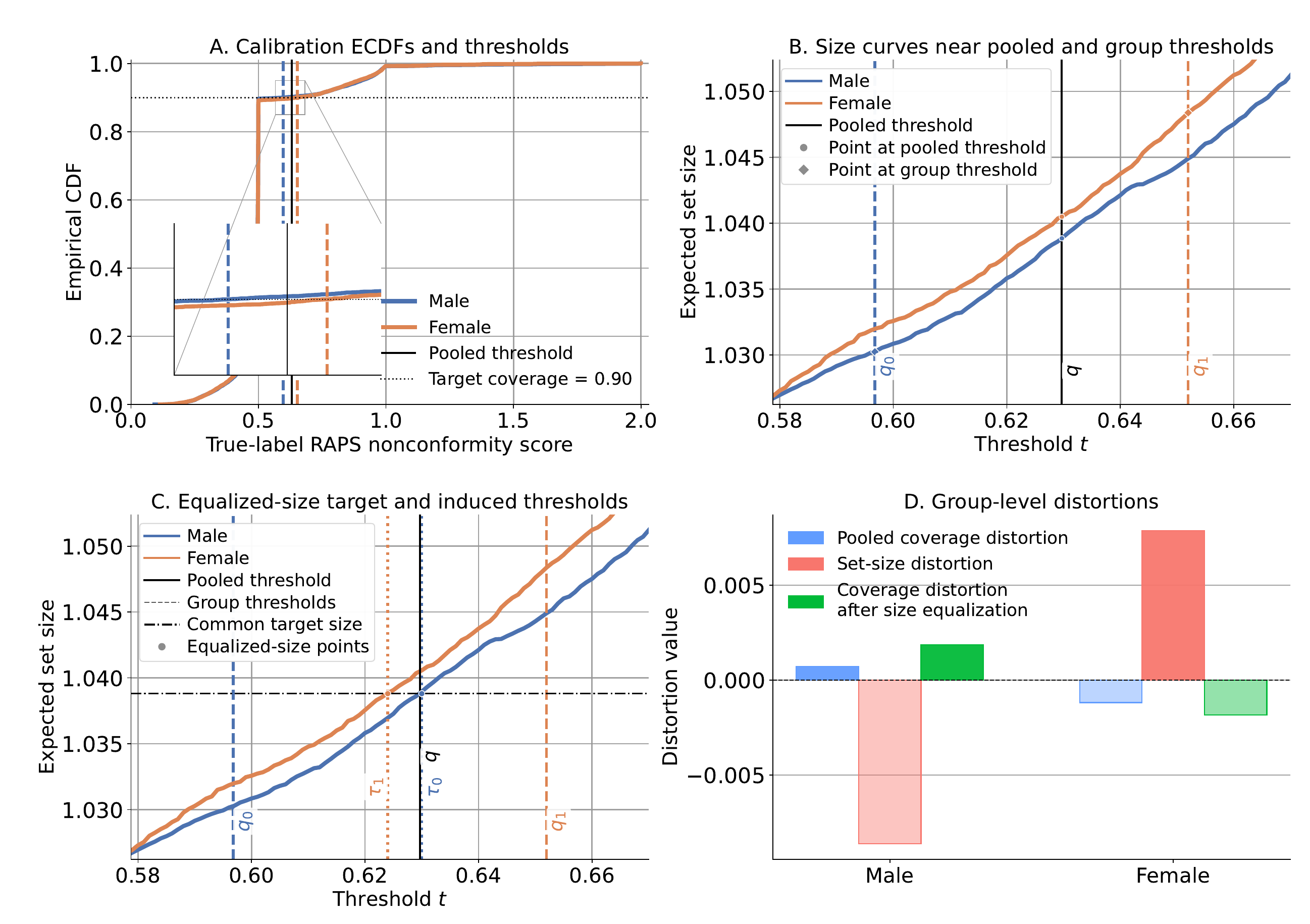}
    \caption{\texttt{Bias in Bios} mechanism view at \(\alpha = 0.10\) for RAPS score. Panel A illustrates the pooled-threshold mechanism in Theorem~\ref{thm:conservation}; Panels B--C illustrate Theorems~\ref{thm:3}--~\ref{thm:size_to_cov_disparity} and  Corollaries~\ref{cor:length_disparity}--\ref{cor:coverage_disparity}; Panel D summarizes the three distortions for male and female groups.}
    \label{fig:biobias-raps}
\end{figure}

\subsection{Finite Calibration of the Pooled-Threshold Floor}
\label{app:biobias-detectability}
This section complements the population uncertainty relations in Section~\ref{sec3:UR}, where the pooled-threshold floor obeys
\begin{equation}
\label{eq:app-pop-floor}
\|\varepsilon(q)\|_{L^2(p)} \ge m_{\mathrm{eff}}(q)\,\sigma_\Delta,
\end{equation}
where \(q\) is the pooled threshold, \(q_g\) are the group-specific \(1-\alpha\) quantiles, and \(\sigma_\Delta=\mathrm{sd}(q_G)\). The question here is finite-sample rather than population-level: how large must the calibration split be before the structural floor in \Cref{eq:app-pop-floor} becomes empirically resolvable under split conformal calibration? 

We fix the trained model outputs and vary only the calibration size. For each score choice (simple, SAPS, RAPS), we repeatedly form group-stratified subsamples of the original calibration split, recompute the pooled conformal threshold \(\hat q_n\), and evaluate the resulting RMS group miscoverage on a fixed test set. Throughout, we set the target level \(\alpha=0.1\) and the number of groups \(K=2\), and we use \(300\) subsamples for each
\[
n\in\{200,500,1000,2000,4000,7000,15000,20000, 25000, 28000\}.
\]
Let \(q^{\star}\) denote the pooled threshold computed from the full test set, used as a population proxy. We then define
\[
\mathrm{floor}^{\star}
:= \Bigl(\sum_{g=1}^K p_g^{\mathrm{test}}\,\varepsilon_g(q^{\star})^2\Bigr)^{1/2},
\qquad
\widehat{\mathrm{floor}}(n)
:= \Bigl(\sum_{g=1}^K p_g^{\mathrm{test}}\,\varepsilon_g(\hat q_n)^2\Bigr)^{1/2},
\]
and summarize the signal-to-noise ratio by
\(
\mathrm{SNR}(n)
:= \frac{\mathrm{floor}^{\star}}{\operatorname{sd}(\widehat{\mathrm{floor}}(n))}.
\)
We interpret \(\mathrm{SNR}(n) \ge 1\) as the point where the structural floor is at least as large as the calibration-induced standard deviation.

For the empirical estimation, the stiffness factor \(m_{\mathrm{eff}}(q)\) is replaced by a local density proxy at \(q^{\star}\), estimated from the pooled test scores by a window
\[
\widehat {m_{\mathrm{eff}}}(q^{\star})
:= \frac{1}{2hn_{\mathrm{test}}}\sum_{i=1}^{n_{\mathrm{test}}}\mathbf 1\{|S_i-q^{\star}|\le h\},
\]
with a Silverman-type~\citep{silverman2018density} bandwidth
\(
h = 0.9\,\hat\sigma_{\mathrm{rob}}\,n_{\mathrm{test}}^{-1/5},
\
\hat\sigma_{\mathrm{rob}}=\min\!\bigl(\hat\sigma,\,\mathrm{IQR}/1.34\bigr),
\)
where \(n_{\mathrm{test}}\) is the test set size and \(\hat \sigma\) is the standard deviation of test scores. Then, the estimated stiffness factor \(\widehat {m_{\mathrm{eff}}}(q^{\star})\) serves as an empirical detectability proxy for the local CDF slope entering \Cref{eq:app-pop-floor}, not as an exact plug-in estimate of the group-specific quantity in Definition~\ref{def:meff}.

To motivate the detectability scale, we use a DKW-style empirical-CDF heuristic~\citep{dvoretzky1956asymptotic}, treating the
calibration scores within each group as IID. Under an infinitely exchangeable model, this IID step can also be read conditionally on the de Finetti directing measure~\citep{barber2024finetti}. This is stronger than the exchangeability assumption
needed for CP and is used only for the sample-size diagnostic below.
If \(n\) denotes the total calibration size and the groups are roughly balanced, each group has
about \(n/K\) calibration scores. A union bound over the \(K\) group-wise empirical CDFs gives a uniform fluctuation scale of order
\begin{equation}
\sqrt{\frac{\log(K/\xi)}{n/K}}=
\sqrt{\frac{K\log(K/\xi)}{n}}.
\end{equation}
After quantile inversion, the corresponding threshold-noise scale is of order
\(
\sqrt{\frac{K\log(K/\xi)}{n\,m_{\mathrm{eff}}(q)^2}} .
\)
Requiring this noise scale to be no larger than the intrinsic quantile separation scale \(\sigma_\Delta\) motivates
\begin{equation}
\label{eq:app-scaling}
n \gtrsim \frac{K\log(K/\xi)}{m_{\mathrm{eff}}(q)^2\sigma_\Delta^2},
\end{equation}

where \(K\) is the number of groups and \(\xi\) is a fixed confidence parameter. We choose \(K=2\) and \(\xi=0.05\) in the plots across the score choices.

Panels~A--B of Figure~\ref{fig:app-detectability-grid} show a clear signal-resolution transition. The mean empirical floor decreases monotonically with \(n\), while the detectability ratio rises steadily and crosses the \(\mathrm{SNR}(n)=1\) threshold first for the simple score, and on the \(1.5\times10^4\) scale for both SAPS and RAPS. The order matches \Cref{eq:app-scaling}. In particular, the score with the steepest local CDF near \(q\) requires the fewest calibration points to resolve the floor.

Panels~C--D of Figure~\ref{fig:app-detectability-grid} show that small calibration sets do not merely add variance. They also inflate  both the observed lower bound and the empirical heterogeneity estimate \(\hat\sigma_\Delta\), so limited calibration size may exaggerate the lower bound.

Finally, let
\[
C_{\mathrm{det}}
:= \frac{n_{\mathrm{detect}}\,m_{\mathrm{eff}}(q)^2\sigma_\Delta^2}{{K\log(K/\xi)}},
\ \xi=0.05.
\]
Here, \(m_{\mathrm{eff}}(q)\) is the local density proxy at the pooled test-proxy threshold, \(\sigma_\Delta\) is the test-proxy heterogeneity of the group quantiles, \(\mathrm{floor}^{\star}\) is the true-floor proxy, and \(n_{\mathrm{detect}}\) is the smallest calibration size with \(\mathrm{SNR}(n)\ge 1\). Table~\ref{tab:app-detectability-summary} shows that \(C_{\mathrm{det}}\) 
is stable across all three scores up to a common multiplicative scale, which serves as a diagnostic of detectability.

In summary, Section~\ref{sec3:UR} identifies a population lower bound driven by intrinsic heterogeneity \(\sigma_\Delta\) and local stiffness \(m_{\mathrm{eff}}(q)\). The present experiments show that observing the lower bound in finite-sample split conformal calibration is itself a sample-size-related phenomenon. Empirically, the relevant resolution scale is well summarized by \(m_{\mathrm{eff}}(q)^2\sigma_\Delta^2\). That is, once \(n\) is large enough relative to the inverse of the resolution scale, the lower bound becomes detectable and the empirical curves stabilize near the population proxy.

\begin{table}[t]
    \centering
    \small
    \caption{Detectability summary at $\alpha=0.1$.}
    \label{tab:app-detectability-summary}
    \begin{tabular*}{0.80\linewidth}{@{\extracolsep{\fill}}lcccccc@{}}
    \toprule
    Score & $m_{\mathrm{eff}}(q)$ & $\sigma_\Delta$ & $\mathrm{floor}^{\star}$ & full-calibration floor & $n_{\mathrm{detect}}$ & $C_{\mathrm{det}}$ \\
    \midrule
    RAPS   & 0.0564 & 0.0139 & 0.0010 & 0.0010 & 15000 & 0.0013 \\
    SAPS   & 0.0672 & 0.0094 & 0.0009 & 0.0011 & 15000 & 0.0008 \\
    Simple & 0.1581 & 0.0094 & 0.0015 & 0.0016 & 7000  & 0.0021 \\
    \bottomrule
    \end{tabular*}
\end{table}

\begin{figure}[t]
    \centering
    \includegraphics[width=1.0\linewidth,trim={4 3 8 6},clip]{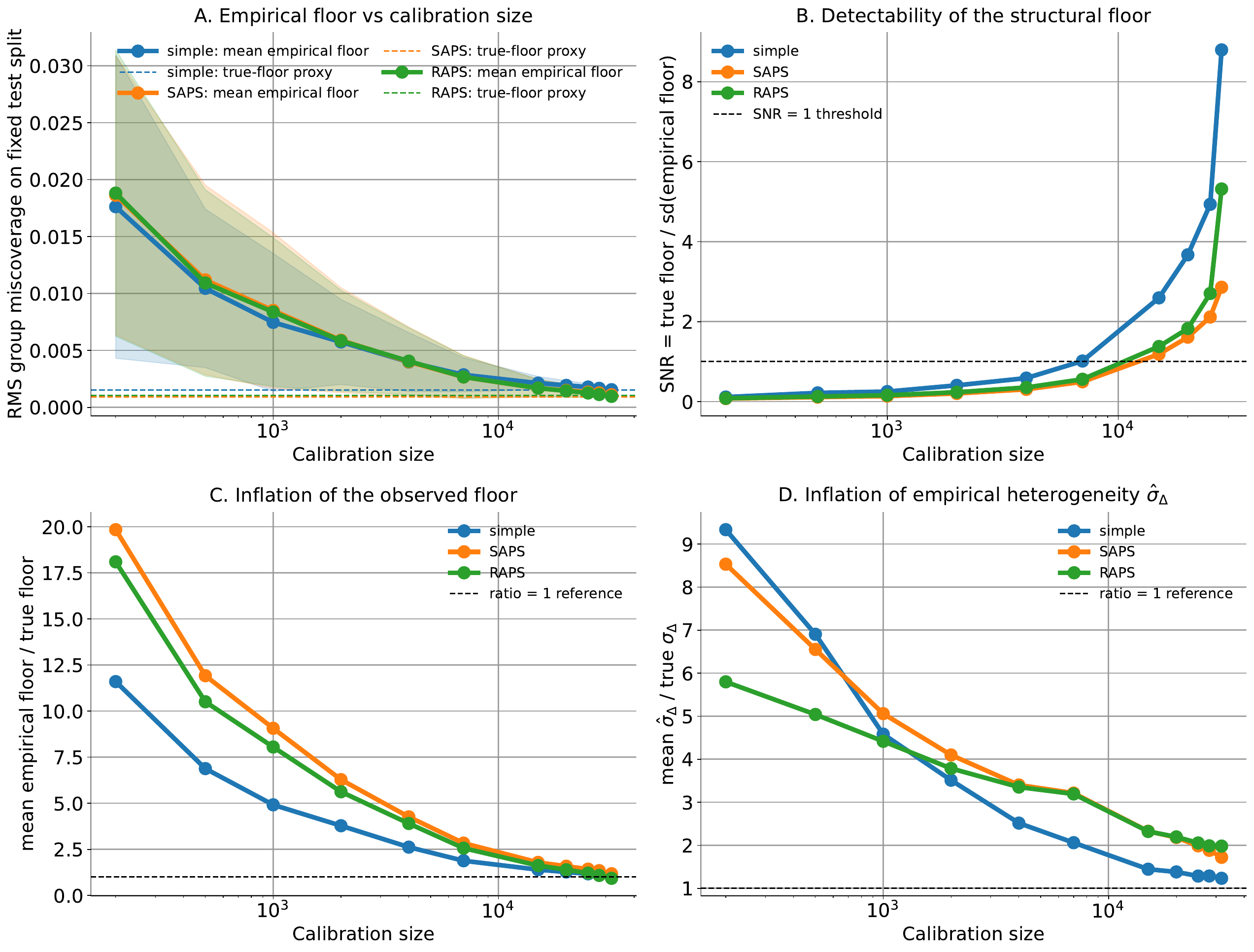}
    \caption{Finite-calibration detectability diagnostics of the pooled-threshold floor from Theorem~\ref{thm:heisenberg_L2} on \texttt{Bias in Bios}. 
    Panels A--B show how the empirical floor and its detectability improve with the calibration size. The dashed line at \(\mathrm{SNR}(n)=1\) marks the empirical detectability benchmark. Panels C--D show that small calibration splits can also inflate the observed floor and estimated heterogeneity.
    }
    \label{fig:app-detectability-grid}
\end{figure}
\FloatBarrier
\section{\texttt{MultiNLI} Experiments}
\label{app:multi}

\subsection{Experiment Settings}
All post-hoc quantities are defined exactly as in Section~\ref{sec:biobias-main} except for the dataset-specific details. We use the Hugging Face \texttt{MultiNLI} corpus\footnote{Hugging Face dataset: \texttt{nyu-mll/multi\_nli}; \url{https://huggingface.co/datasets/nyu-mll/multi_nli}. The dataset card states that most of the corpus is released under the OANC license, while the FICTION section includes Creative Commons Attribution 3.0, Creative Commons Share-Alike 3.0, and U.S. public-domain materials.}, treat the ten genres as groups, and train a \texttt{DistilBERT} classifier for three-class NLI. The run uses a maximum sequence length of \(256\), a learning rate of \(2 \times 10^{-5}\), a weight decay of \(0.01\), batch sizes of \(64/32\), a warm-up ratio of \(0.05\), and two epochs. The validation accuracy and macro-F1 are \(0.8104\) and \(0.8100\). The detailed split sizes and the summary of the key quantities are shown in Tables~\ref{tab:multi-splits} and \ref{tab:multi-primary-mechanism}. In addition, Table~\ref{tab:multi-groups} reports the full per-genre quantities at level \(\alpha=0.1\). The alternative score results for SAPS and RAPS are deferred to Section~\ref{app:multi-alt-scores}. 

\begin{table}[t]
    \centering
    \small
    \caption{Split sizes for \texttt{MultiNLI}.}
    \label{tab:multi-splits}
    \begin{tabular*}{0.37\linewidth}{@{\extracolsep{\fill}}lc}
\toprule
Split & $n$ \\
\midrule
Model train & 353,431 \\
Model validation & 39,271 \\
Calibration & 9,823 \\
Test & 9,824 \\
\bottomrule
\end{tabular*}
\end{table}

\begin{table}
    \centering
    \small
    \caption{Summary for \texttt{MultiNLI} at $\alpha=0.10$ with simple score.}
    \label{tab:multi-primary-mechanism}
    \begin{tabular*}{0.80\linewidth}{@{\extracolsep{\fill}}cccccc@{}}
\toprule
$\sigma_\Delta$ & \shortstack{RMS pooled\\coverage} & $\lambda$ & \shortstack{RMS size\\distortion} & $\sigma_\lambda$ & \shortstack{RMS coverage\\distortion} \\
\midrule
0.0354 & 0.0150 & 1.2717 & 0.0532 & 0.0779 & 0.0209 \\
\bottomrule
\end{tabular*}
\end{table}

\begin{table}
    \centering
    \small
    \setlength{\tabcolsep}{3.5pt}
    \caption{Per-genre summary for \texttt{MultiNLI} at \(\alpha=0.1\) with simple score.}
    \label{tab:multi-groups}
    \resizebox{\linewidth}{!}{%
    \begin{tabular*}{0.90\linewidth}{@{\extracolsep{\fill}}lccccccc@{}}
\toprule
Genre & $p_g$ & $q_g$ & $\varepsilon_g(q)$ & $\lambda_g$ & $\lambda_g-\ell_g(q)$ & $\tau_g$ & $\hat F_{S\mid g}(\tau_g)-\hat F_{S\mid g}(q_g)$ \\
\midrule
Government & 0.0989 & 0.7215 & 0.0198 & 1.1420 & -0.0905 & 0.8200 & 0.0463 \\
Letters & 0.1006 & 0.7561 & 0.0119 & 1.1741 & -0.0304 & 0.8360 & 0.0192 \\
Travel & 0.1007 & 0.7747 & 0.0090 & 1.2285 & -0.0303 & 0.8030 & 0.0081 \\
Oup & 0.0998 & 0.7802 & 0.0082 & 1.2439 & -0.0173 & 0.8050 & 0.0051 \\
Verbatim & 0.0990 & 0.8151 & 0.0044 & 1.3464 & 0.0380 & 0.7730 & -0.0195 \\
Fiction & 0.1004 & 0.8371 & 0.0037 & 1.3398 & 0.0659 & 0.7910 & -0.0132 \\
Facetoface & 0.1006 & 0.8223 & 0.0008 & 1.3168 & 0.0506 & 0.7960 & -0.0101 \\
Telephone & 0.1001 & 0.7951 & -0.0150 & 1.3001 & 0.0020 & 0.7810 & -0.0061 \\
Nineeleven & 0.1005 & 0.7883 & -0.0155 & 1.2249 & -0.0091 & 0.8220 & 0.0132 \\
Slate & 0.0996 & 0.8410 & -0.0329 & 1.4008 & 0.0982 & 0.7750 & -0.0307 \\
\bottomrule
\end{tabular*}
    }
\end{table}

\subsection{Robustness across Target Coverage}
Figure~\ref{fig:multi-alpha} shows that across \(\alpha\in\{0.05,0.07,0.085,0.10\}\), the empirical pooled-threshold distortion stays above the estimated lower-bound scale, while the induced size and coverage distortions remain nonzero throughout.
\begin{figure}
    \centering
    \includegraphics[width=\linewidth,trim={6 3 8 6},clip]{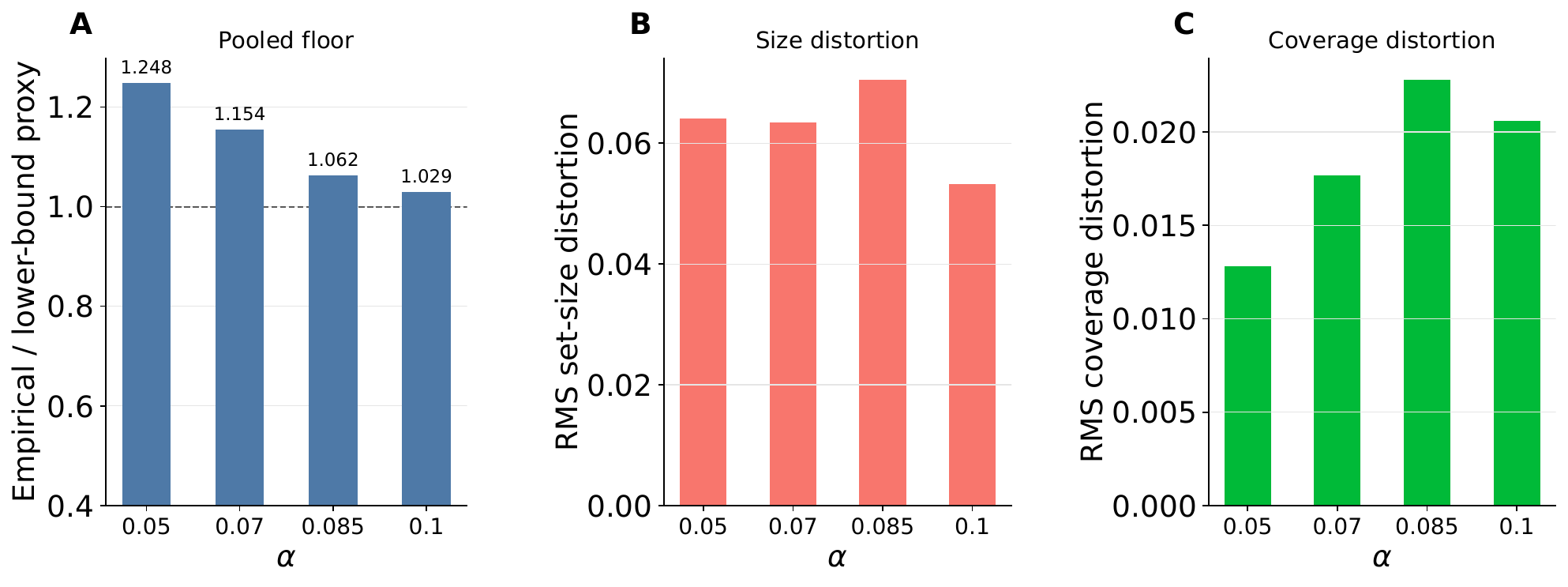}
    \caption{\texttt{MultiNLI} robustness across \(\alpha\) for the simple score. Panel A displays Theorem~\ref{thm:heisenberg_L2}: the empirical pooled-threshold distortion stays near or above the lower-bound scale across the tested \(\alpha\)-grid. Panels B--C illustrate Corollaries~\ref{cor:length_disparity}--\ref{cor:coverage_disparity}: the induced set-size and coverage distortions remain nonzero throughout.}
    \label{fig:multi-alpha}
\end{figure}

\subsection{Controlled Genre Temperature Sweep}
At \(\alpha=0.1\), we perturb only the \texttt{facetoface} genre via temperature scaling. Figure~\ref{fig:multi-temperature} shows that across the sweep, the empirical pooled-threshold distortion remains above the estimated lower bound, and the induced size and coverage distortions stay nonzero.
\begin{figure}
    \centering
    \includegraphics[width=\linewidth,trim={6 3 8 6},clip]{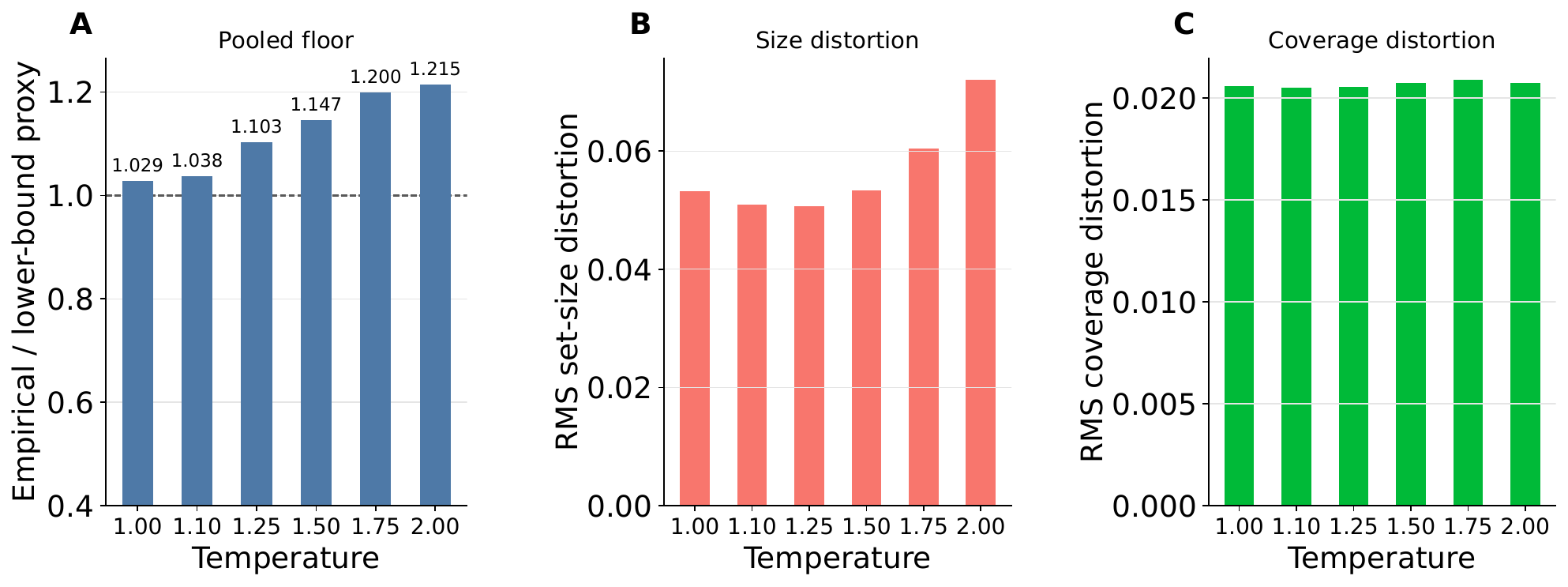}
    \caption{Controlled \texttt{MultiNLI} temperature sweep at $\alpha=0.10$ for the simple score, perturbing the \textit{facetoface} genre only. Panel~A corresponds to Theorem~\ref{thm:heisenberg_L2}. Panels B--C show Corollaries~\ref{cor:length_disparity}--\ref{cor:coverage_disparity}. The same trade-off mechanism remains visible across the temperature sweep.
    }
    \label{fig:multi-temperature}
\end{figure}

\subsection{Alternative Scores}
\label{app:multi-alt-scores}
The same mechanism persists under SAPS and RAPS with the model, groups, and split protocol fixed. SAPS uses the default temperature \(T=1.0\) and \(\lambda_{\mathrm{SAPS}}=0.3\); RAPS uses the default temperature \(T=0.6\), \(k_{\mathrm{reg}}=1\), and \(\lambda_{\mathrm{RAPS}}=0.02\). Both scores utilize the randomization level \(u=0.5\). Figure~\ref{fig:multi-alt-primary} demonstrates that at \(\alpha=0.1\), the SAPS score shows a pooled floor, a nonzero set size distortion after switching from pooled \(q\) to group-wise \(q_g\), and a nonzero coverage distortion after imposing the equalized expected set size. Robustness checks are provided in Figures~\ref{fig:multi-alt-alpha} and \ref{fig:multi-alt-temperature}. We do not view these results as requiring pointwise lower-bound dominance at every \(\alpha\) or temperature value. For SAPS and RAPS, Panel~A is best interpreted as a finite-sample diagnostic based on an estimated lower-bound proxy, while the main evidence remains visible.

\begin{figure}
    \centering
    \includegraphics[width=0.5\linewidth,trim={4 3 8 6},clip]{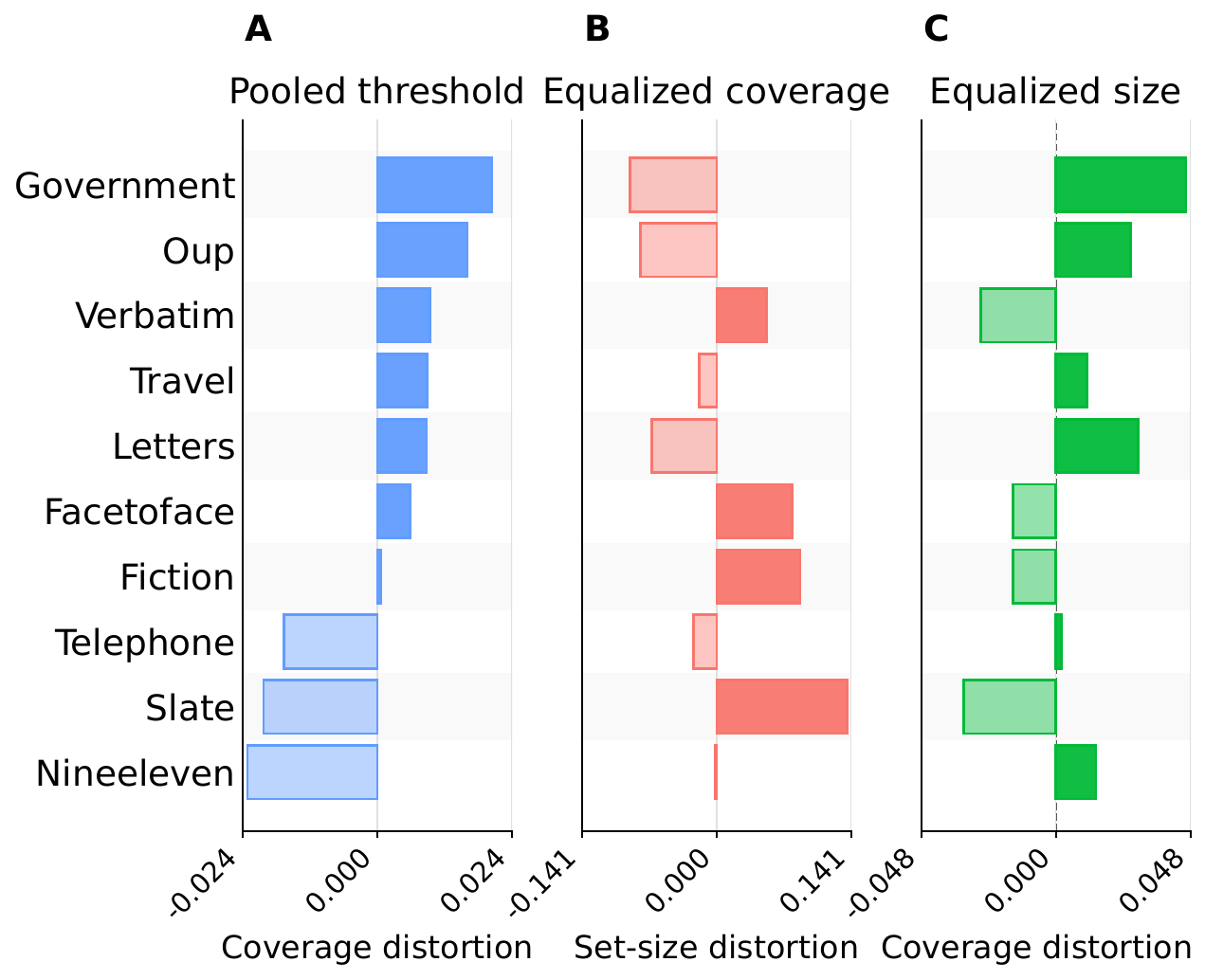}
    \caption{\texttt{MultiNLI} at \(\alpha = 0.10\) using SAPS score. For this score, Panel A shows pooled quantile consequence of Theorem~\ref{thm:conservation}; Panels B and C illustrate the set size distortion in Corollary~\ref{cor:length_disparity}, and the coverage distortion in Corollary~\ref{cor:coverage_disparity}, respectively.}
    \label{fig:multi-alt-primary}
\end{figure}

\begin{figure}
    \centering
    \includegraphics[width=\linewidth,trim={6 3 8 6},clip]{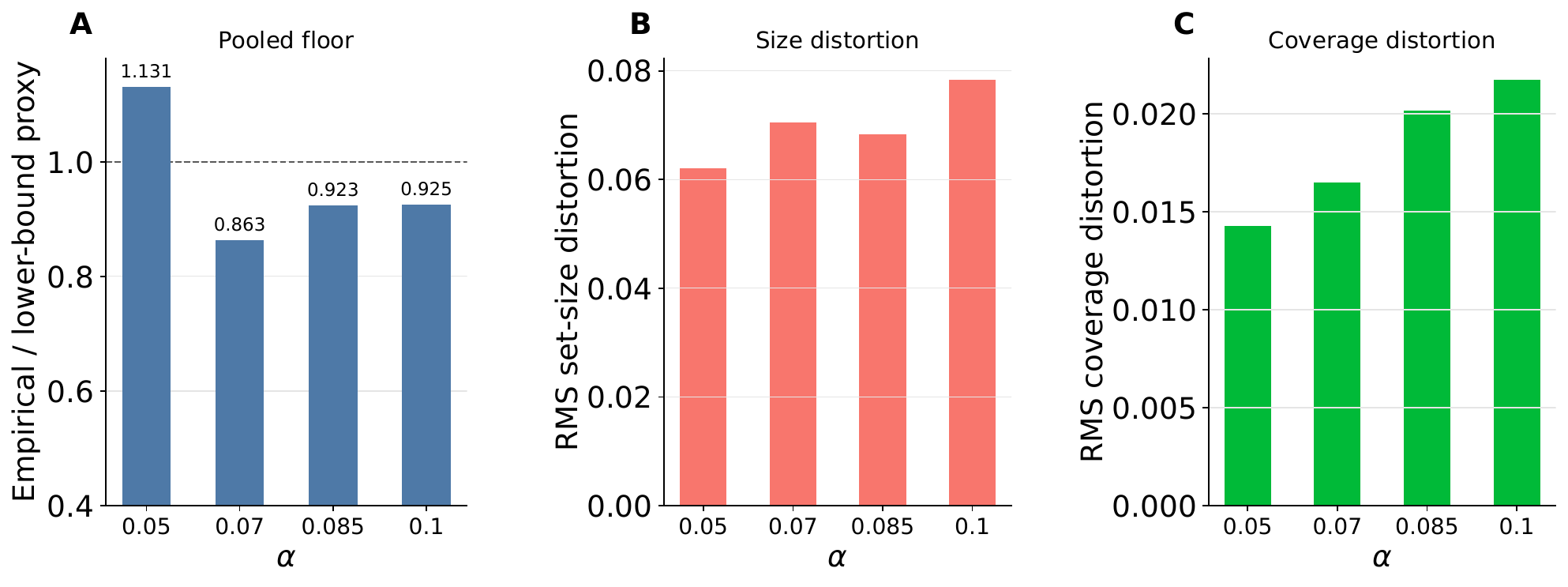}\\[-0.4em]
    \includegraphics[width=\linewidth,trim={6 3 8 6},clip]{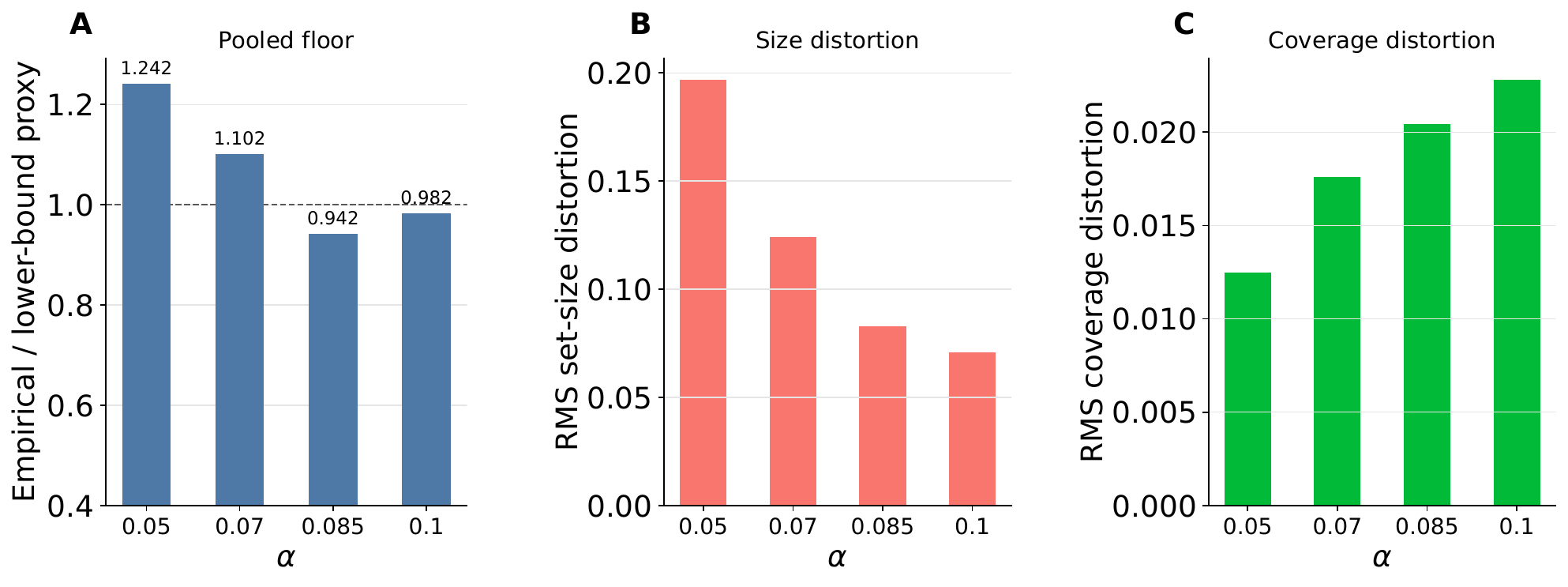}
    \caption{\texttt{MultiNLI} robustness across \(\alpha\) for SAPS (top) and RAPS (bottom) scores. In each row, Panel A is best read as a finite-sample diagnostic for Theorem~\ref{thm:heisenberg_L2} based on an estimated lower-bound proxy, rather than a pointwise lower-bound verification. Panels B--C show Corollaries~\ref{cor:length_disparity}--\ref{cor:coverage_disparity}. The induced set-size and coverage distortions remain visible across the tested \(\alpha\)-grid.}
    \label{fig:multi-alt-alpha}
\end{figure}

\begin{figure}
    \centering
    \includegraphics[width=\linewidth,trim={6 3 8 6},clip]{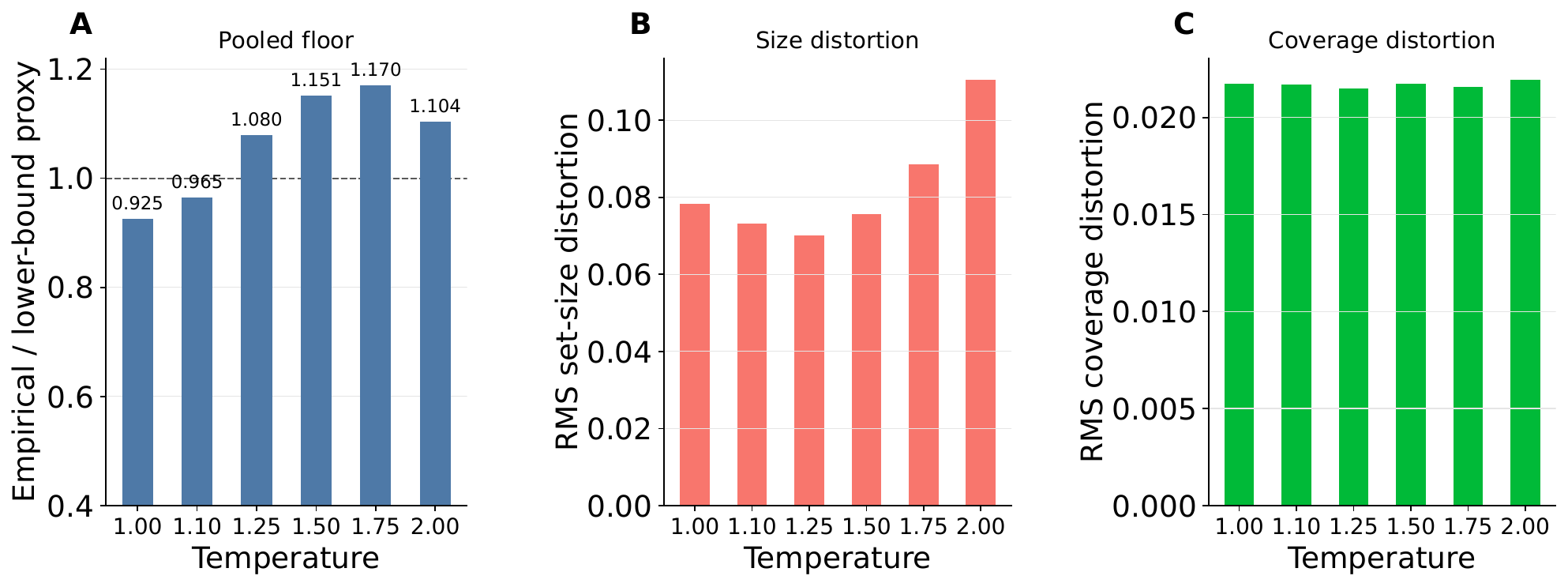}\\[-0.4em]
    \includegraphics[width=\linewidth,trim={6 3 8 6},clip]{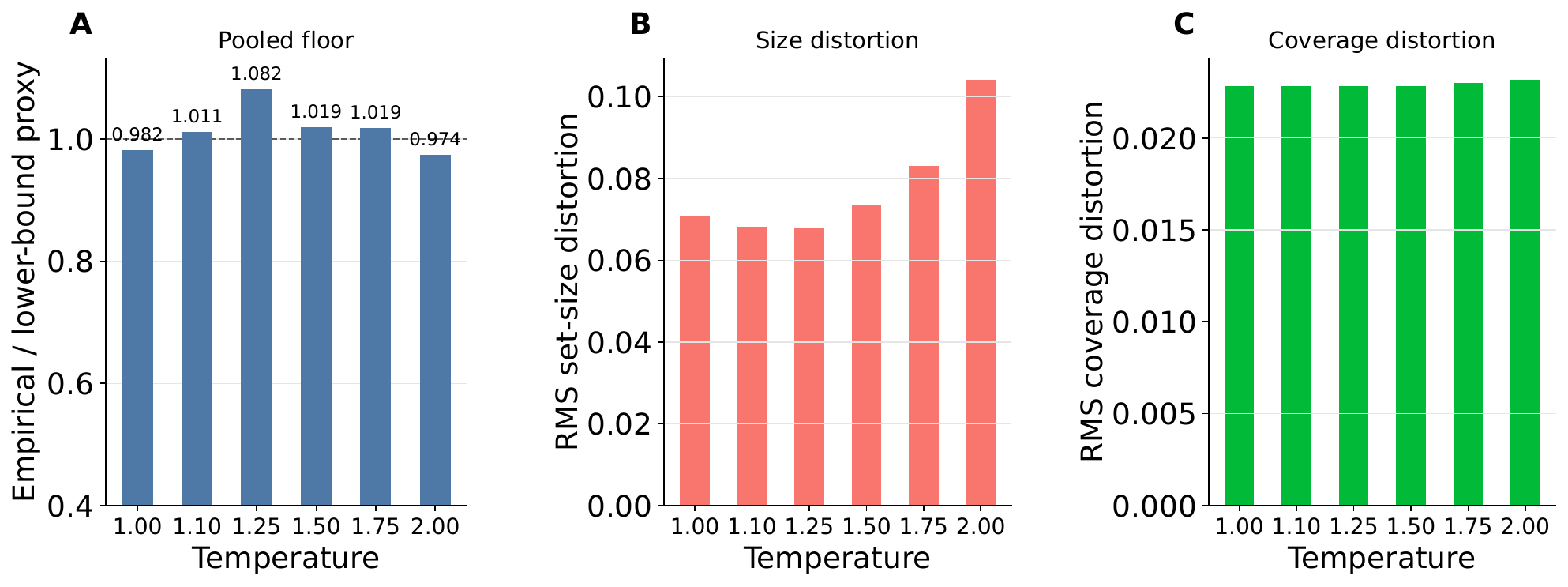}
    \caption{Controlled \texttt{MultiNLI} temperature sweep at $\alpha=0.10$ for SAPS (top) and RAPS (bottom) scores, perturbing the \textit{facetoface} genre only. Panel~A is best read as a finite-sample diagnostic for Theorem~\ref{thm:heisenberg_L2}. Panels B--C illustrate Corollaries~\ref{cor:length_disparity}--\ref{cor:coverage_disparity}. The same trade-off mechanism remains visible across the temperature sweep.}
    \label{fig:multi-alt-temperature}
\end{figure}

\subsection{Detectability under Finite Calibration}
\label{app:multi-detectability}
This subsection follows Appendix~\ref{app:biobias-detectability} with the same finite calibration protocol. Because \texttt{MultiNLI} has fewer genre-stratified calibration examples than Bias in Bios, we use the grid
\[
n\in\{50,100,150,200,250,300,350,400,450,500\}.
\] We fix the trained \texttt{MultiNLI} outputs and test split, vary only the calibration size through genre-stratified subsampling, recompute the pooled threshold on each subsample, and evaluate the resulting RMS genre-wise miscoverage on the fixed test set. Figure~\ref{fig:multi-detectability} shows that the floor becomes visible around \(n=100\), and is clearly resolved for all three scores thereafter. Hence, on the \texttt{MultiNLI} dataset, the structural signal is empirically visible once the calibration split is moderately large.

\begin{figure}[t]
    \centering
    \begin{subfigure}[t]{0.49\linewidth}
        \centering
        \includegraphics[width=\linewidth,trim={8 3 8 6},clip]{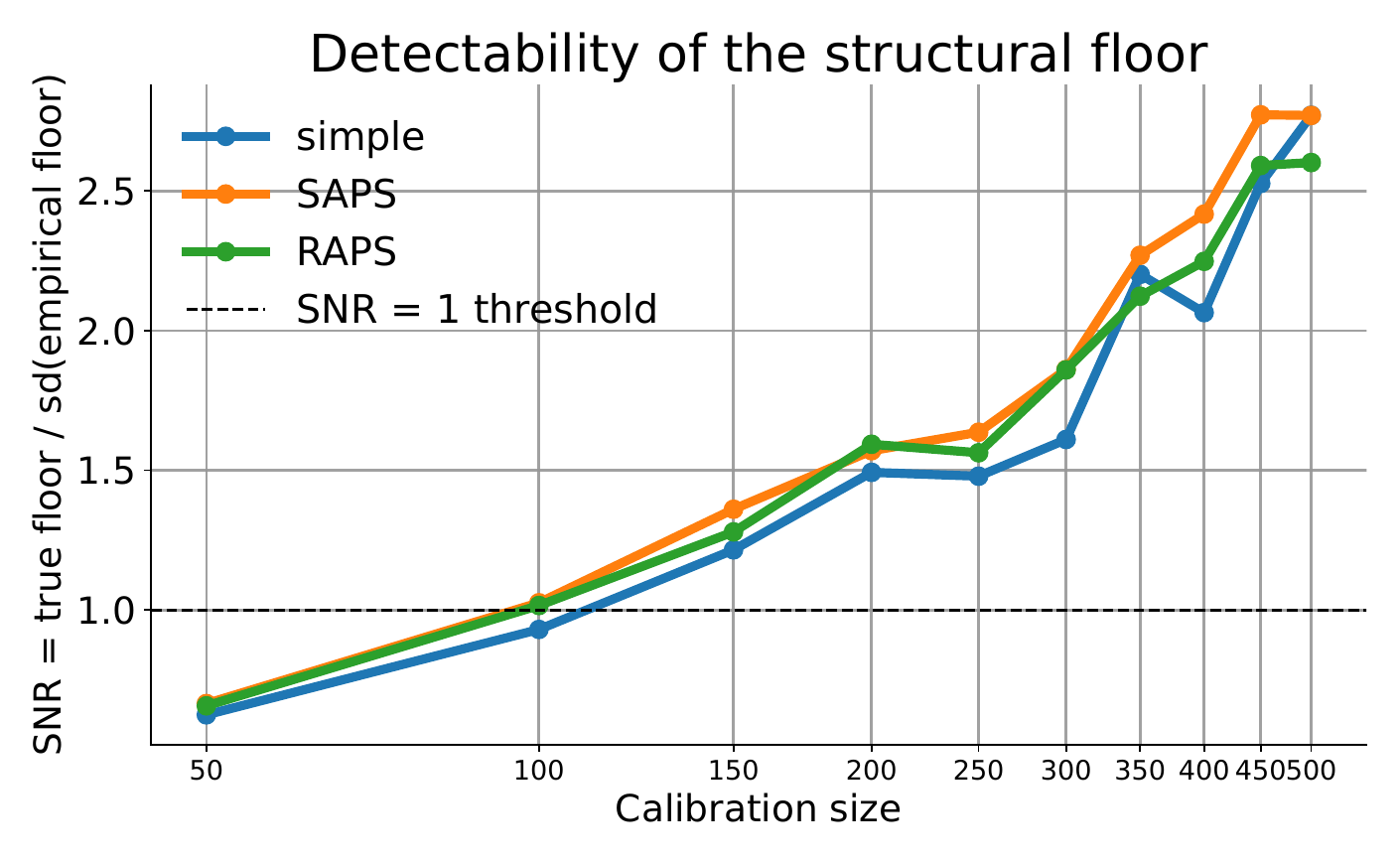}
        \caption{}
        \label{fig:multi-detectability}
    \end{subfigure}
    \hfill
    \begin{subfigure}[t]{0.49\linewidth}
        \centering
        \includegraphics[width=\linewidth,trim={8 3 8 6},clip]{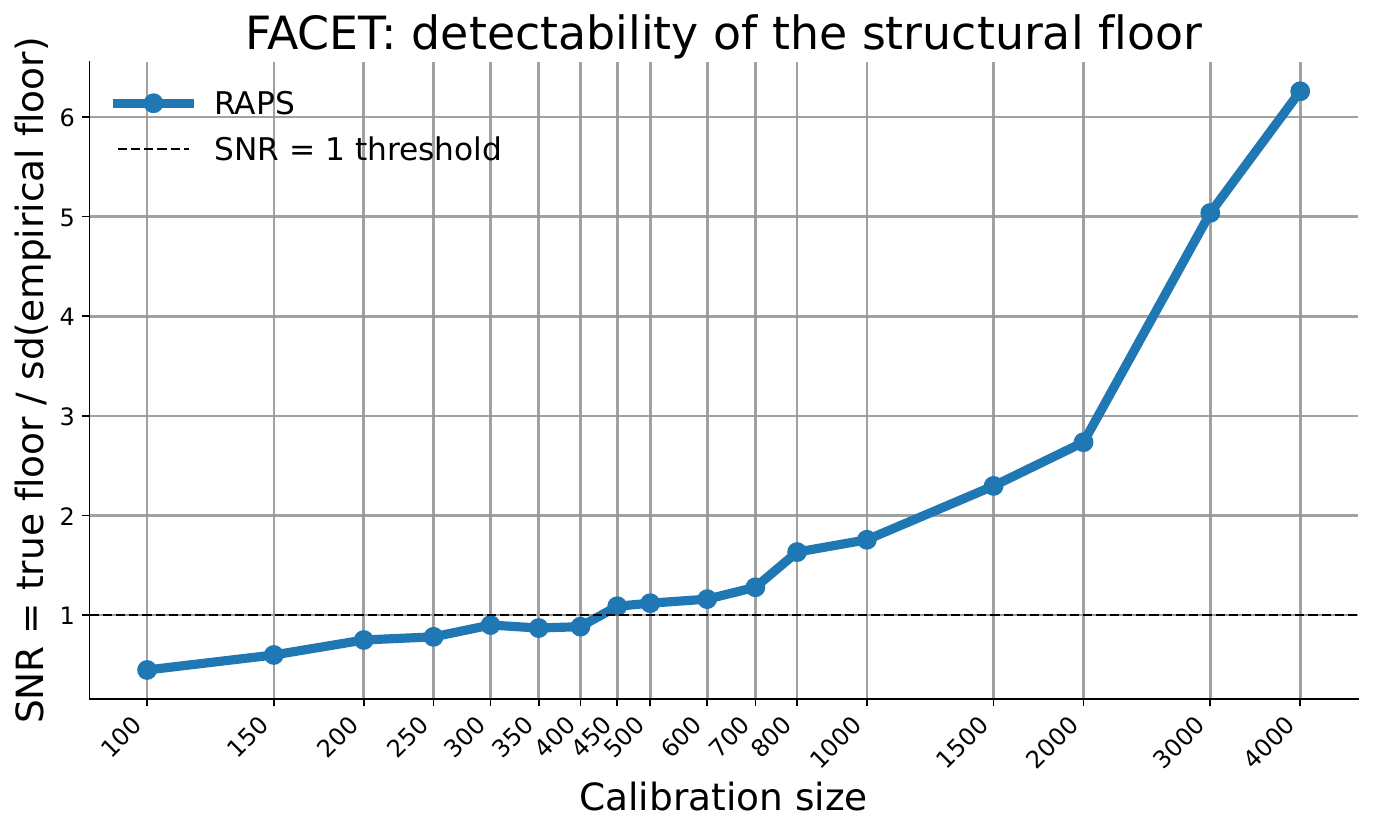}
        \caption{}
        \label{fig:facet-detectability}
    \end{subfigure}
    \caption{
    Finite-calibration detectability of the pooled-threshold floor from Theorem~\ref{thm:heisenberg_L2}. Left: \texttt{MultiNLI}, for simple, SAPS, and RAPS scores. Right: \texttt{FACET}, for RAPS score. The dashed line at \(\mathrm{SNR}(n)=1\) marks the empirical detectability benchmark, so crossing it indicates that the structural floor is becoming distinguishable from finite-sample fluctuation.
    }
\end{figure}
\FloatBarrier
\section{\texttt{FACET} Experiments}
\label{app:facet}

\subsection{Experiment Settings and Summary}
For \texttt{FACET},\footnote{We use the \texttt{FACET} benchmark dataset released by Meta AI; official dataset page: \url{https://ai.meta.com/datasets/facet/}. The \texttt{FACET} data card in the ICCV 2023 supplemental material states that the license is ``Custom license, see dataset download agreement,'' and that \texttt{FACET} is intended for evaluation only. For a description of the benchmark, see ~\citep{gustafson2023facet}.} we use \texttt{CLIP ViT-L/14}.\footnote{\texttt{CLIP ViT-L/14} from OpenAI repository: \url{https://github.com/openai/CLIP}. The repository LICENSE file states that the code is released under the MIT License.} All post-hoc quantities are defined exactly as in Sections~\ref{sec:biobias-main} and ~\ref{sec:multinli-main}. Tables~\ref{tab:facet-primary-summary}, \ref{tab:facet-primary-groups} and~\ref{tab:facet-run-config} list the key quantities, the full per-group quantities, and the experiment configuration, respectively.
\begin{table}[t]
\centering
\small
\caption{\texttt{FACET} RAPS summary at \(\alpha=0.10\).}
\label{tab:facet-primary-summary}
\begin{tabular*}{0.30\linewidth}{@{\extracolsep{\fill}}lc@{}}
\toprule
Metric & Value \\
\midrule
\(q\) & 1.0194 \\
\(\sigma_\Delta\) & 0.0040 \\
Pooled RMS & 0.0083 \\
Lower bound & 0.0080 \\
Ratio & 1.034 \\
\(\lambda\) & 2.1735 \\
RMS set size & 0.1717 \\
\(\sigma_\lambda\) & 0.2817 \\
RMS coverage & 0.0199 \\
\bottomrule
\end{tabular*}
\end{table}

\begin{table}[t]
\centering
\small
\caption{Per-group quantities for \texttt{FACET} at \(\alpha=0.1\).}
\label{tab:facet-primary-groups}
\begin{tabular*}{0.9\linewidth}{@{\extracolsep{\fill}}lccccccc@{}}
    \toprule
    Age group & \(p_g\) & \(q_g\) & \(\varepsilon_g(q)\) & \(\lambda_g\) & \(\lambda_g-\ell_g(q)\) & \(\tau_g\) & \(\hat F_{S\mid g}(\tau_g)-\hat F_{S\mid g}(q_g)\) \\
    \midrule
    Younger & 0.1795 & 1.0177 & 0.0176 & 1.8757 & -0.0892 & 1.0200 & 0.0054 \\
    Middle  & 0.5718 & 1.0196 & -0.0035 & 2.0921 & 0.0182 & 1.0200 & 0.0025 \\
    Older   & 0.0461 & 1.0111 & -0.0105 & 1.9895 & -0.3421 & 1.0171 & 0.0263 \\
    Unknown & 0.2026 & 1.0276 & -0.0018 & 2.7090 & 0.3329 & 1.0126 & -0.0419 \\
    \bottomrule
\end{tabular*}
\end{table}

\begin{table}[t]
\centering
\small
\caption{Configuration for \texttt{FACET} with RAPS.}
\label{tab:facet-run-config}
\begin{tabular*}{0.80\linewidth}{@{\extracolsep{\fill}}p{0.33\linewidth}p{0.39\linewidth}@{}}
\toprule
Setting & Value \\
\midrule
Dataset & \texttt{FACET} \\
Groups &  Younger, Middle, Older, Unknown \\
Labels & 20 occupations \\
Vision-language model & \texttt{CLIP ViT-L/14} \\
Image preprocessing & bbox crop with 0.08 expansion \\
Validation samples  & 4,122 \\
Calibration samples & 11,776\\
Test samples & 4,122 \\
Primary target \(\alpha\) & 0.10 \\
RAPS temperature & 0.60 \\
RAPS \(\lambda\) & 0.02 \\
RAPS \(k_{\mathrm{reg}}\) & 1 \\
Temperature-perturbed group & Younger \\
Temperature values & 1.0, 1.1, 1.25, 1.5, 1.75, 2.0 \\
Test accuracy / macro-F1 & 0.6994 / 0.7369 \\
\bottomrule
\end{tabular*}
\end{table}

\subsection{Robustness across Target Coverage and Temperature Perturbation}
\label{app:facet-robustness}
\looseness-1 Figure~\ref{fig:facet-alpha-robustness} shows that for various target levels \(\alpha\in\{0.05, 0.07, 0.085, 0.10\}\), the empirical pooled-threshold distortion stays above the empirical lower bound. The induced set size distortion remains nonzero throughout and the equalized expected set size policy continues to produce a nonzero RMS coverage distortion.

\begin{figure}
    \centering
    \includegraphics[width=\linewidth,trim={6 3 8 6},clip]{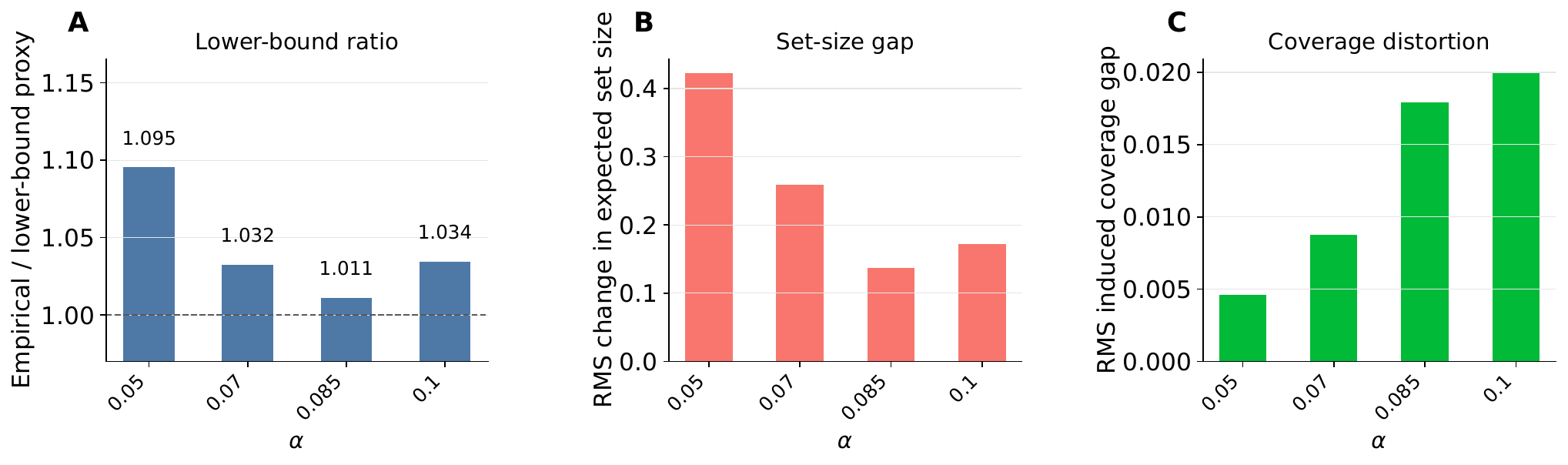}
    \caption{
    \texttt{FACET} robustness across \(\alpha\) for the RAPS score. Panel A illustrates Theorem~\ref{thm:heisenberg_L2}: the empirical pooled-threshold distortion stays near or above the lower-bound scale across the tested \(\alpha\)-grid. Panels B--C show Corollaries~\ref{cor:length_disparity}--\ref{cor:coverage_disparity}: the induced set-size and coverage distortions remain nonzero throughout.
    }
    \label{fig:facet-alpha-robustness}
\end{figure}

Next, we perturb only the Younger group through temperature scaling while keeping the rest of the groups fixed. Figure~\ref{fig:facet-temperature-robustness} shows that across perturbations, the pooled RMS gap remains at or above the lower bound proxy. The equalized expected set size policy continues to introduce a visible coverage distortion across the sweep. 

\begin{figure}[!htbp]
    \centering
    \includegraphics[width=\linewidth,trim={6 3 8 6},clip]{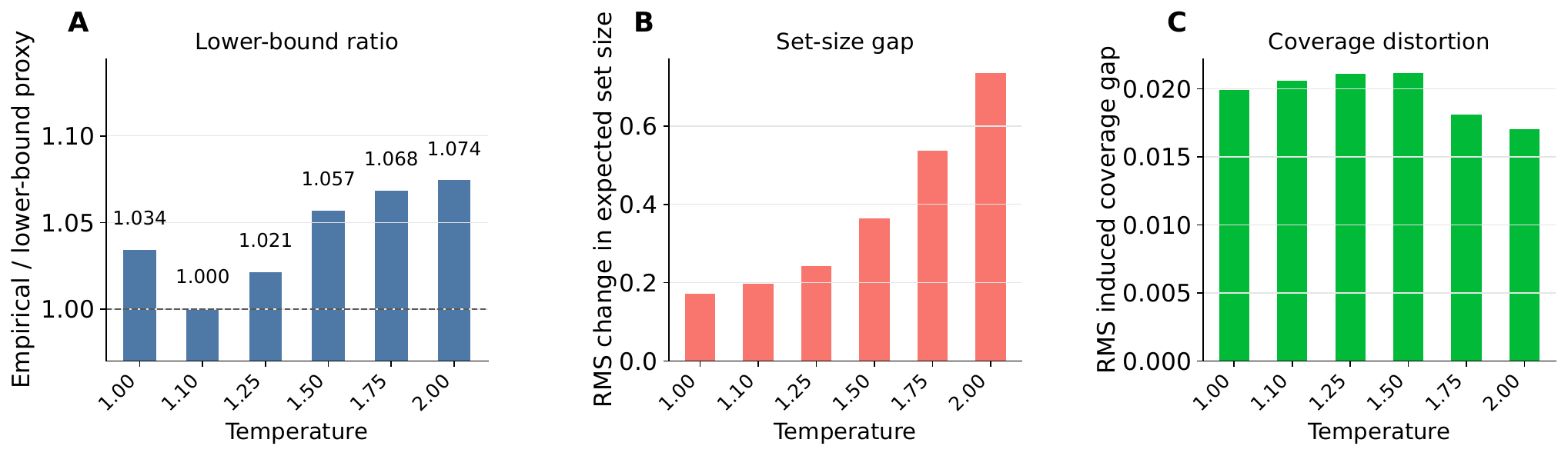}
    \caption{
    Controlled \texttt{FACET} temperature sweep at \(\alpha=0.10\) for the RAPS score, perturbing the Younger group only. Panel A evaluates the lower-bound behavior in Theorem~\ref{thm:heisenberg_L2}. Panels B--C show Corollaries~\ref{cor:length_disparity}--\ref{cor:coverage_disparity}. The same trade-off mechanism remains visible across the temperature sweep.
    }
    \label{fig:facet-temperature-robustness}
\end{figure}

\subsection{Detectability under Finite Calibration}
\label{app:facet-detectability}
Following the detectability analysis in Appendices~\ref{app:biobias-detectability} and \ref{app:multi-detectability}, we fix the model outputs and test split, vary only the calibration sample size through age-stratified subsampling, recompute the pooled threshold on each subsample, and evaluate the resulting RMS group-wise coverage on the fixed test split. For \texttt{FACET}, we use the grid 
\[n\in\{100,150,200,250,300,350,400,450,500,600,700,800,1000,1500,2000,3000,4000\}.
\]
Figure~\ref{fig:facet-detectability} illustrates the signal-to-noise ratio \(\mathrm{SNR}(n) = \mathrm{floor}^{\star}/\mathrm{sd}(\widehat{\mathrm{floor}}_n)\). On \texttt{FACET}, the ratio crosses \(1\) at about \(n=450\) and then increases. Although \texttt{FACET} exhibits substantial group-size imbalance, the same detectability pattern is still observed. In particular, group imbalance affects finite-sample variability but does not eliminate the signal itself, which highlights that the trade-off mechanism is intrinsic.

\subsection{Stability under Calibration Resampling}
\label{app:facet-resampling}

We verify that the trade-off mechanism observed on the \texttt{FACET} dataset is not an artifact of a single calibration split. We keep the model outputs and the test set fixed, and repeat age-stratified calibration subsampling at \(\alpha=0.1\). Concretely, we perform \(20\) repetitions using \(70\%\) of the full calibration size. For each repetition, we calculate the key quantities as in Section~\ref{sec:facet-main}.

Table~\ref{tab:facet-primary-resampling-stability} demonstrates that the resampling means remain close to the values with full calibration size. The pooled RMS coverage gap stays near \(0.008\), the induced RMS set-size gap remains around \(0.18\), and the equalized set size policy continues to produce a nonzero RMS coverage gap of about \(0.019\). Therefore, the empirical trade-off pattern in Section~\ref{sec:facet-main} is stable under repeated calibration perturbations.

\begin{table}[!htbp]
\centering
\small
\caption{\texttt{FACET} stability under calibration subsampling at \(\alpha=0.10\).}
\label{tab:facet-primary-resampling-stability}
\begin{tabular*}{0.80\linewidth}{@{\extracolsep{\fill}}lcc@{}}
\toprule
Metric & Original & Resampling mean \(\pm\) sd \\
\midrule
Pooled RMS coverage gap & 0.0083 & 0.0085 \(\pm\) 0.0005 \\
RMS set-size gap after \(q \to q_g\) & 0.1717 & 0.1828 \(\pm\) 0.0311 \\
RMS coverage gap after equalized size & 0.0199 & 0.0186 \(\pm\) 0.0044 \\
\bottomrule
\end{tabular*}
\end{table}
\section{Additional Robustness Analysis}
\label{app:additional-robustness}
\subsection{Repeated-split and cross-fitted robustness}
We evaluate the primary score settings---SIMPLE for \texttt{Bias in Bios} and \texttt{MultiNLI}, and RAPS for \texttt{FACET}---over 500 group-stratified calibration/test splits. Table~\ref{tab:repeated-split} reports the split mean and the 2.5th--97.5th split percentiles.

\begin{table}[t]
\centering
\small
\caption{Repeated-split robustness over 500 group-stratified calibration/test
splits. Entries are mean [2.5th, 97.5th split percentiles]. Coverage quantities
are reported in percentage points (pp). The intervals are descriptive split
percentiles conditional on the fixed predictor and available held-out pool;
they are not population confidence intervals.}
\label{tab:repeated-split}
\begin{tabular}{lccc}
\toprule
Dataset &
Pooled CP:  &
Equalized Coverage CP: &
Equalized Expected \\

& groupwise coverage & groupwise expected & Set Size CP: groupwise \\

& distortion (pp) & set size disparity (labels) & coverage distortion (pp)
\\
\midrule
Bias in Bios & 0.245 [0.106, 0.462] & 0.0049 [0.0002, 0.0118] & 0.166 [0.009, 0.428] \\
MultiNLI     & 1.521 [1.089, 2.011] & 0.0591 [0.0400, 0.0782] & 2.027 [1.481, 2.581] \\
FACET        & 1.000 [0.379, 1.680] & 0.1490 [0.0675, 0.2403] & 0.944 [0.483, 1.503] \\
\bottomrule
\end{tabular}
\end{table}

The same bidirectional transfer persists across splits. For \texttt{FACET}, the repeated-split size-to-coverage mean is \(0.944\) pp, compared with \(1.992\) pp on the original fixed split; we therefore base the robustness interpretation on the more conservative repeated-split distribution. Two-fold cross-fitting changes the size-to-coverage means by only \(0.1\%\), \(0.7\%\), and \(2.5\%\) on \texttt{Bias in Bios}, \texttt{MultiNLI}, and \texttt{FACET}, respectively. The repeated-split and cross-fitted analyses consistently reproduce the bidirectional coverage--size transfer, showing that it is robust to calibration/test resplits.

\subsection{Practical scale of the observed disparities}
Under the same score settings and 500 splits, the mean max-minus-min group ranges on \texttt{Bias in Bios}, \texttt{MultiNLI}, and \texttt{FACET} are, respectively: pooled-coverage gaps of \(0.33\), \(5.13\), and \(3.30\) pp; Equalized Coverage expected-set-size gaps of \(0.009\), \(0.281\), and \(0.714\) labels; and coverage-change gaps from moving to Equalized Expected Set Size of \(0.33\), \(7.21\), and \(3.01\) pp.

For context, \citet{cresswell2024cphuman} associate a \(0.28\)-label increase in average set size with a \(2.2\) pp decrease in human accuracy. More directly, for \texttt{FACET} under marginal CP, \citet{cresswell2025conformal} report a maximum cross-group set-size gap of \(0.75\) labels and a \(3.23\) pp maximum cross-group gap in the human accuracy improvement provided by prediction sets. Our \texttt{FACET} expected-set-size gap of \(0.714\) labels is therefore nearly identical in scale to the set-size disparity observed in that human study. Thus, the \texttt{MultiNLI} and \texttt{FACET} gaps occur at scales that prior human experiments have associated with measurable differences in downstream accuracy or in the accuracy benefit provided by prediction sets.

\subsection{Independent population validation of the lower bounds}
For the synthetic results, nonconformity scores are generated directly from the analytic data-generating models. We independently compute each population distortion and its theoretical lower bound: the bound uses analytic CDFs and densities, population quantiles, and expected-set-size slopes, whereas the distortion is evaluated separately. Any population left-hand-side-to-bound ratio below one would falsify the corresponding inequality.

\begin{table}[t]
\centering
\small
\caption{Independent population validation of the theoretical lower bounds.}
\label{tab:population-bound-validation}
\begin{tabular}{lccc}
\toprule
Policy & Configurations & Minimum ratio & Median ratio \\
\midrule
Pooled CP: coverage RMS & 100 & 1.000 & 1.212 \\
Equalized Coverage CP: size RMS & 16 & 1.001 & 1.125 \\
Equalized Expected Set Size CP: coverage RMS & 15 & 1.546 & 1.686 \\
\bottomrule
\end{tabular}
\end{table}

All \(131\) independently evaluated population configurations satisfy their corresponding bounds, with the Equalized Coverage expected-set-size bound nearly attained in the tightest cases.
\section{Computational Resources}
\label{app:compute_resource}
All experiments were conducted on a Windows 11 computer with an Intel Core i9-12900H CPU, an NVIDIA GeForce RTX 3080 Ti Laptop GPU (16 GB VRAM), and 32 GB RAM. Each reported experiment was completed within one hour.

\end{document}